\documentclass{article}

 \usepackage[preprint]{neurips}

\usepackage[utf8]{inputenc} 
\usepackage[T1]{fontenc}    
\usepackage{hyperref}       
\usepackage{url}            
\usepackage{booktabs}       
\usepackage{amsfonts}       
\usepackage{nicefrac}       
\usepackage{microtype}      

\usepackage[usenames,dvipsnames]{color}
\hypersetup{
  colorlinks,
  citecolor=Blue,
  linkcolor=Red,
  urlcolor=Violet}

\usepackage[noline,ruled]{algorithm2e}
\usepackage{amsmath,bbm,bm,graphicx,amsthm}
\usepackage{wrapfig}
\usepackage{subcaption}

\usepackage[T1]{fontenc}
\usepackage[utf8]{inputenc}
\usepackage{authblk}

\newtheorem{theorem}{Theorem}[section]

\newtheorem{lemma}[theorem]{Lemma}
\newtheorem{assumption}{Assumption}

\newtheorem{define}{Definition}

\DeclareMathOperator*{\KL}{KL}

\title{Artificial Neural Variability for Deep Learning: \\ On Overfitting, Noise Memorization, and Catastrophic Forgetting}

\author[1,2]{Zeke Xie}
\author[3]{Fengxiang He}
\author[3]{Shaopeng Fu}
\author[1,2]{Issei Sato }
\author[3]{Dacheng Tao }
\author[2,1]{Masashi Sugiyama}
\affil[1]{The University of Tokyo}
\affil[2]{RIKEN Center for AIP}
\affil[3]{UBTECH Sydney AI Centre, The University of Sydney}

\affil[ ]{\textit {xie@ms.k.u-tokyo.ac.jp}}
\affil[ ]{ \textit {\{fengxiang.he,shfu7008,dacheng.tao\}@sydney.edu.au}}
\affil[ ]{\textit {\{sato,sugi\}@k.u-tokyo.ac.jp}}

\begin{document}

\maketitle

\begin{abstract}
Deep learning is often criticized by two serious issues which rarely exist in natural nervous systems: overfitting and catastrophic forgetting. It can even memorize randomly labelled data, which has little knowledge behind the instance-label pairs. When a deep network continually learns over time by accommodating new tasks, it usually quickly overwrites the knowledge learned from previous tasks. Referred to as the {\it neural variability}, it is well-known in neuroscience that human brain reactions exhibit substantial variability even in response to the same stimulus. This mechanism balances accuracy and plasticity/flexibility in the motor learning of natural nervous systems. Thus it motivates us to design a similar mechanism named {\it artificial neural variability} (ANV), which helps artificial neural networks learn some advantages from ``natural'' neural networks. We rigorously prove that ANV plays as an implicit regularizer of the mutual information between the training data and the learned model. This result theoretically guarantees ANV a strictly improved generalizability, robustness to label noise, and robustness to catastrophic forgetting. We then devise a {\it neural variable risk minimization} (NVRM) framework and {\it neural variable optimizers} to achieve ANV for conventional network architectures in practice. The empirical studies demonstrate that NVRM can effectively relieve overfitting, label noise memorization, and catastrophic forgetting at negligible costs.
\footnote{Code: \url{https://github.com/zeke-xie/artificial-neural-variability-for-deep-learning}. }
\end{abstract}

\section{Introduction}
\label{sec:intro}

Inspired by natural neural networks, artificial neural networks have achieved comparable performance with humans in a variety of application domains \citep{lecun2015deep, witten2016data, silver2016mastering, he2016deep, litjens2017survey}. Deep neural networks are usually highly overparametrized \citep{keskar2017large, dinh2017sharp, arpit2017closer,kawaguchi2019effect}; the number of weights is usually way larger than the sample size. The extreme overparameterization gives deep neural network excellent approximation \citep{cybenko1989approximation, funahashi1989approximate, hornik1989multilayer, hornik1993some} and optimization \citep{allen2019convergence, arora2018optimization, li2018learning, allen2019learning} abilities as well as a prohibitively large hypothesis capacity. This phenomenon makes almost all capacity-based generalization bounds vacuous. Besides, former empirical results demonstrate that deep neural networks almost surely achieve zero training error even when the training data is randomly labelled \citep{zhang2017understanding}. This memorization of noise suggests that deep learning is “good at” overfitting. 


Deep learning performs poorly at learning multiple tasks from dynamic data distributions \citep{parisi2019continual}. The functionality of artificial neural networks is sensitive to weight perturbations. Thus, continually learning new tasks can quickly overwrite the knowledge learned through previous tasks, which is called catastrophic forgetting \citep{mccloskey1989catastrophic,goodfellow2013empirical}.  
Neuroscience has motivated a few algorithms for overcoming catastrophic forgetting and variations in data distributions \citep{kirkpatrick2017overcoming,zenke2017continual,chen2019improving}. 

Natural neural networks have much better generalizability and robustness. Can we learn from human brains again for more innovations in deep learning? An extensive body of works in neuroscience suggest that {\it neural variability} is essential for learning and proper development of human brains, which refers to the mechanism that human brain reactions exhibit substantial variability even in response to the same stimulus \citep{churchland2006neural, churchland2010stimulus, dinstein2015neural}. 
Neural variability acts as a central role in motor learning, which helps balance the need for accuracy and the need for plasticity/flexibility \citep{fetters2010perspective}. The ever-changing environment requires performers constantly adapt to both external (e.g., slippery surface) and internal (e.g., injured muscle) perturbations. It is also suggested that adult motor control systems can perform better by generating neural variability actively in order to leave room for adaptive plasticity/flexibility \citep{tumer2007performance}. An appropriate degree of neural variability is necessary to studies of early development \citep{hedden2004insights, olveczky2011changes}. A study on Parkinson’s disease suggests that the learning ability to new movements and adaptability to perturbations is dramatically reduced when neural variability is at a low degree \citep{mongeon2013impact}. 

Inspired by the neuroscience knowledge, this paper formally formulates artificial neural variability theory for deep learning. We mathematically prove that ANV plays as an implicit regularizer of the mutual information between the learned model weights and the training data. A beautiful coincidence in neuroscience is that neural variability in the rate of response to a steady stimulus also penalizes the information carried by nerve impulses (spikes) \citep{stein2005neuronal,houghton2019calculating}. Our theoretical analysis guarantees that ANV can strictly relieve overfitting, label noise memorization, and catastrophic forgetting. 

We further propose a {\it neural variable risk minimization} (NVRM) framework, which is an efficient training method to achieve ANV for artificial neural networks. In the NVRM framework, we introduce weight perturbations during inference to simulate neural variability of human brains to relieve overfitting and catastrophic forgetting. The empirical mean of the loss in the presence of weight perturbations is referred to as the {\it neural variable risk} (NVR). Similar to the neural variability, replacing the conventional empirical risk minimization (ERM) by NVRM would balance the accuracy-plasticity tradeoff in deep learning.

The rest of this paper is organized as follows. In Section \ref{sec:theory}, we propose the neural variability theory, and mathematically validate that ANV relieves overfitting, label noise memorization, and catastrophic forgetting. In Section \ref{sec:NVRM}, we propose the NVRM framework and neural variable optimizers, which can achieve ANV efficiently in practice. In Section \ref{sec:empirical}, we conduct extensive experiments to validate the theoretical advantages of NVRM. Particularly, training neural networks via neural variable optimizers can easily achieve remarkable robustness to label noise and weight perturbation. In Section \ref{sec:conclusion}, we conclude our main contribution. 

\section{Neural Variability Theory}
\label{sec:theory}

In this section, we will formally introduce artificial neural variability into deep learning. 
We denote a model with the weights $\theta$ as $\mathcal{M}(\theta)$ and the training dataset as $S= \{(x^{(i)},y^{(i)})\}_{i=1}^{m} $ drawn from the data distribution $\mathcal{S}$. We define the empirical risk over the training dataset $S$ as $\hat{L}(\theta) = L(\theta,S) =  \frac{1}{m} \sum_{i=1}^{m} L(\theta, (x^{(i)}, y^{(i)})) $, and the population risk over the data distribution $\mathcal{S}$ as $L(\theta) = \mathbb{E}_{(x,y)\sim \mathcal{S}} [L(\theta, (x, y))] $. 
We formally define $(b, \delta)$-neural variability ($(b, \delta)$-NV) as Definition \ref{df:nv}.
 \begin{define}[Neural Variability/Regional Flatness]
 \label{df:nv}
Suppose $L(\theta, S)$ is the loss function for the model $\mathcal{M}(\theta)$ on the dataset $S$, $ \hat{\theta}$ obeys a Gaussian distribution centered at $\theta$ as $\hat{\theta} \sim \mathcal{N}(\theta, b^{2}I)$, and
\[  | \mathbb{E}_{\hat{\theta} \sim \mathcal{N}(\theta, b^{2}I)  } [ L(\hat{\theta}, S)]  - L(\theta, S) ] |  \leq  \delta, \]
where $|\cdot|$ denotes the absolute value, and both $\delta$ and $b$ are positive. Then, the model $\mathcal{M}(\theta)$ is said to achieve $(b, \delta)$-neural variability at $\theta$ on the dataset $S$. It can also be said the model achieves $(b, \delta)$-regional flatness at $\theta$ on the dataset $S$.
 \end{define}
The definition has a similar form to $(\mathcal{C}_{\epsilon}, A)$-sharpness defined by \citet{keskar2017large}. 
A model $\mathcal{M}(\theta)$ with $(b, \delta)$-neural variability can work almost equally well when its weights are randomly perturbed as $ \hat{\theta} \sim \mathcal{N}(\theta, b^{2}I)$. This definition mimics the neuroscience mechanism that human brains can work well or even better by actively generating perturbations \citep{tumer2007performance}. The definition of $(b, \delta)$-neural variability is also a measure of robustness to weight perturbations and a measure of weight uncertainty for Bayesian neural networks.

 \subsection{Generalization}

In this section, we will formulate the information theoretical foundation of $(b, \delta)$-neural variability by using the PAC-Bayesian framework \citep{mcallester1999some,mcallester1999pac}. The PAC-Bayesian framework provides guarantees on the expected risk of a randomized predictor (hypothesis) that depends on the training dataset. The hypothesis is drawn from a distribution $Q$ and sometimes referred to as a posterior. We then denote the expected risk with respect to the distribution $Q$ as $L(Q)$ and the empirical risk with respect to the distribution $Q$ as $\hat{L}(Q)$. Suppose $P$ is the prior distribution over the weight space $\Theta$.

 \begin{lemma}[The PAC-Bayesian Generalization Bound \citep{mcallester1999some}]
 \label{pr:pac}
For any real $\Delta \in (0,1)$, with probability at least $1-\Delta$, over the draw of the training dataset $S$, the expected risk for all distributions $Q$ satisfies
\[  L(Q) \leq  \hat{L}(Q) + 4\sqrt{\frac{1}{m} [\KL(Q \| P) + \ln(\frac{2m}{\Delta})]}   , \]
where $\KL(Q \| P)$ denotes the Kullback–Leibler divergence from $P$ to $Q$. 
\end{lemma}

 The PAC-Bayesian generalization bound closely depends on the prior $P$ over the model weights. We make a mild Assumption \ref{as:prior}.
 \begin{assumption}
 \label{as:prior}
The prior over model weights is Gaussian, $P = \mathcal{N}(0, \sigma^{2}I)$.
\end{assumption}
Assumption \ref{as:prior} justified as it can be interpreted as weight decay, which is widely used in related papers \citep{graves2011practical,neyshabur2017exploring,he2019control}. We note that $\sigma^{2}$ is very large in practice, as $\sigma^{2}$ is equal to the inverse weight decay strength.

We consider a distribution $Q_{\mathrm{nv}}$ over model weights of the form $\theta+\epsilon$, where $\theta$ is drawn from the distribution $Q$ and $\epsilon \sim \mathcal{N}(0, b^{2}I)$ is a random variable. Following the theoretical analysis, particularly Equation 7 of \citet{neyshabur2017exploring}, we formulate Theorem \ref{pr:pacgen}.

 \begin{theorem}[The Generalization Advantage of ANV]
 \label{pr:pacgen}
Suppose the model $\mathcal{M}(\theta^{\star})$ achieves $(b, \delta)$-neural variability at $\theta^{\star}$, and Assumption \ref{as:prior} holds. Then, for any real $\Delta \in (0,1)$, with probability at least $1-\Delta$, over the draw of the training dataset $S$, the expected risk for all distributions $Q_{\mathrm{nv}}$ satisfies
\[  L(Q_{\mathrm{nv}}) \leq \hat{L}(\theta^{\star}) +  4\sqrt{\frac{1}{m} \left[ \KL(Q_{\mathrm{nv}}\| P) + \ln(\frac{2m}{\Delta})\right]} + \delta,  \]
where $N$ is the number of model weights and  $\KL(Q_{\mathrm{nv}}\| P) = \sum_{i=1}^{N} \left[  \log\left(\frac{\sigma}{b}\right) + \frac{b^{2} + \theta_{i}^{\star 2}}{2\sigma^{2}}  - \frac{1}{2}\right]$
 \end{theorem}

We leave all proofs in Appendix \ref{app:proofs}. We note that $\KL(Q_{\mathrm{nv}}\| P) $ as the function of $b$ decreases with $b$ for $b \in (0,\sigma)$, and reaches the global minimum at $b=\sigma$. As $\sigma$ is much larger than $1$ and $b$ in practice, the PAC-Bayesian bound monotonically decreases with $b$ given $\delta$.  The bound is tighter than the bound in Lemma \ref{pr:pac} when the model has strong ANV, which means $b$ is large given a small $\delta$.

It is known that the information in the model weights relates to overfitting \citep{hinton1993keeping} and flat minima \citep{hochreiter1997flat}. \citet{achille2019information} argued that the information in the weights controls the PAC-Bayesian bound. We show that the generalization bound in Theorem \ref{pr:pacgen} positively correlates with the mutual information of learned model weights and training data. Given two random variables $\theta$ and $S$, their Shannon mutual information is defined as $I(\theta; S) = \mathbb{E}_{S\sim \mathcal{S}}[\KL(p(\theta|S) \| p(\theta))] $ which is the expected Kullback-Leibler divergence from the prior distribution $p(\theta)$ of $\theta$ to the distribution $p(\theta|S)$ after an observation of $S$ \citep{cover2012elements}. In the case of Theorem \ref{pr:pacgen}, we have
\begin{align}
\label{eq:expectedkl}
\mathbb{E}_{S\sim \mathcal{S}}[\KL( Q_{\mathrm{nv}} \| P)]  = I( \theta; S),
\end{align}
where $\theta \sim Q_{\mathrm{nv}}$. It indicates that penalizing the mutual information of the learned model weights and training data, $I( \theta; S)$, is equivalent to decreasing the expected $\KL( Q_{\mathrm{nv}} \| P)$, which may improve generalization.
As $S \rightarrow \theta \rightarrow \theta + \epsilon$ is a Markov process, we have the data processing mutual information inequality $I(\theta+\epsilon;S) < I(\theta;S)$. It indicates that ANV regularizes the mutual information between the learned model weights and training data. This theoretical evidence is quite close to the neuroscience mechanism of penalizing the information carried by nerve impulses \citep{stein2005neuronal}. 

Different from the PAC-Bayesian approach, another theoretical framework for the generalization bound based on mutual information was proposed by \citet{xu2017information}. Following \citet{xu2017information}, we formulate an alternative mutual-information-based generalization bound in Appendix \ref{app:mibound}.

\subsection{The Robustness to Label Noise}
Noisy labels can remarkably damage the generalization of deep networks, because deep networks can completely memorize corrupted label information \citep{zhang2017understanding}. Memorizing noisy labels is one of the most serious overfitting issues in deep learning. We will show that ANV relieves deep networks from memorizing noisy labels by penalizing the mutual information of the model weights $\theta$ and the labels $y$ conditioned on the inputs $x$. 

In Section 4 of \citet{achille2018emergence}, the expected cross entropy loss can be decomposed into several terms to describe it. If the data distribution $\mathcal{S}$ is fixed, the expected cross entropy loss for the training performance can be decomposed into three terms:
\begin{align}
\label{eq:decomp}
\mathcal{H}_{f}(y|x, \theta) =&\mathbb{E}_{S}  \mathbb{E}_{\theta\sim Q(\theta|S)} \left[\sum_{i=1}^{m} - \log f(y^{(i)}|x^{(i)}, \theta) \right] \\
= & \mathcal{H}(y|x) + \mathbb{E}_{x,\theta \sim Q(\theta|S)} \KL[ p(y|x) \| f(y|x,\theta)] - I(\theta; y|x),
\end{align}
where $f$ denotes the model's map from an input $x$ to a class distribution, $\mathcal{H}(\cdot)$ denotes the entropy $\mathbb{E}[-\log(\cdot)]$. The meaning of each term has been interpreted by \citet{achille2018emergence} in details.
The first term relates to the intrinsic error that we would commit in predicting the labels even if we knew the underlying data distribution.
The second term relates to the efficiency of the model and the class of functions $f$ with respect to which the loss is optimized.
Here we focus on the last term: the label memorization can be given by the mutual information between the model weights and the labels conditioned on inputs, namely $I (\theta;y|x)$. In the traditional paradigm of deep learning, minimizing $- I(\theta; y|x)$ is expected. Thus, deep learning easily overfits noisy labels. Noisy labels as outliers of the data distribution implies a positive value of $I(\theta;y|x)$, which requires more information to be memorized. We need to reduce $ I(\theta; y|x)$ effectively to prevent deep networks from overfitting noisy labels. With this approach, Theorem 2.1 of \citet{harutyunyan2020improving} also supported that memorization of noisy labels is prevented by decreasing $I (\theta;y|x)$. 

Suppose a model $\mathcal{M}(\theta)$ achieves $(b, \delta)$-NV, $\delta$ is small, and we inject weight noise $\epsilon$ to this model. We have a Markov process as $y|x \rightarrow \theta \rightarrow \hat{\theta}$, where we denote $\theta + \epsilon$ as $\hat{\theta}$. Based on Equation \eqref{eq:decomp}, we have 
\begin{align}
\label{eq:nvdecomp}
\mathcal{H}_{f}(y|x, \hat{\theta}) =\mathcal{H}(y|x) + \mathbb{E}_{x,\theta \sim Q_{\mathrm{nv}}(\theta|S)} \KL[ p(y|x) \| f(y|x,\theta)] - I(\hat{\theta}; y|x),
\end{align}
According to Definition \ref{df:nv}, the model $\mathcal{M}(\hat{\theta})$ may achieve nearly equal training performance to $\mathcal{M}(\theta)$ given small $\delta$. At the same time, obviously, $I (\hat{\theta};y|x)$ is smaller than $I (\theta;y|x)$ as $\epsilon$ penalizes the mutual information. This suggests that increasing $b$ given $\delta$ for a $(b, \delta)$-neural variable model can penalize the memorization of noisy labels by regularizing the mutual information of learned model weights and training data.

\subsection{The Robustness to Catastrophic Forgetting}

The ability to continually learn over time by accommodating new tasks while retaining previously learned tasks is referred to as continual or lifelong learning \citep{parisi2019continual}. However, the main issue of continual learning is that artificial neural networks are prone to catastrophic forgetting. In natural neural systems, neural variability leaves room for the excellent plasticity and continual learning ability \citep{tumer2007performance}. It is natural to conjecture that ANV can help relieve catastrophic forgetting and enhance continual learning. 

We take regularization-based continual learning \citep{kirkpatrick2017overcoming,zenke2017continual,aljundi2018memory} as an example. 
The basic idea of regularization-based continual learning is to strongly regularize weights most relevant to previous tasks. Usually, the regularization is strong enough to fix learning near the solution learned from previous tasks. In a way, the model tends to learn in the overlapping region of optimal solutions for multiple tasks \citep{doan2021theoretical}. The intuitional explanation behind ANV is clear: if a model is more robust to weight perturbation, it will have a wider optimal region shared by multiple tasks.

Suppose a model $\mathcal{M}(\theta)$ continually learns Task A and Task B, where the learned solutions are, respectively, $\theta^{\star}_{A}\sim Q_{A}$ and $\theta^{\star}_{B} \sim Q_{B}$. The distribution $Q_{B}$ describes models weights of the form $\theta^{\star}_{A}+\epsilon$, where $\epsilon$ is a random variable. For any real $\Delta \in (0,1)$, with probability at least $1-\Delta$, over the draw of the training dataset $S$ for Task A, the expected risk for all distributions $Q_{B}$ satisfies
\begin{align}
\label{eq:pcacf}
  L(Q_{B}) \leq  \hat{L}(Q_{B}) +  4\sqrt{\frac{1}{m} \left[ \KL( Q_{B} \| P)+ \ln(\frac{2m}{\Delta})\right]} .
\end{align}
For $\theta^{\star}_{B} = \theta^{\star}_{A} + \epsilon$, as $S \rightarrow \theta^{\star}_{A} \rightarrow \theta^{\star}_{B} $ is a Markov process and the weight perturbation $\epsilon$ is learned from the training dataset $S_{B}$ (of Task B) only, the the data processing mutual information inequality $I(\theta^{\star}_{B}; S) <  I(\theta^{\star}_{A}; S) $ still holds. We recall the mutual information analysis above and have $I(\theta^{\star}_{B}; S) = \mathbb{E}_{S \sim \mathcal{S}}[\KL(Q_{B} \| P)]  $. Thus, we have
\begin{align}
 L(Q_{B}) - \hat{L}(Q_{B}) & = [L(Q_{B}) - \hat{L}(Q_{A}) ] - [ \hat{L}(Q_{B})  - \hat{L}(Q_{A}) ] \\
 & < 4\sqrt{\frac{1}{m} \left[ \KL( Q_{B} \| P)+ \ln(\frac{2m}{\Delta})\right]} 
 \end{align}
Considering the expectation with respect to the data distribution, we obtain
\begin{align}
E_{S}[ (L(Q_{B}) - \hat{L}(Q_{A}) - \delta_{AB} )^{2} ]  < \frac{16}{m} \left[  I(\theta^{\star}_{B}; S)  + \ln(\frac{2m}{\Delta})\right] ,
\label{eq:anvcf}
 \end{align}
where $\delta_{AB} = \hat{L}(Q_{B})  - \hat{L}(Q_{A})$ and $I(\theta^{\star}_{B}; S) <  I(\theta^{\star}_{A}; S) $. So the population risk $L(Q_{B})$ (for Task A) can be well bounded by the empirical risk increasing $\delta_{AB}$ and the mutual information $I(\theta^{\star}_{A}; S)$. Here $\delta_{AB}$, the empirical risk increasing due to weight perturbation, is a kind of measure of robustness to weight perturbation. It suggests that increasing robustness to weight perturbation can help relieve catastrophic forgetting.

\section{Neural Variable Risk Minimization}
\label{sec:NVRM}

In this section, we aim to learn empirical minimizers with ANV, given a conventional network architecture. Our method is to minimize the empirical risk in a certain region rather than the empirical loss at a single point: 
\begin{align}
L_{\mathrm{NV}}(\theta, S)= \mathbb{E}_{\hat{\theta} \sim \mathcal{N}(\theta,\, b^{2}I)} [ L(\hat{\theta}, S) ],
\end{align}
where $b$ is the variability hyperparameter. We call the risk the Neural Variable Risk (NVR), and call optimizing the NVR Neural Variable Risk Minimization (NVRM).  The model $\mathcal{M}(\theta^{\star})$ learned by NVRM can naturally achieve $(b, \delta)$-neural variability, where $\delta = |L_{\mathrm{NV}}(\theta^{\star}, S) - L(\theta^{\star}, S)| $. In this paper, we usually let $\epsilon$ obey a Gaussian distribution, because Gaussian noise is the noise type which penalizes information most effectively given a certain variance. But it is easy to generalize our framework to other noise types, such as Laplace noise and uniform noise. The noise type can be regarded as a hyperparameter.

The next question is: how to perform NVRM? Unfortunately, NVR is intractable in practice. But it is possible to approximately estimate the NVR and its gradient by sampling $\hat{L}_{\mathrm{NV}}(\theta, S)=  L(\theta + \epsilon, S)$, where $\epsilon \sim \mathcal{N}(0,\, b^{2}I) $. This unbiased estimation method is also used in variational inference \citep{graves2011practical}.

We propose a class of novel optimization algorithms to employ NVRM in practice. We can write NVRM update as
\begin{align}
\theta_{t} = \theta_{t-1} - \eta \frac{\partial L(\theta_{t-1}+\epsilon_{t-1}, (x,y))}{\partial \theta}.
\end{align}
NVRM is mimicking the neural variability of human brains in response to the same stimulus. NVRM exhibits the variability of predictions and back-propagation even in response to the same inputs. This update cannot be implemented inside an optimizer. But if we introduce $\hat{\theta}_{t} = \theta_{t} + \epsilon_{t}$ into the NVRM update, then we obtain a novel updating rule:
\begin{align}
\hat{\theta}_{t} = \hat{\theta}_{t-1} - \eta \frac{\partial L(\hat{\theta}_{t-1}, (x,y))}{\partial \hat{\theta}} + \epsilon_{t} - \epsilon_{t-1}.
\end{align}
We can combine this updating rule with popular optimizers, such as SGD\citep{bottou1998online, sutskever2013importance}, then we easily get a class of novel optimization algorithms, such as NVRM-SGD. We call this class of optimization algorithms {\it neural variable optimizers}. The pseudocode of NVRM-SGD is displayed in Algorithm \ref{algo:nvrm}. Similarly, we can also easily obtain NVRM-Adam by adding the four colored lines of Algorithm \ref{algo:nvrm} into Adaptive Momentum Estimation (Adam) \citep{kingma2014adam}. The sourcecode is available at \url{https://github.com/zeke-xie/artificial-neural-variability-for-deep-learning}. We note that it is necessary to apply the de-noising step before we evaluate the model learned by NVRM on test datasets, as $\hat{\theta_{t}} = \theta_{t} + \epsilon_{t}$. We also call such weight perturbations Virtual Perturbations, which need be applied before inference and removed after back-propagation. We can easily empower neural networks with ANV by importing a neural variable optimizer to train them. 

\begin{algorithm}[H]
 \label{algo:nvrm}
 \caption{NVRM-SGD}
  \KwIn{Training data $S$, the variability hyperparameter $b$, the number of iterations $T$, learning rate $\eta$, initialized weights $\theta_{0}$}
  \KwOut{The model weights $\theta$}
  $ \color{blue} \epsilon_{0} = 0$\;
  \Repeat{stopping criterion is not met}{
    \For{$(x,y) \in S$}
    {
      $\color{blue}  \epsilon_{t} \sim \mathcal{N}(0,b^{2}I)$\;
      $g_{t} = \frac{\partial L(\theta_{t-1} , (x,y))}{\partial \theta}$\;
      $\theta_{t} = \theta_{t-1} - \eta g_{t}  $\;
      $\color{blue} \theta_{t} = \theta_{t}  + \epsilon_{t} - \epsilon_{t-1} $\;
    }
  }
$\color{blue} \theta_{T} = \theta_{T} - \epsilon_{T}$
\end{algorithm}

\paragraph{A Deep Learning Dynamical Perspective} We also note that it is possible to theoretically analyze NVRM from a deep learning dynamical perspective. NVRM actually introduces Hessian-dependent gradient noise into learning dynamics instead of injected white Gaussian noise in conventional noise injection methods, as the second-order Taylor approximation $\nabla L(\theta + \epsilon) \approx \nabla L(\theta ) + \nabla^{2} L(\theta) \epsilon$ holds for small weight perturbation. \citet{zhu2019anisotropic} argued that anisotropic gradient noise is often beneficial for escaping sharp minima. \citet{xie2020diffusion,xie2020adai} further quantitatively proved that Hessian-dependent gradient noise is exponentially helpful for learning flat minima. Again, flat minima \citep{hochreiter1997flat} are closely related to overfitting and the information in the model weights. This can mathematically explain the advantage of NVRM from a different perspective. We leave the diffusion-based approach as future work.

\paragraph{Related Work}
One related line of research is injecting weight noise into deep networks during training \citep{an1996effects,neelakantan2015adding,zhou2019toward}. For example, Perturbed Stochastic Gradient Descent (PSGD) is SGD with a conventional weight noise injection method, which is displayed in Algorithm \ref{algo:psgd}. Another famous example is Stochastic Gradient Langevin Dynamics (SGLD) \citep{welling2011bayesian} which differs from PSGD only in the magnitude of injected Gaussian noise. However, this conventional line does not remove the injected weight noise after each iteration, which makes it essentially different from our method. In Section \ref{sec:empirical}, we will empirically verify that the de-noising step is significantly helpful for preventing overfitting.

Variational Inference (VI) for Bayesian neural networks \citep{graves2011practical,blundell2015weight,khan2018fast} aims at estimating the posterior distribution of model weights given training data. VI requires expensive costs to update the posterior distribution (model uncertainty) during training. This line believes estimating the exact posterior is important but ignores the importance of enhancing model certainty. In contrast, our method is the first to actively encourage model uncertainty for multiple benefits by choosing the variability hyperparameter $b$. ANV may be regarded as applying a neuroscience-inspired hyperprior over model uncertainty. Inspired by recent works on Bayesian neural networks, we conjecture that the NVRM framework could help improve adversarial robustness \citep{carbone2020robustness} and fix overconfidence problems \citep{kristiadi2020being}

Another related line of research is Randomized Smoothing \citep{duchi2012randomized,nesterov2017random}. A recent paper \citep{wen2018smoothout} applied the idea of Randomized Smoothing in training of deep networks and proposed the so-called SmoothOut method to optimize a weight-perturbed loss. This is also what the proposed NVRM does. We note that the original SmoothOut is actually a different implementation of NVRM with uniform noise, which both belong to Randomized Smoothing. However, this line of research \citep{duchi2012randomized,wen2018smoothout} only focused on improving performance on clean datasets by escaping from sharp minima. To the best of our knowledge, our work is the first along this line to theoretically and empirically analyze label noise memorization and catastrophic forgetting.

In summary, our paper further made two important contributions beyond the existing related work: (1) we discovered that NVRM can play a very important role in regularizing mutual information, which helps relieve label noise memorization and catastrophic forgetting, and (2) we implemented the weight-perturbed gradient estimation as a simple and effective optimization framework. NVRM as an optimizer is more elegant and easy-to-use than the existing methods like SmoothOut, which need to update the weights to calculate a perturbed loss before each back propagation.

\section{Empirical Analysis}
\label{sec:empirical}

We conducted systematic comparison experiments to evaluate the proposed NVRM framework. To secure a fair comparison, every experimental setting was repeatedly trialled by 10 times while all irrelevant variables were strictly controlled. We evaluated NVRM by the mean performance and the standard deviations over 10 trials. We leave implementation details in Appendix \ref{app:exp}.


\paragraph{1. Robustness to weight perturbation.} For ResNet-34 \citep{he2016deep} trained on clean data, CIFAR-10 and CIFAR-100 \citep{krizhevsky2009learning}, we perturbed the weights by isotropic Gaussian noise of different noise scales to evaluate the test accuracy to weight perturbation. Figure \ref{fig:wn} demonstrates that the models trained by NVRM are significantly more robust to weight perturbation and have lower expected minima sharpness defined by \citet{neyshabur2017exploring}. This empirically verifies that the conventional neural networks trained via NVRM indeed learn strong ANV.

\begin{figure}[t!]
    \begin{subfigure}{0.99\linewidth}
    \centering
    \includegraphics[width=0.32\linewidth]{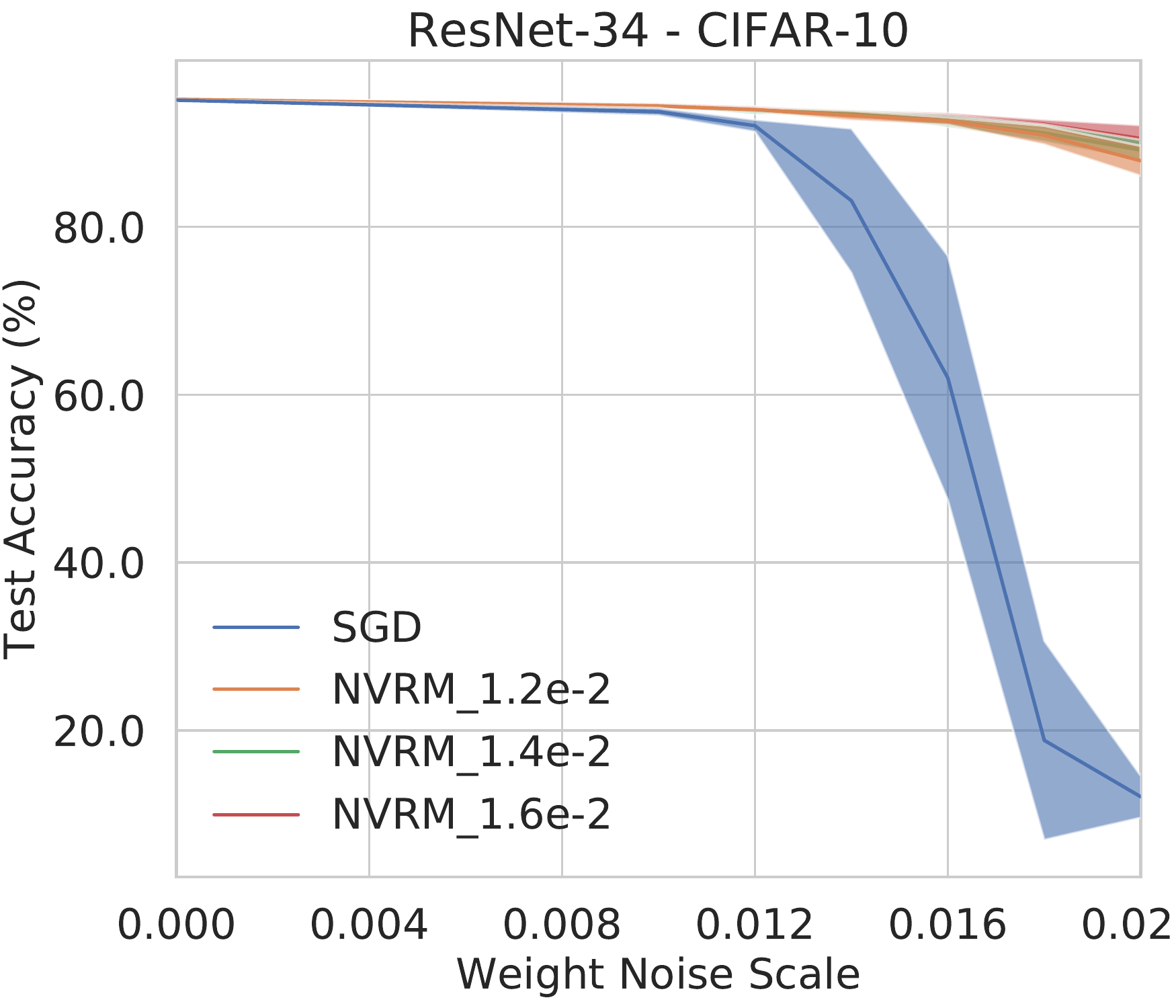}
    \includegraphics[width=0.32\linewidth]{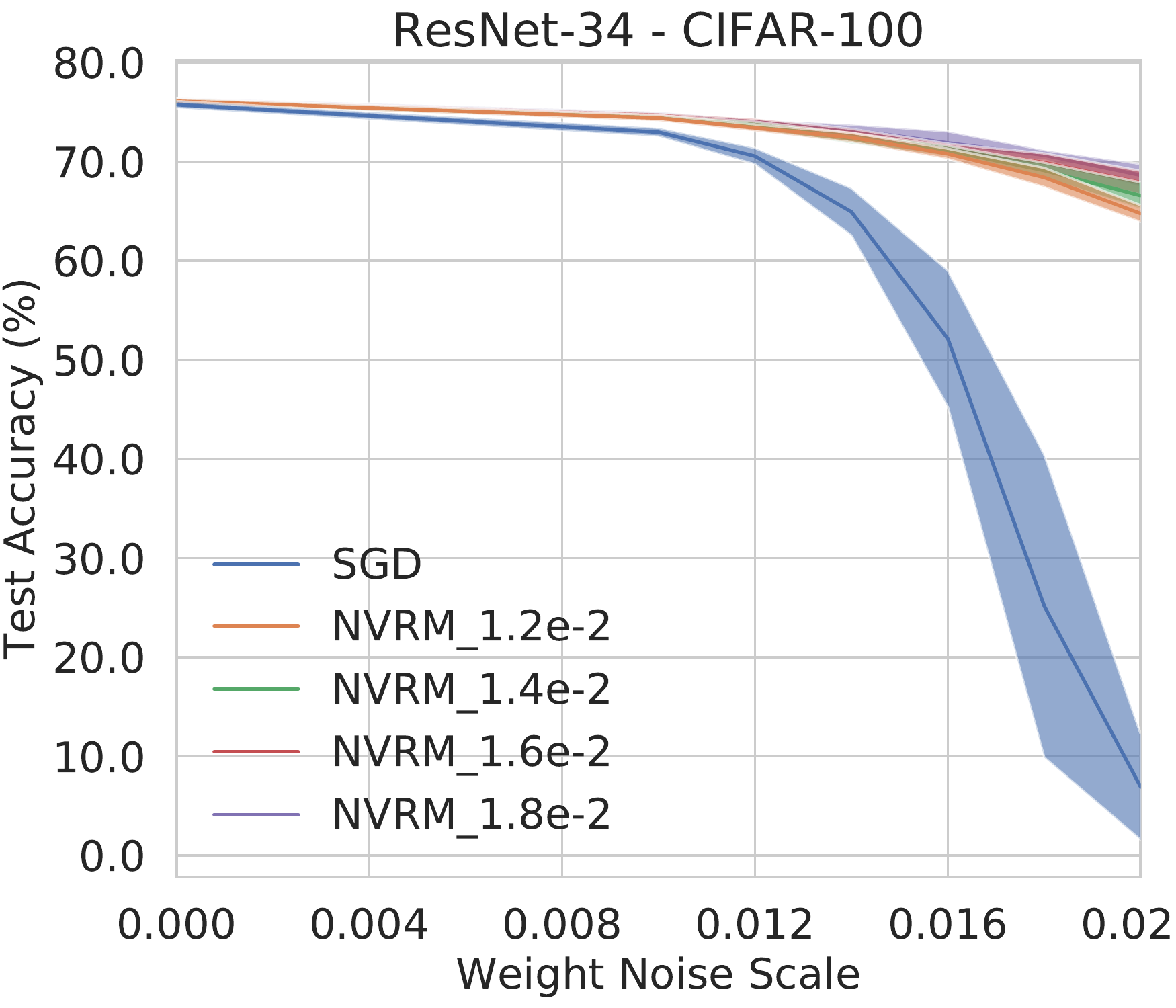}
  \end{subfigure}
  
  \caption{Curves of test accuracy to weight noise scale. The NVRM-trained network can almost retain reasonably well performance, while the SGD-trained network has nearly lost all learned knowledge due to relatively large weight noise. }
  \label{fig:wn}
\end{figure}
 
\paragraph{2. Improved generalization.} Model: VGG-16 \citep{Simonyan15} and MobileNetV2 \citep{sandler2018mobilenetv2}. Dataset: CIFAR-10 and CIFAR-100. We evaluated the test accuracy and the generalization gap, which is defined as the difference between the training accuracy and the test accuracy. The results in Figure \ref{fig:generalization} clearly demonstrate that NVRM can significantly narrow the generalization gap while slightly improve the test accuracy, which was also supported by \citet{nesterov2017random,wen2018smoothout}. 

\begin{figure}
\includegraphics[width=0.2425\linewidth]{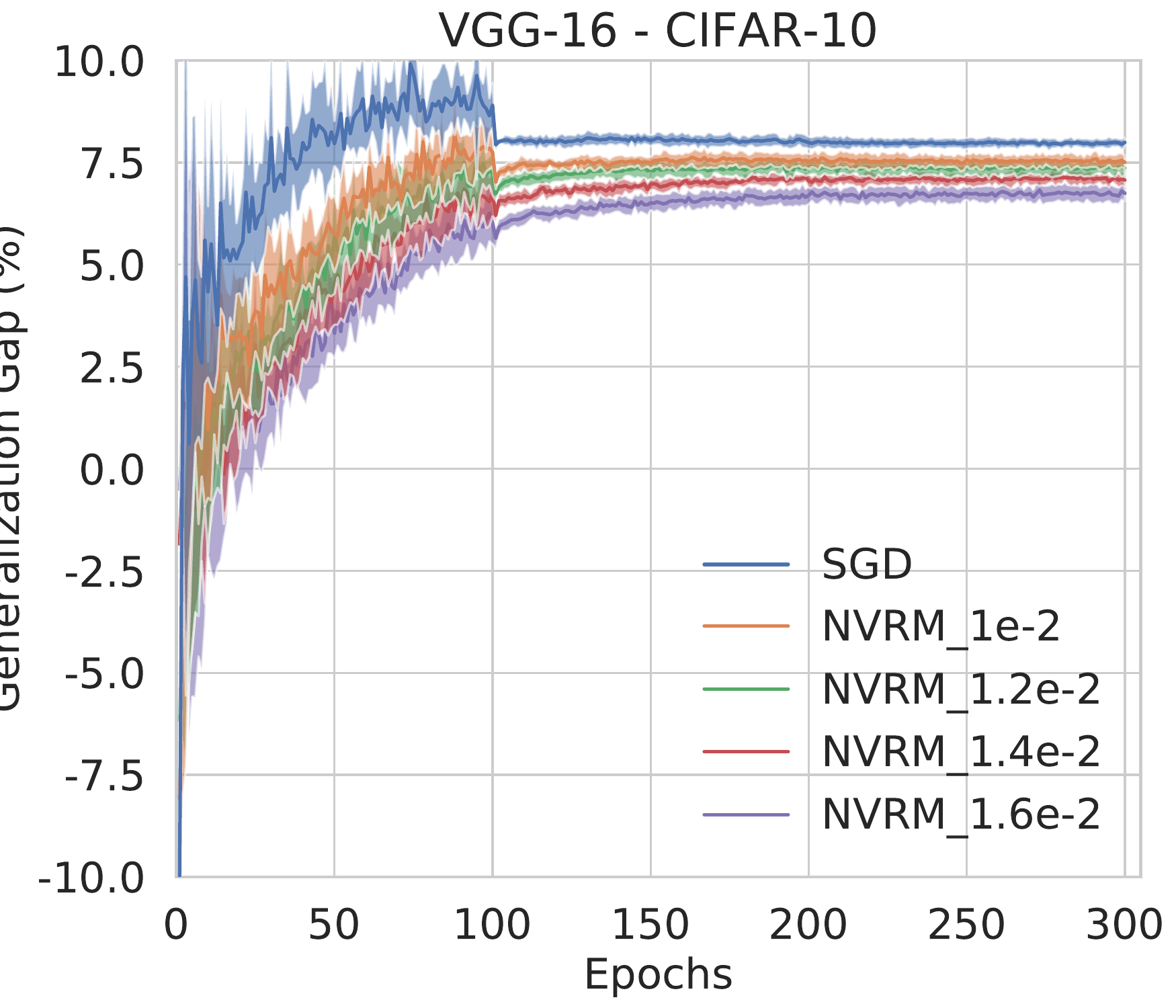}
\includegraphics[width=0.2425\linewidth]{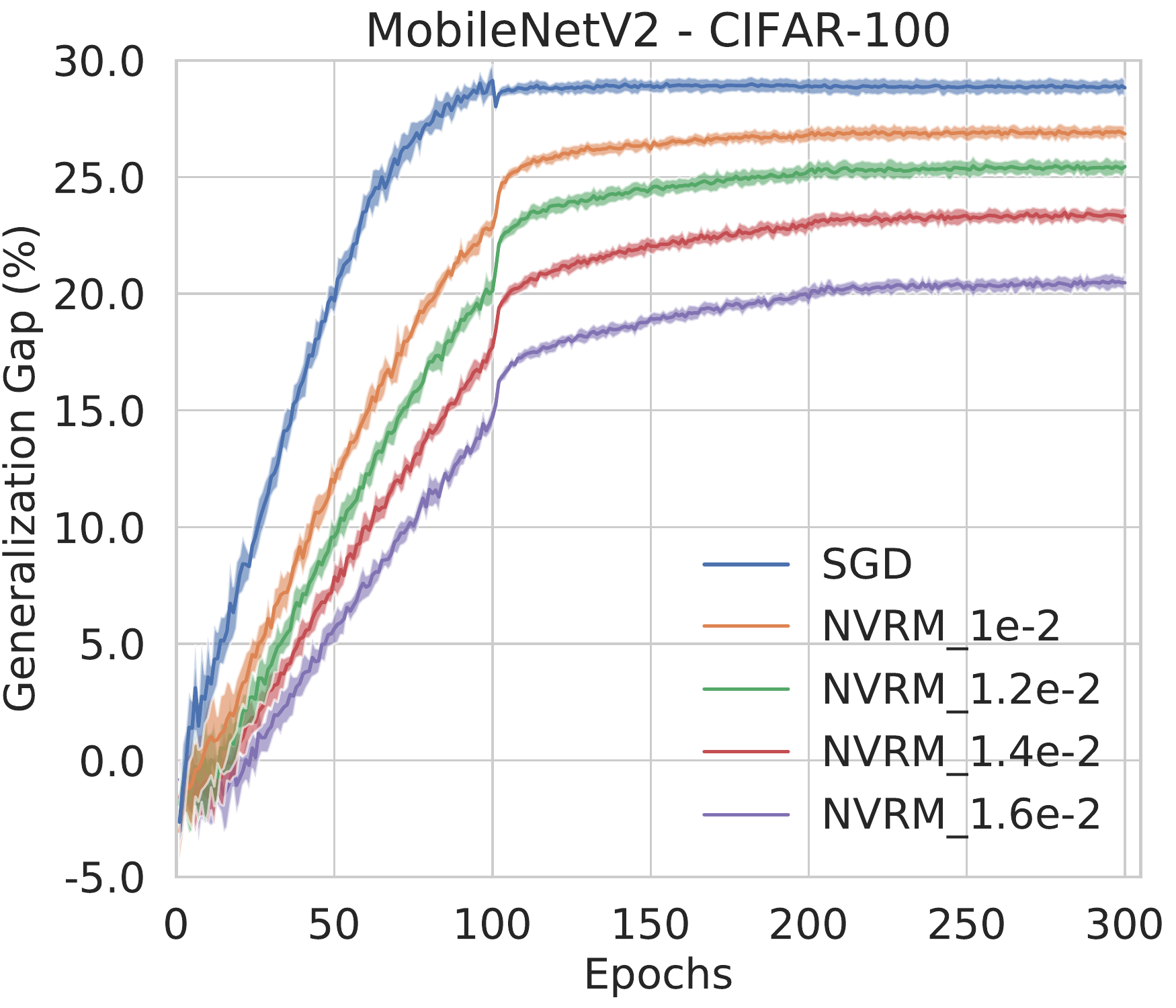}
\includegraphics[width=0.2425\linewidth]{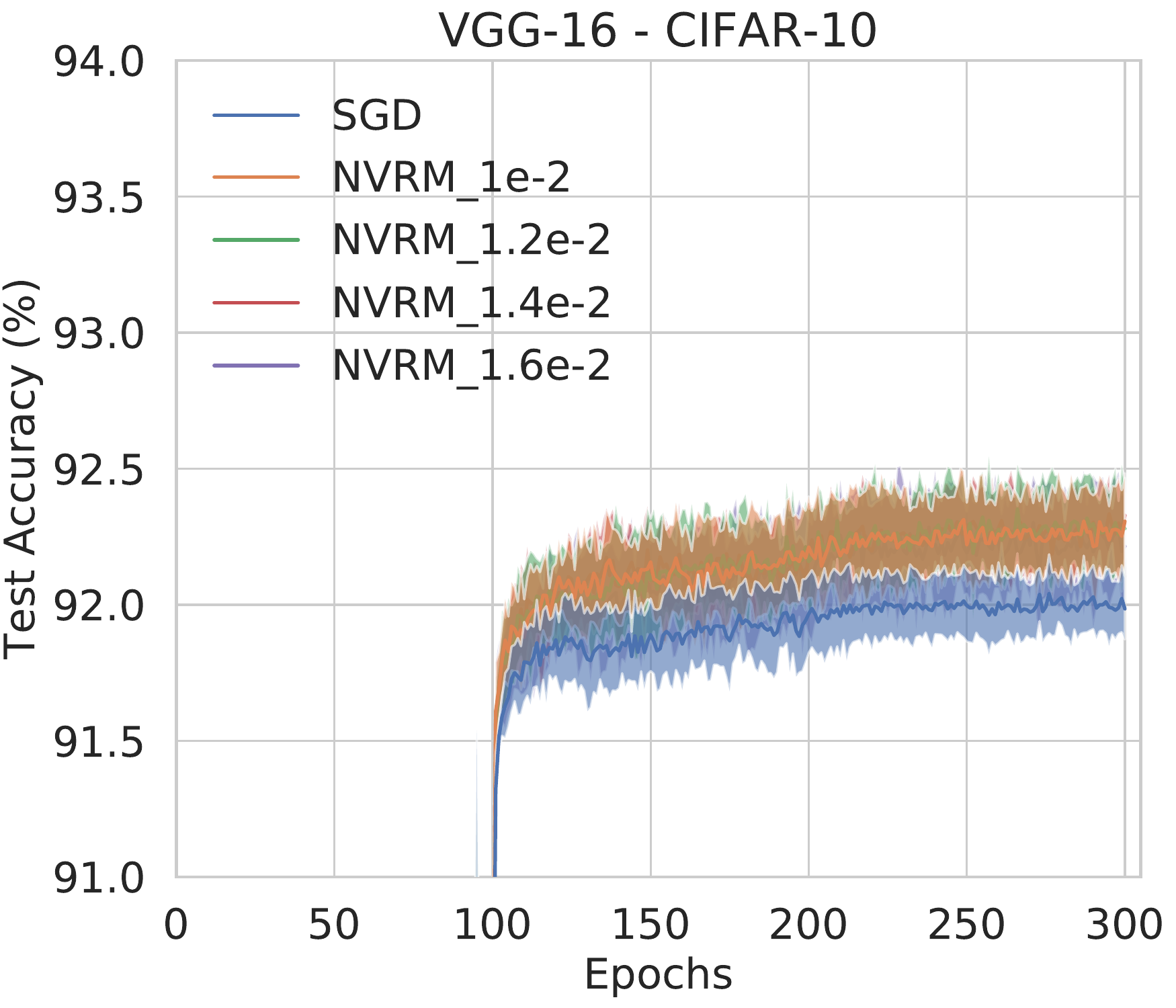}
\includegraphics[width=0.2425\linewidth]{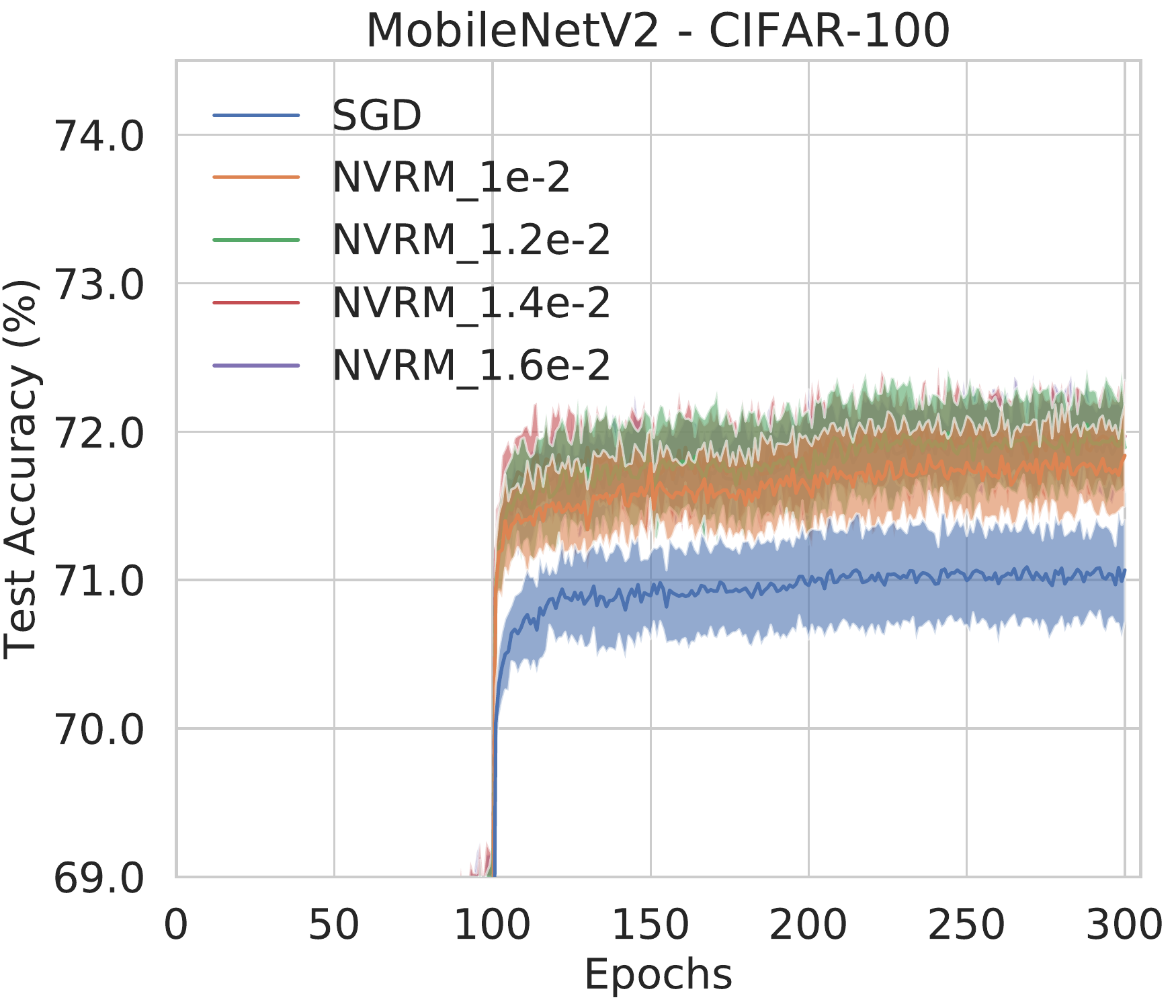}
\caption{NVRM with various variability scales $b$ can consistently improve generalization. Left two subfigures: Curves of generalization gap. Right two subfigures: Curves of test accuracy. We train VGG-16 on CIFAR-10 and MobileNetV2 on CIFAR-100. More results of VGG-16 on CIFAR-100 and MobileNetV2 on CIFAR-10 can be found in Appendix \ref{app:results}.}
 \label{fig:generalization}
\end{figure}

\paragraph{3. Robustness to noisy labels.} Model: ResNet-34. Dataset: CIFAR-10 and CIFAR-100. We inserted two classes of label noise into the datasets: the uniform flip label noise (symmetric label noise) and the pair-wise flip label noise (asymmetric label noise). Figure \ref{fig:resnet_asym} demonstrates that SGD seriously overfits noisy labels; meanwhile, NVRM can avoid memorizing noisy labels effectively. The experimental results of symmetric label noise in Appendix \ref{app:results} also support our theory.

\begin{figure}
  \begin{subfigure}{0.97\linewidth}
    \includegraphics[width=0.2425\linewidth]{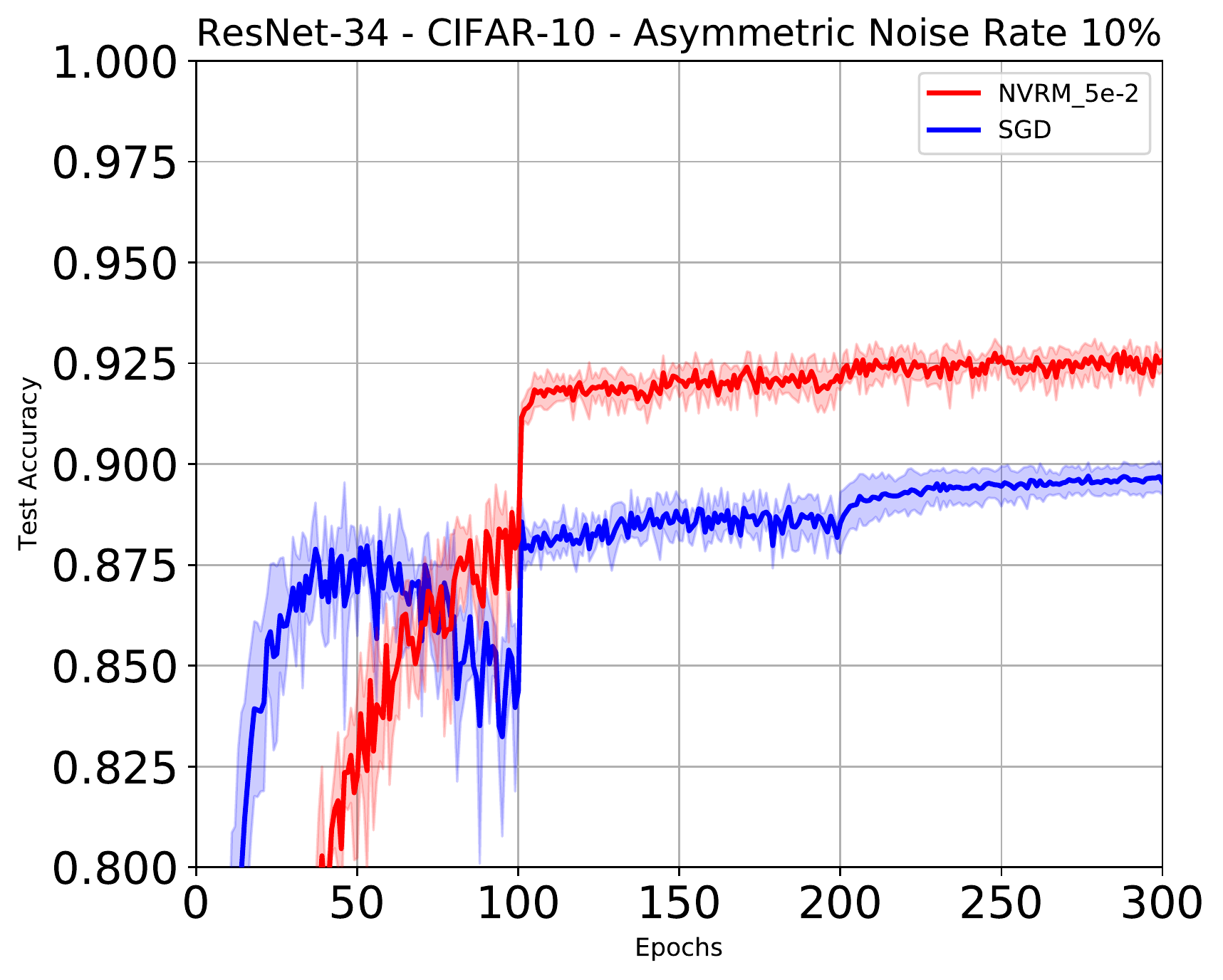}
    \includegraphics[width=0.2425\linewidth]{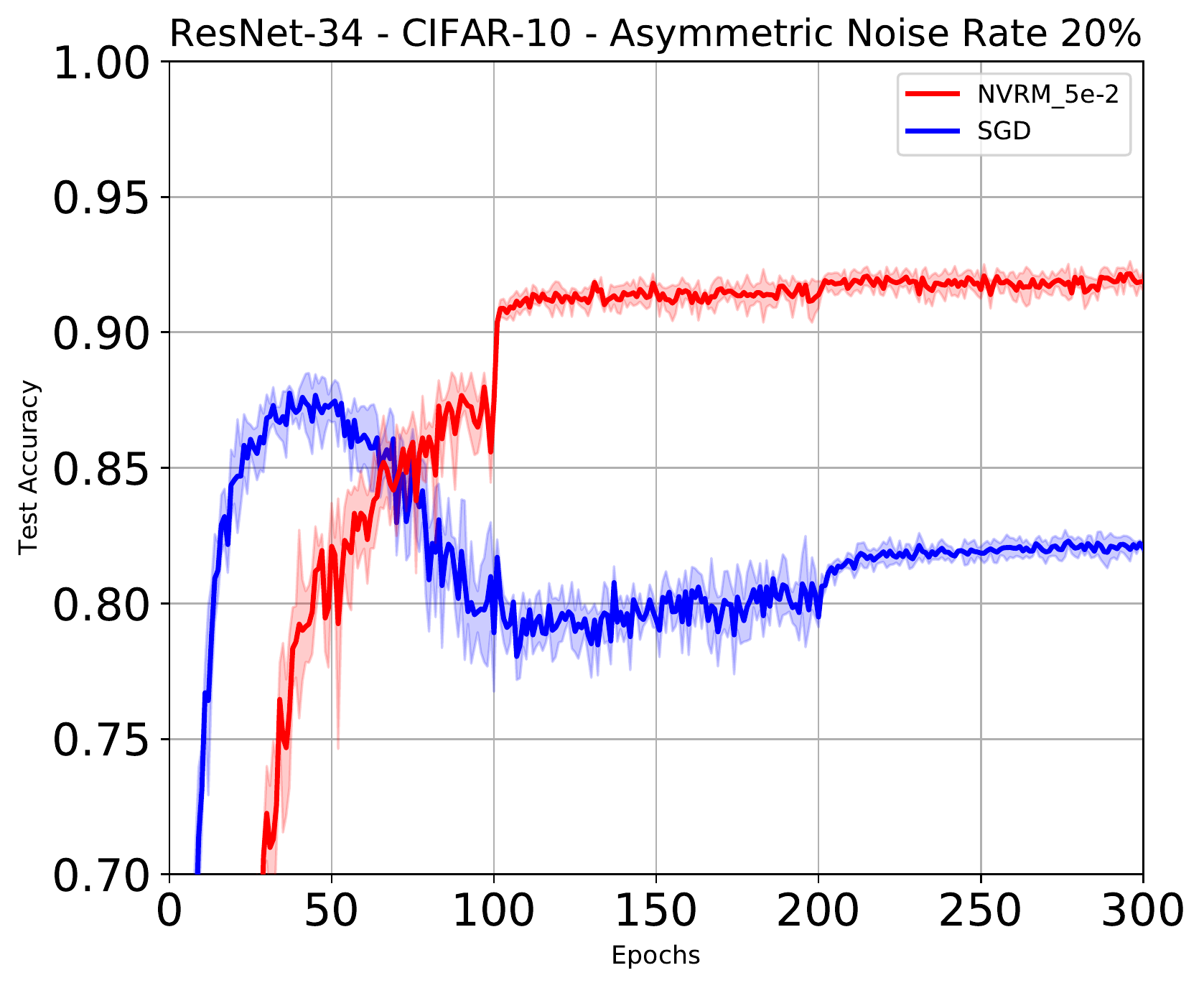}
    \includegraphics[width=0.2425\linewidth]{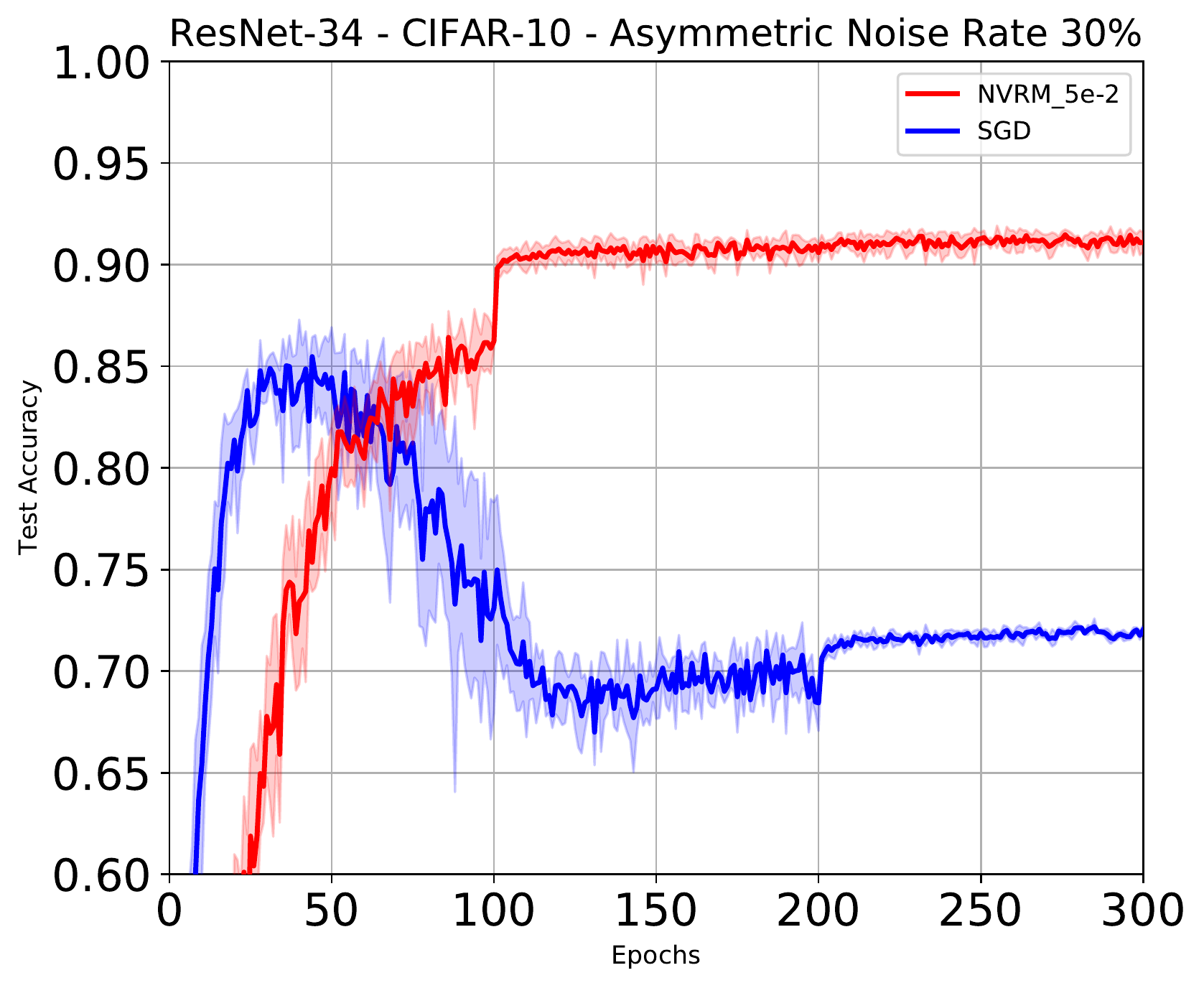}
    \includegraphics[width=0.2425\linewidth]{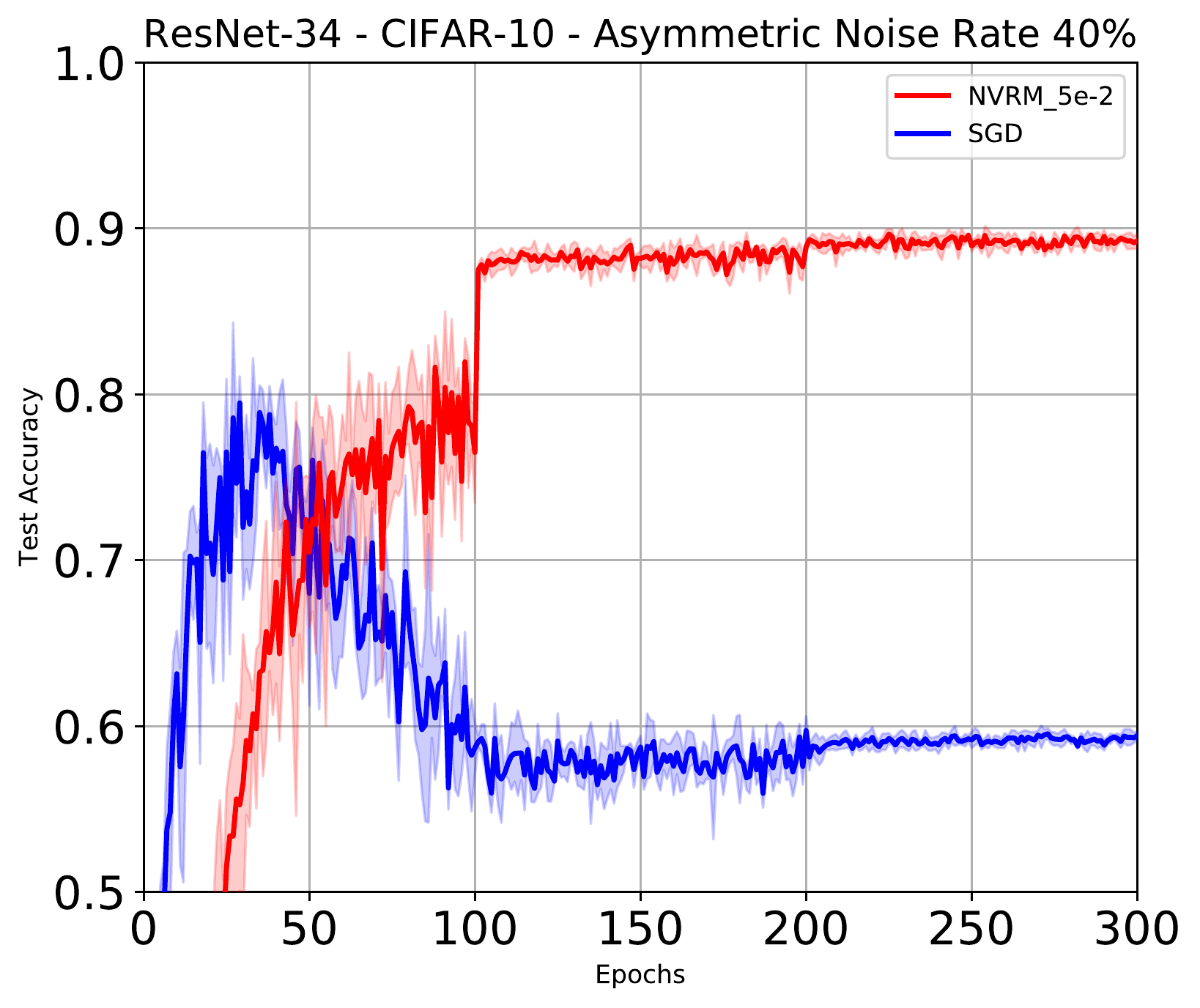}
    \caption{Test accuracy to epochs of ResNet-34 on CIFAR-10 with noisy labels.}
    \label{fig:C10_asymnoise}
  \end{subfigure} \\[0.5em]
  
    \begin{subfigure}{0.97\linewidth}
    \includegraphics[width=0.2425\linewidth]{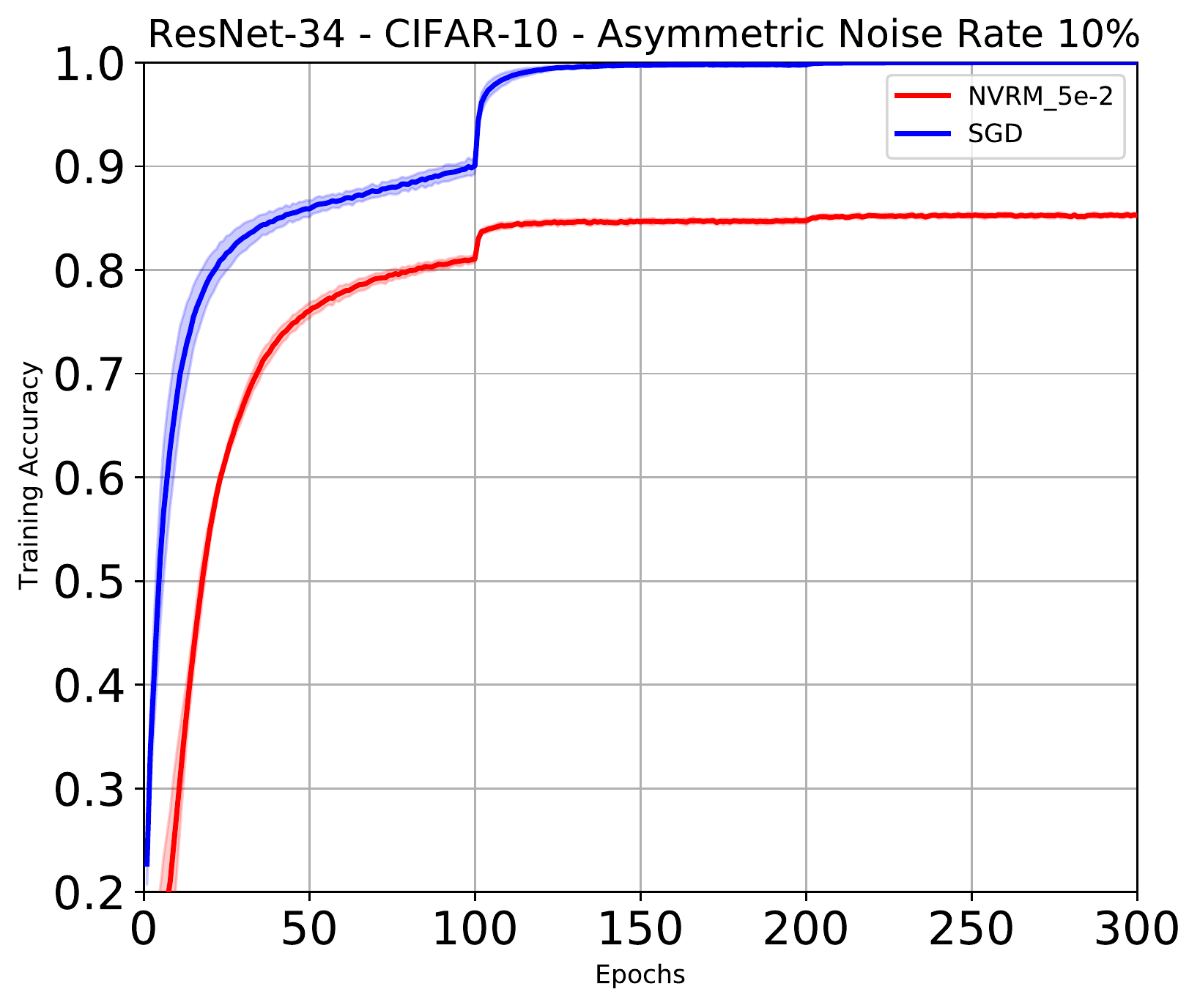}
    \includegraphics[width=0.2425\linewidth]{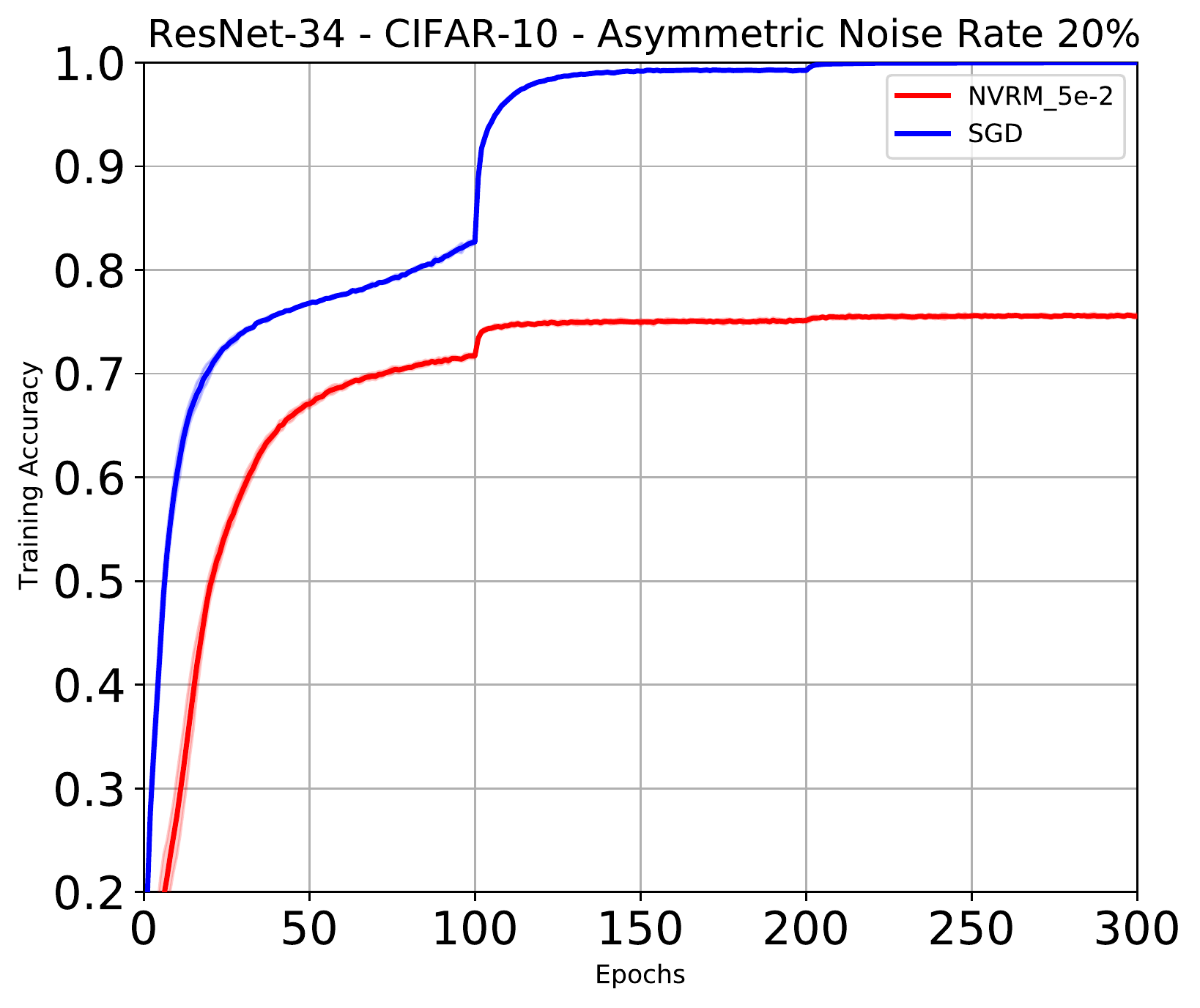}
    \includegraphics[width=0.2425\linewidth]{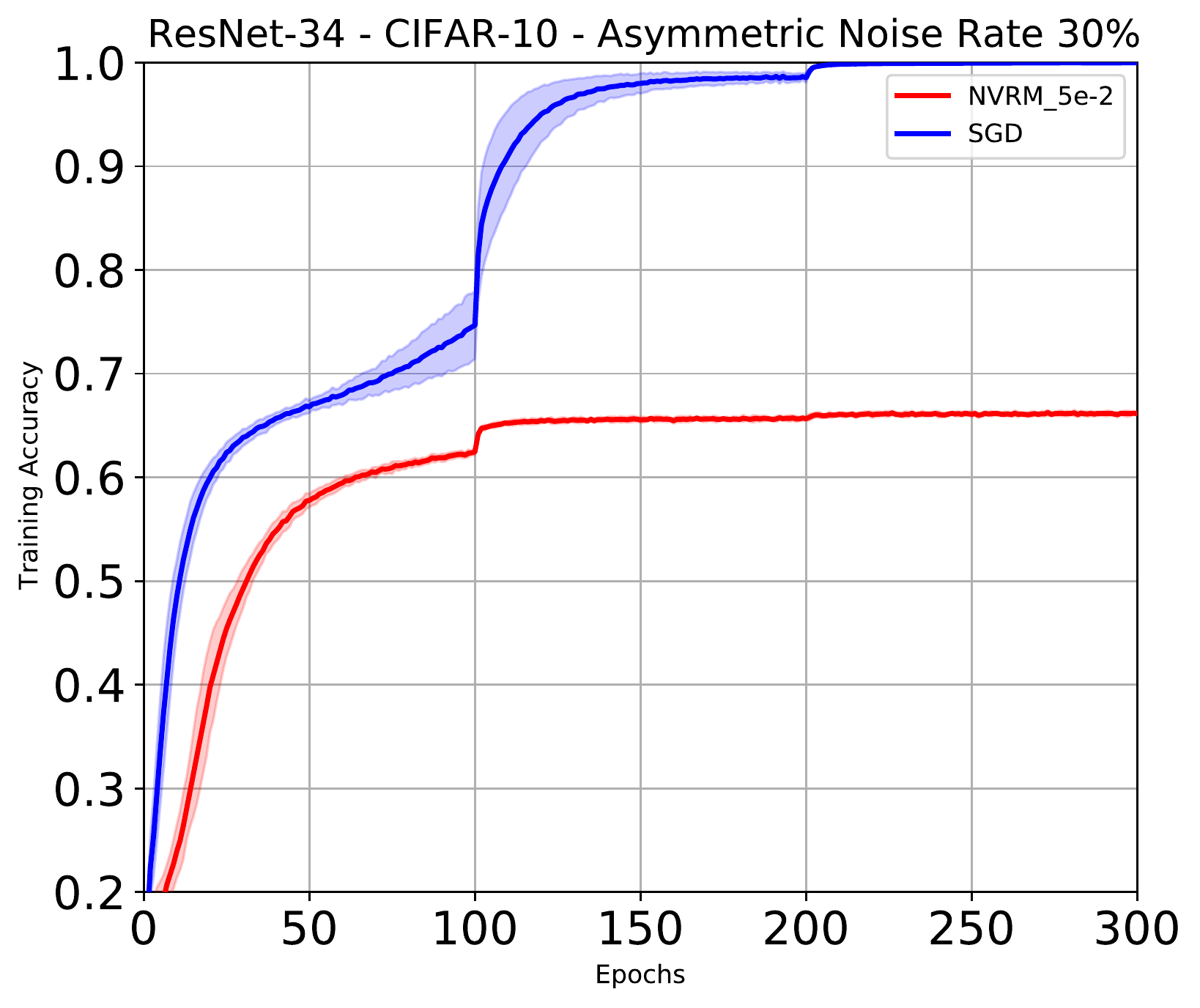}
    \includegraphics[width=0.2425\linewidth]{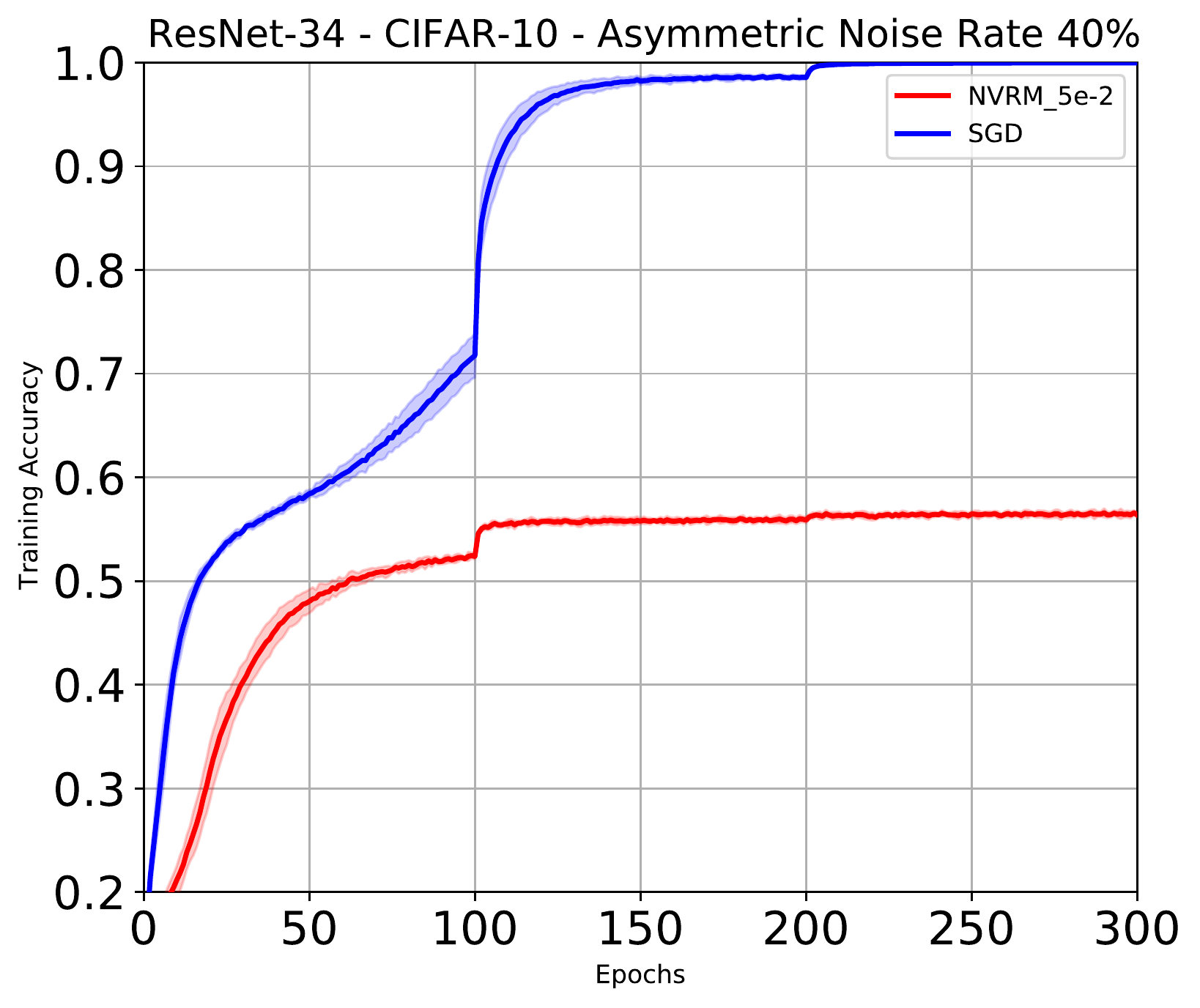}
    \caption{Training accuracy to epochs of ResNet-34 on CIFAR-10 with noisy labels.}
    \label{fig:C10_asymnoise}
  \end{subfigure} \\[0.5em]

  \begin{subfigure}{0.97\linewidth}
    \includegraphics[width=0.2425\linewidth]{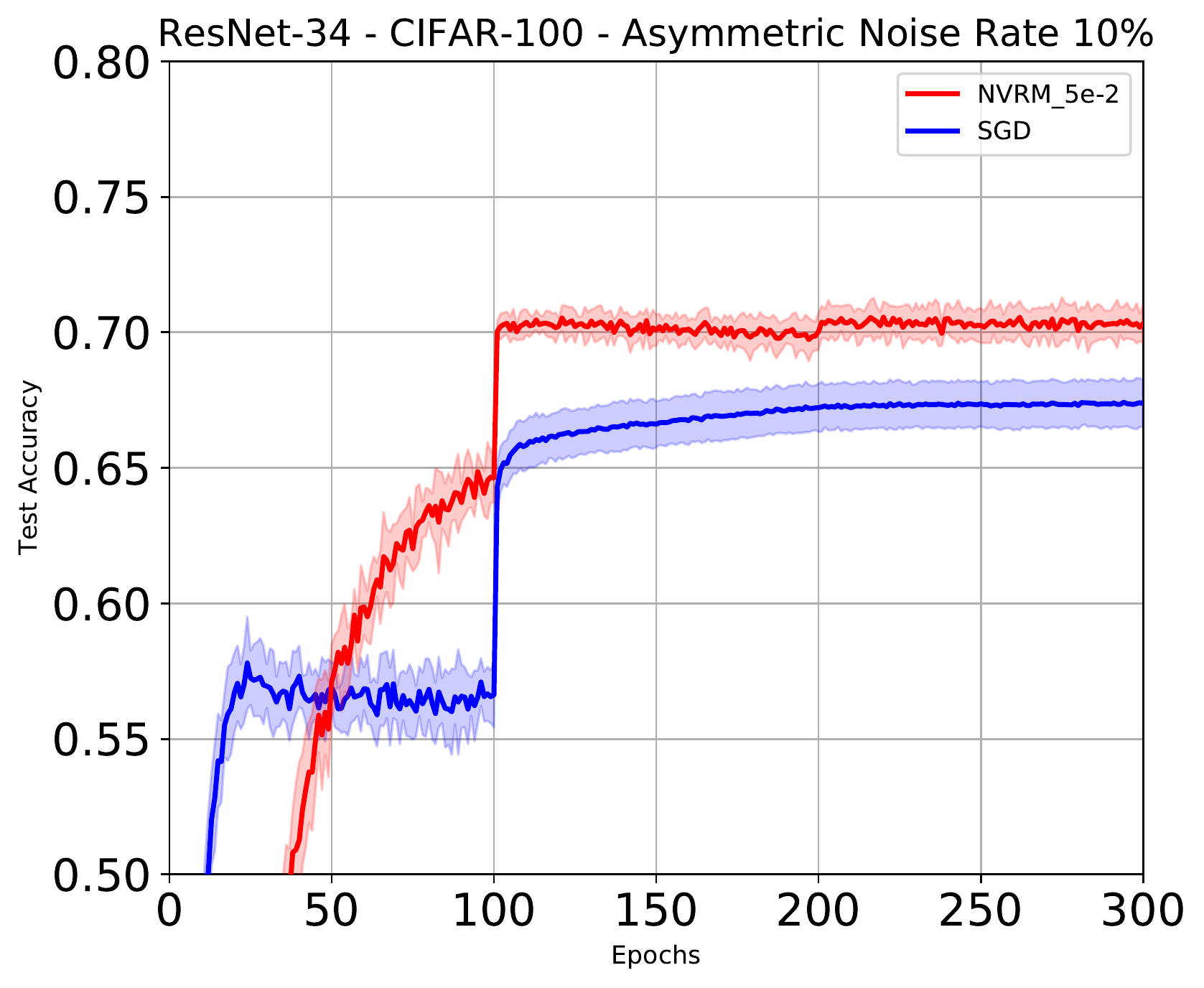}
    \includegraphics[width=0.2425\linewidth]{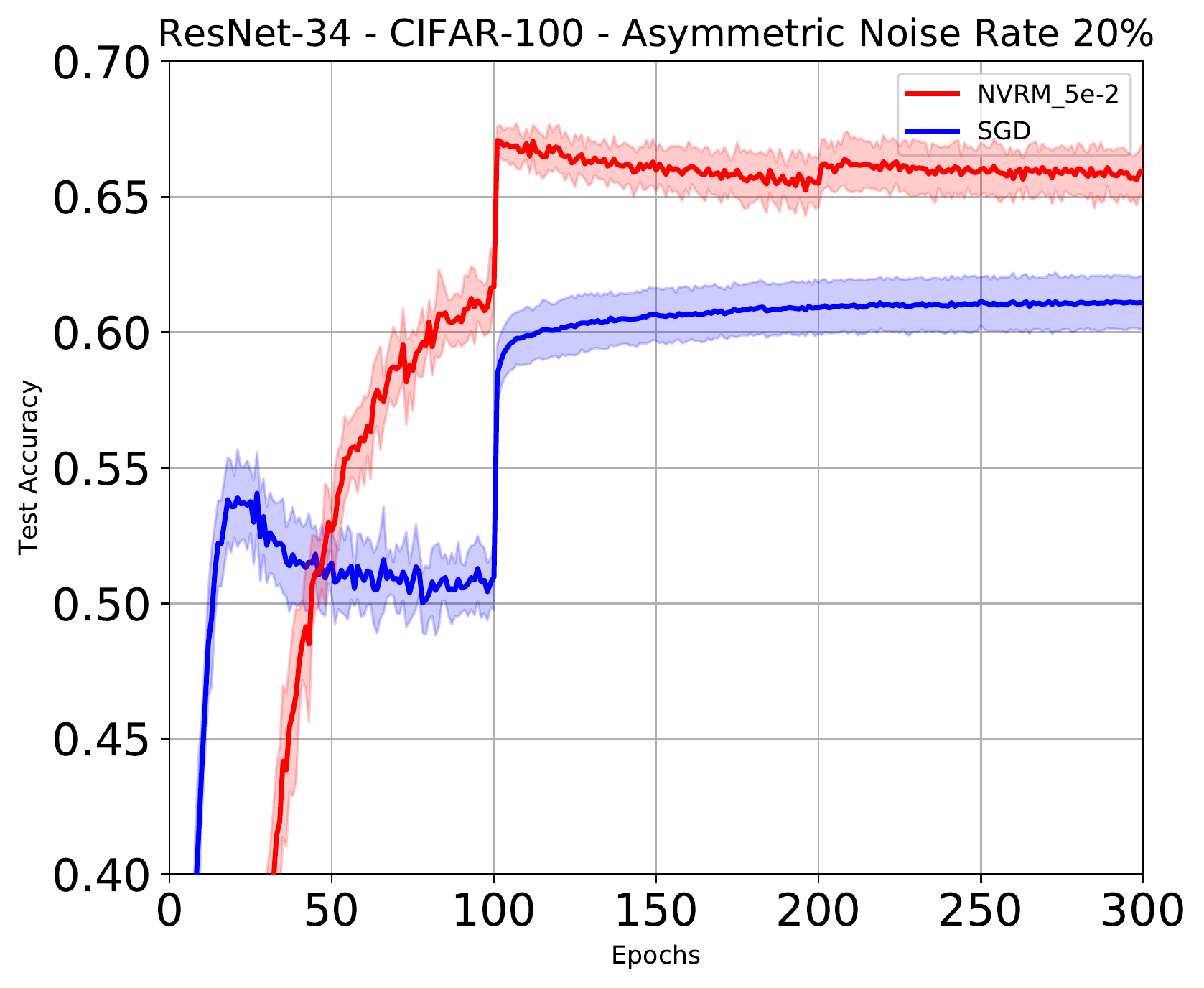}
    \includegraphics[width=0.2425\linewidth]{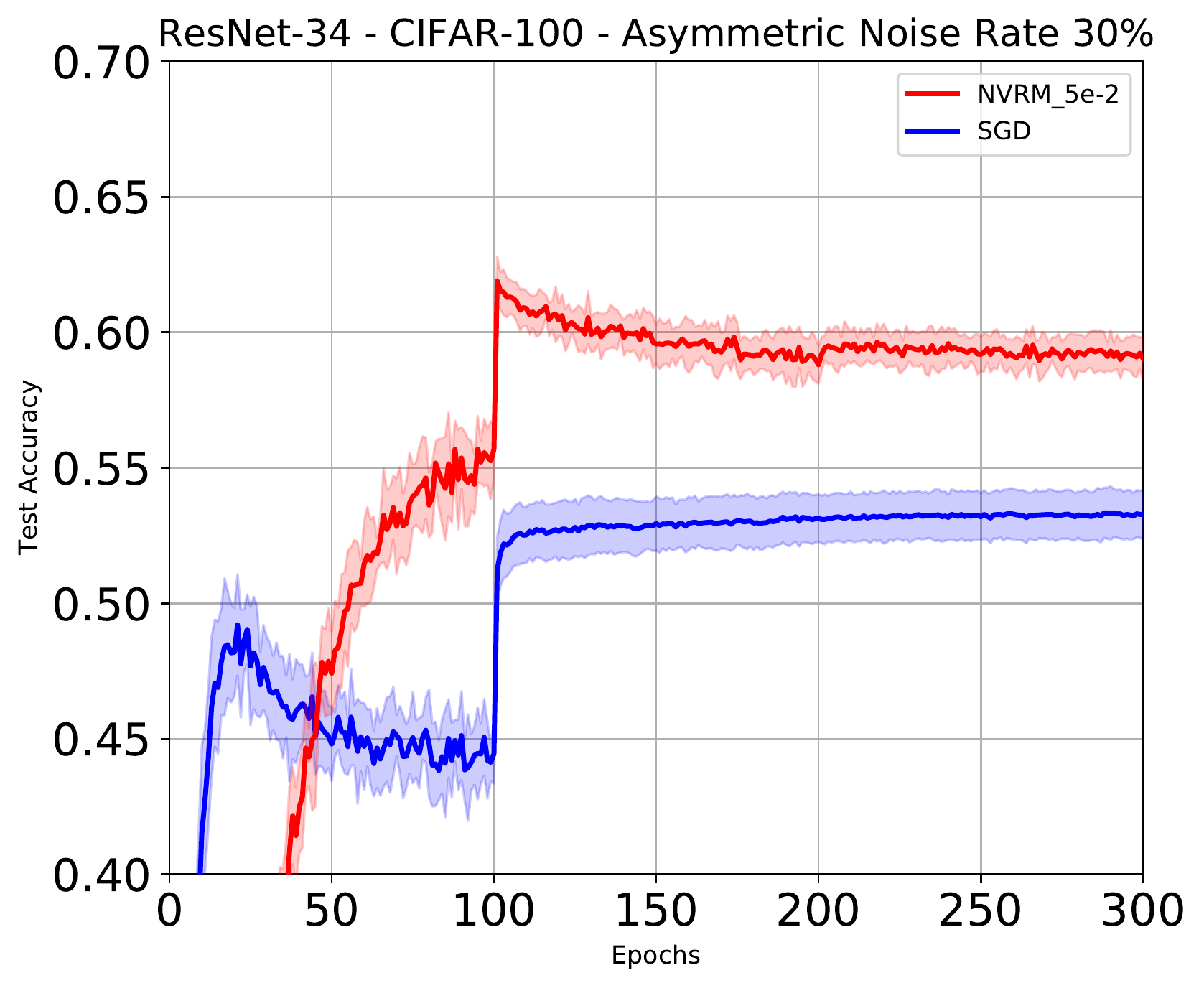}
    \includegraphics[width=0.2425\linewidth]{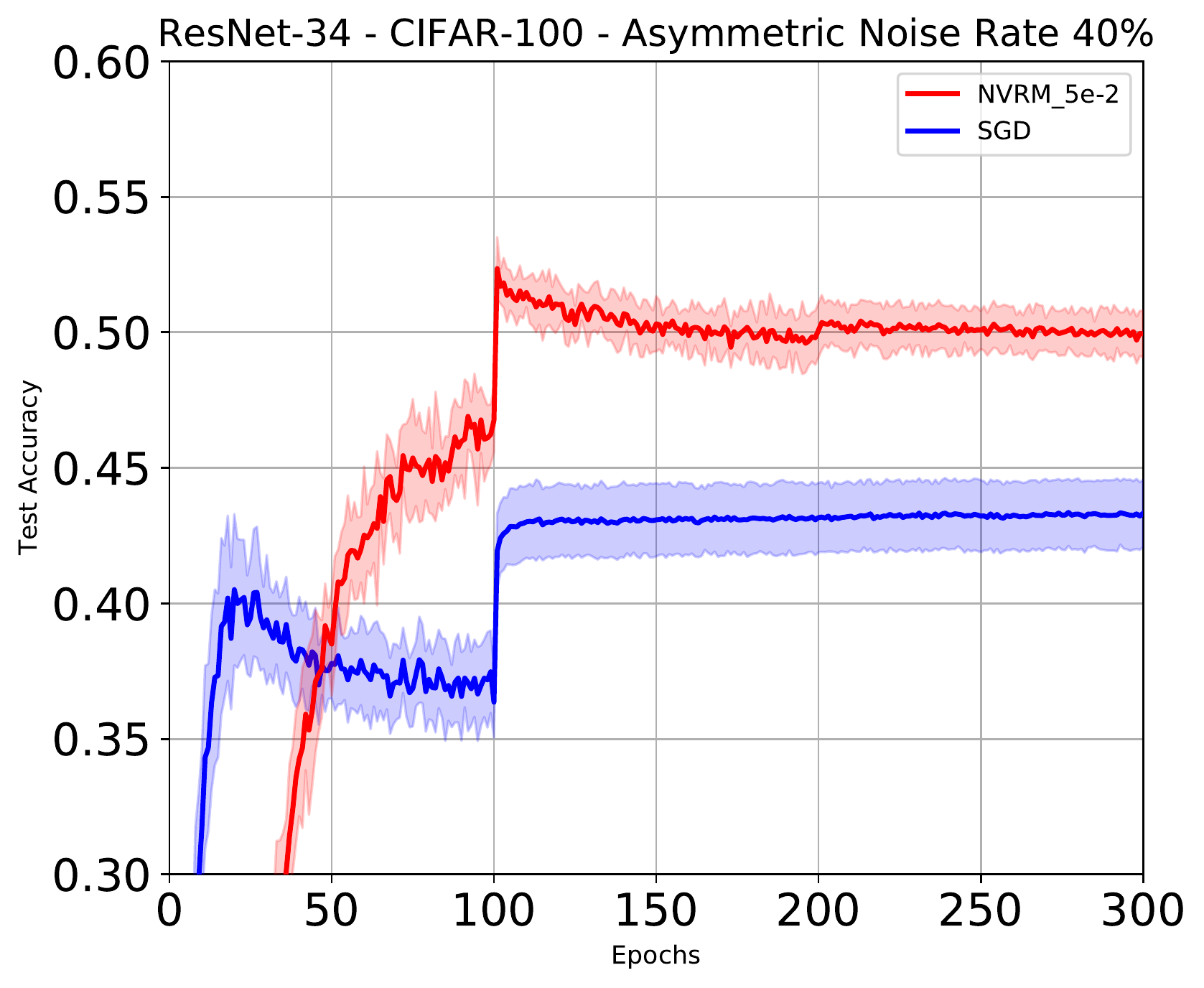}
    \caption{Test accuracy to epochs of ResNet-34 on CIFAR-100 with noisy labels.}
    \label{fig:C100_asymnoise}
  \end{subfigure}
  
    \begin{subfigure}{0.97\linewidth}
    \includegraphics[width=0.2425\linewidth]{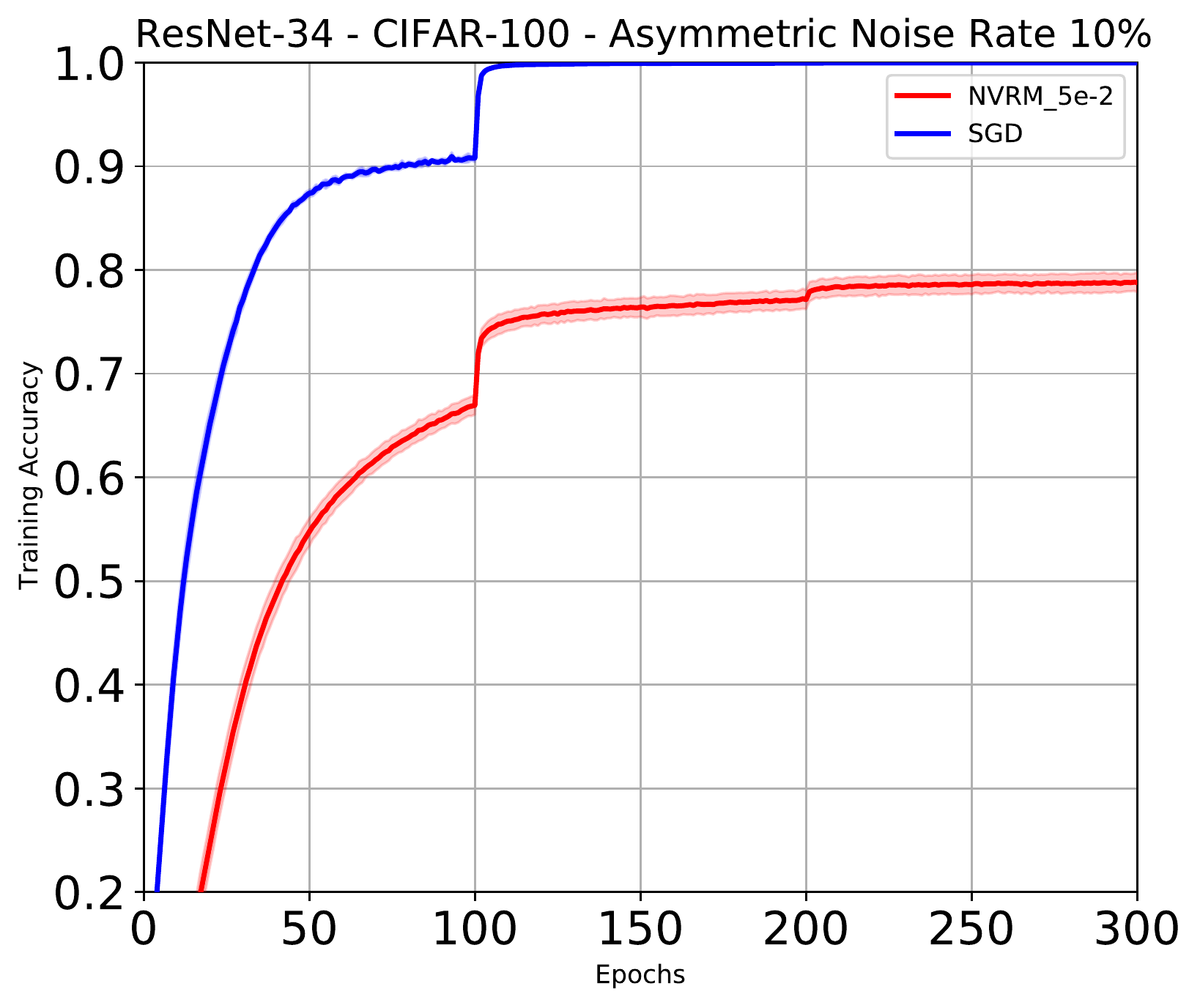}
    \includegraphics[width=0.2425\linewidth]{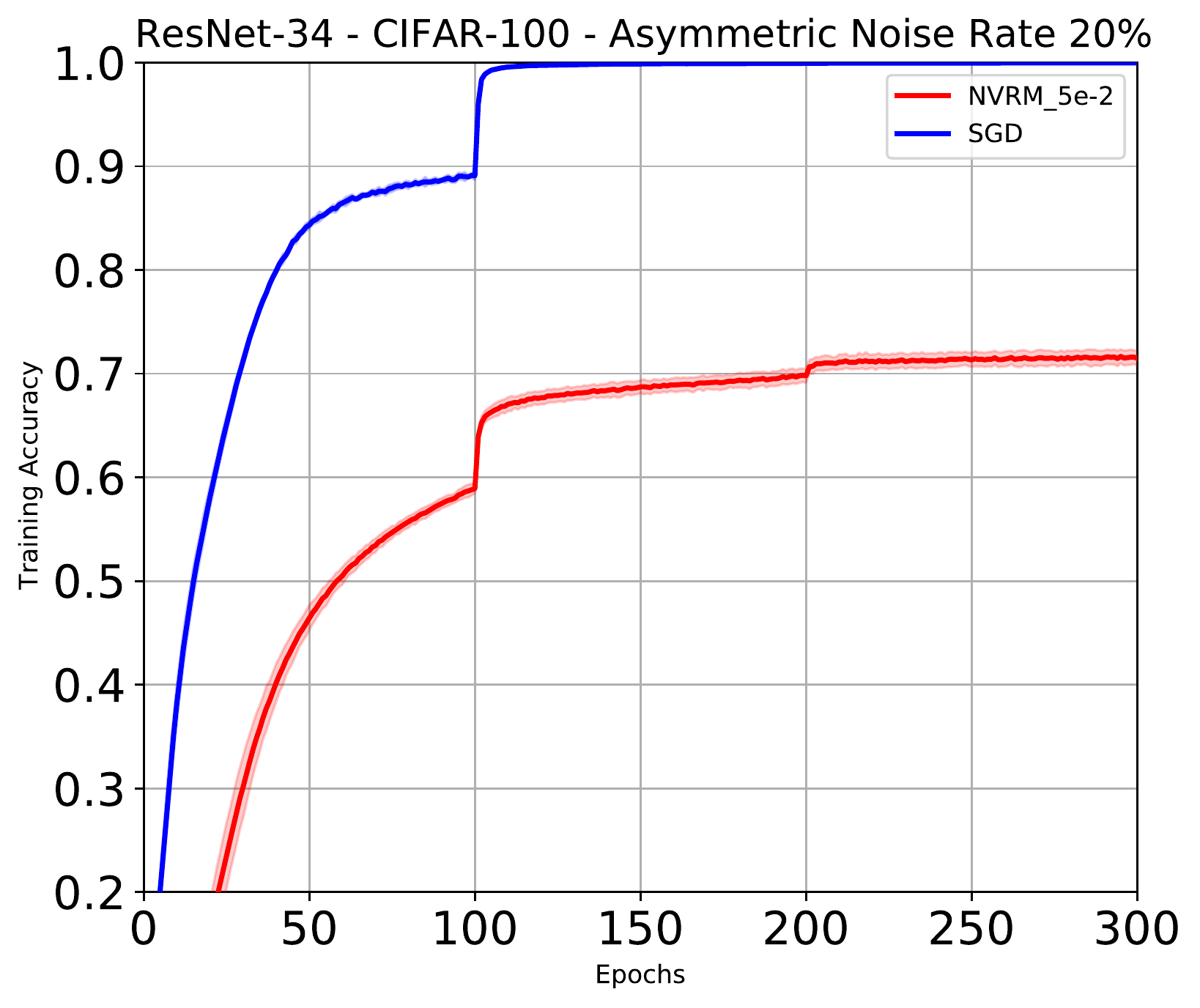}
    \includegraphics[width=0.2425\linewidth]{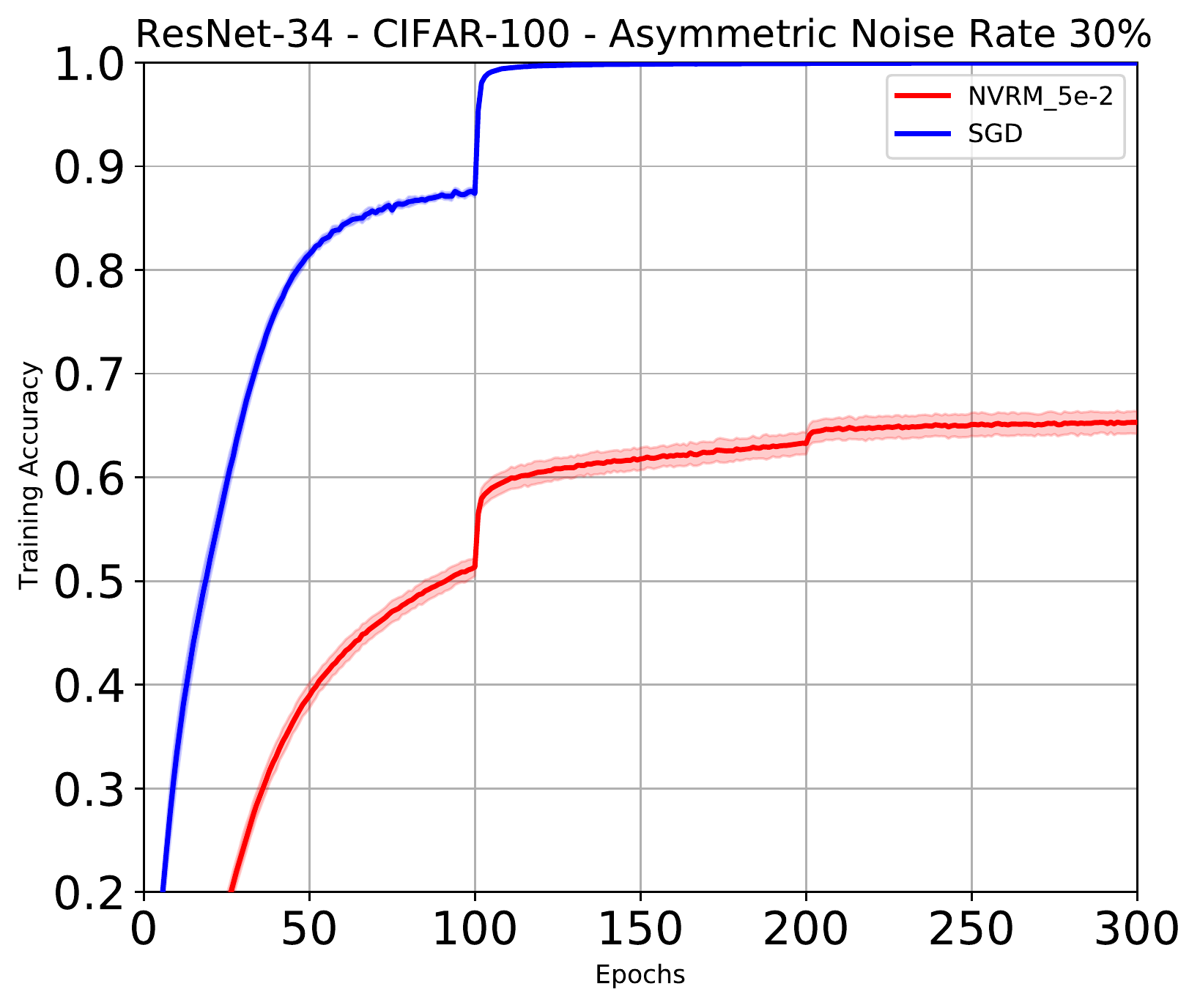}
    \includegraphics[width=0.2425\linewidth]{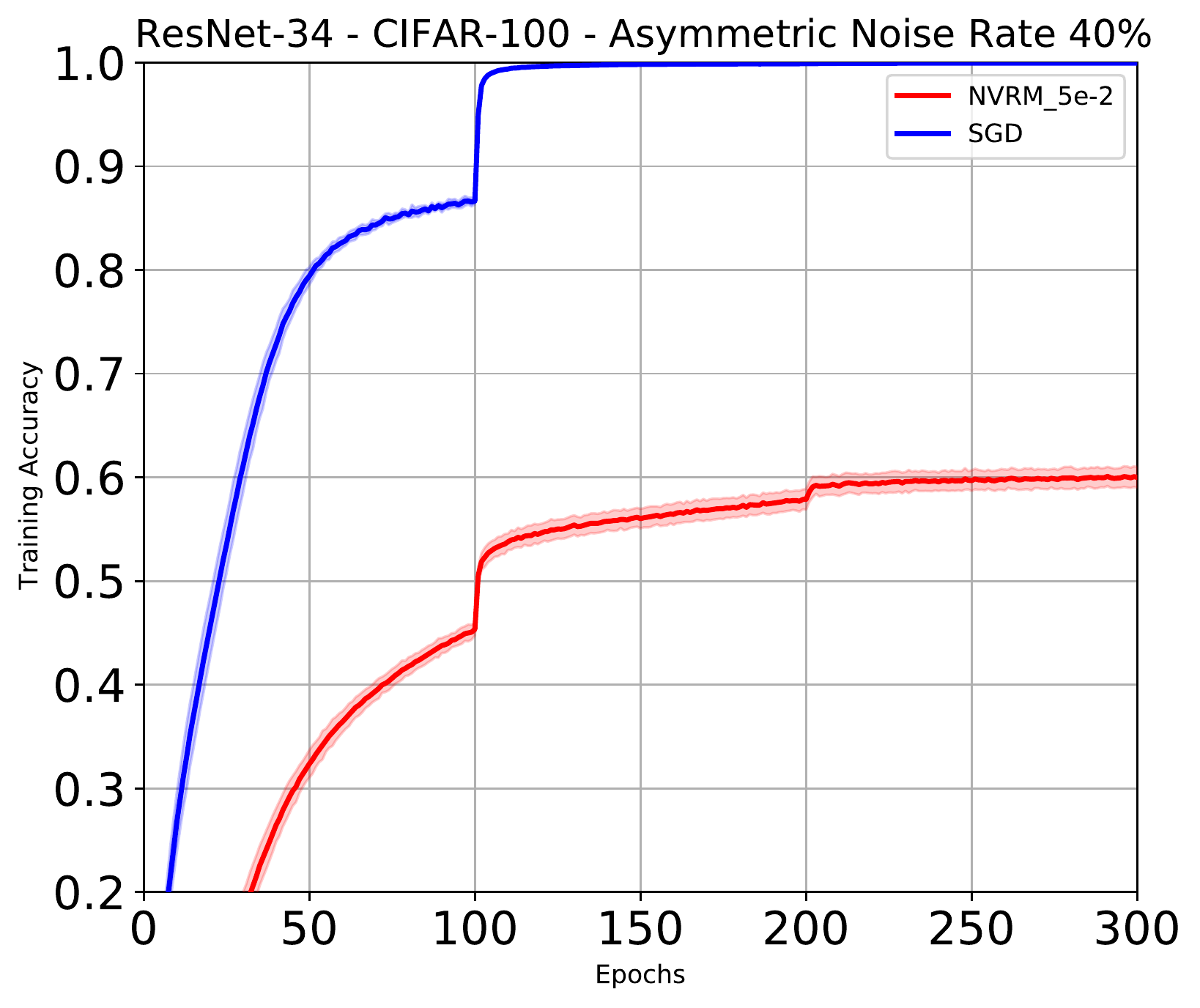}
    \caption{Training accuracy to epochs of ResNet-34 on CIFAR-100 with noisy labels.}
    \label{fig:C100_asymnoise}
  \end{subfigure}
  \caption{Curves of test accuracy to epochs of ResNet-34. NVRM with default $b=0.05$ can significantly relieve memorizing noisy labels. Particularly, NVMR stops learning when the training error is close to the label noise rate. SGD almost memorizes all noisy labels while NVRM almost only learns clean labels. The four columns are for asymmetric label noise rate $10\%$, $20\%$, $30\%$, and $40\%$ respectively.}
  \label{fig:resnet_asym}
\end{figure}

\paragraph{4. Robustness to catastrophic forgetting.} Model: Three-layer Fully-Connected Network (FCN). Dataset: Permuted MNIST \citep{lecun1998mnist}. Continual learning setting: FCN continually learns 5 tasks, and we made different random pixels permutation for each task. We evaluated the accuracy of the base task (the first task) and the mean accuracy of all learned tasks after each task. Figure \ref{fig:cfscale} shows that NVRM forgets the knowledge learned from previous task much slower than standard Empirical Risk Minimization. The empirical results demonstrate that the models learned under NVRM framework are significantly more robust to catastrophic forgetting. In Figure \ref{fig:cfewc}, we also verified that NVRM can enhance a popular neuroscience-inspired continual learning method, Elastic Weight Consolidation (EWC) \citep{kirkpatrick2017overcoming}. We present the incremental class learning task in Appendix \ref{app:exp}.

\begin{figure}
    \centering
        \begin{minipage}[t]{0.49\textwidth}
        \centering
    \includegraphics[width=0.47\linewidth]{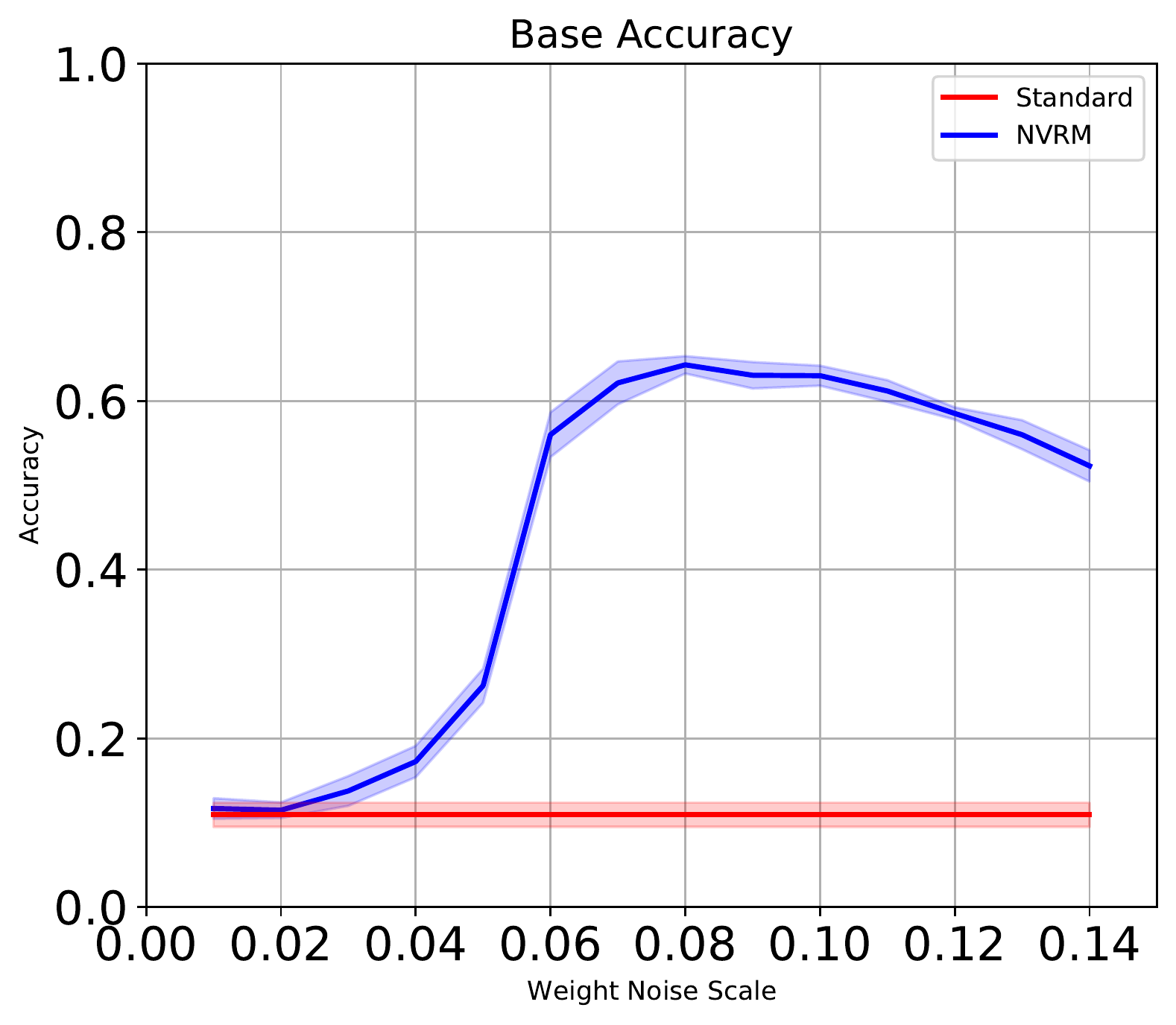}
    \includegraphics[width=0.47\linewidth]{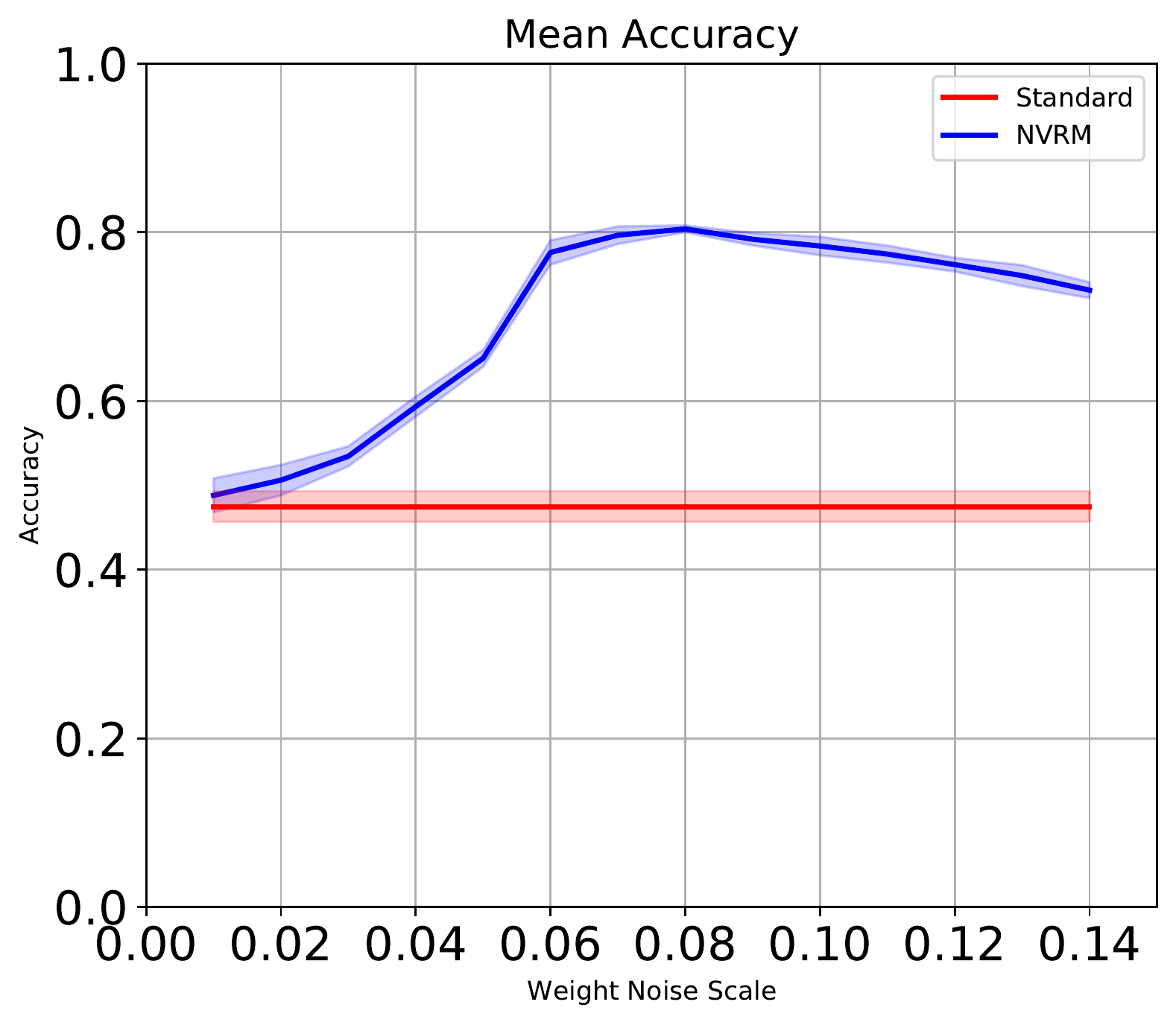}
      \caption{NVRM prevents catastrophic forgetting effectively with various variability scales $b$ (weight noise scales). Left: the accuracy of the first task after continually learning five tasks. Right: the mean accuracy of all five tasks after continually learning five tasks.}
  \label{fig:cfscale}
    \end{minipage} \hfill
     \begin{minipage}[t]{0.49\textwidth}
        \centering
        \includegraphics[width=0.49\linewidth]{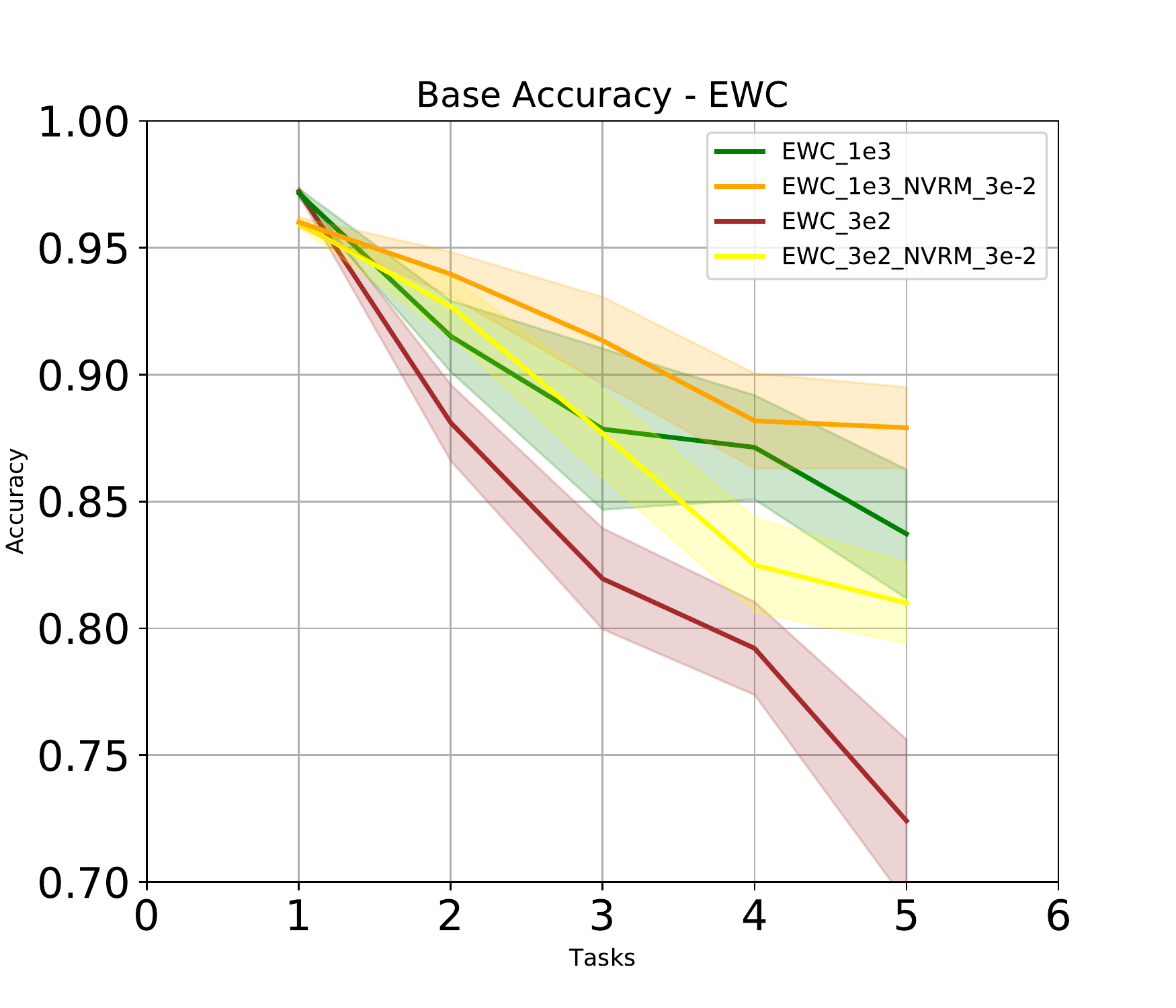}
        \includegraphics[width=0.49\linewidth]{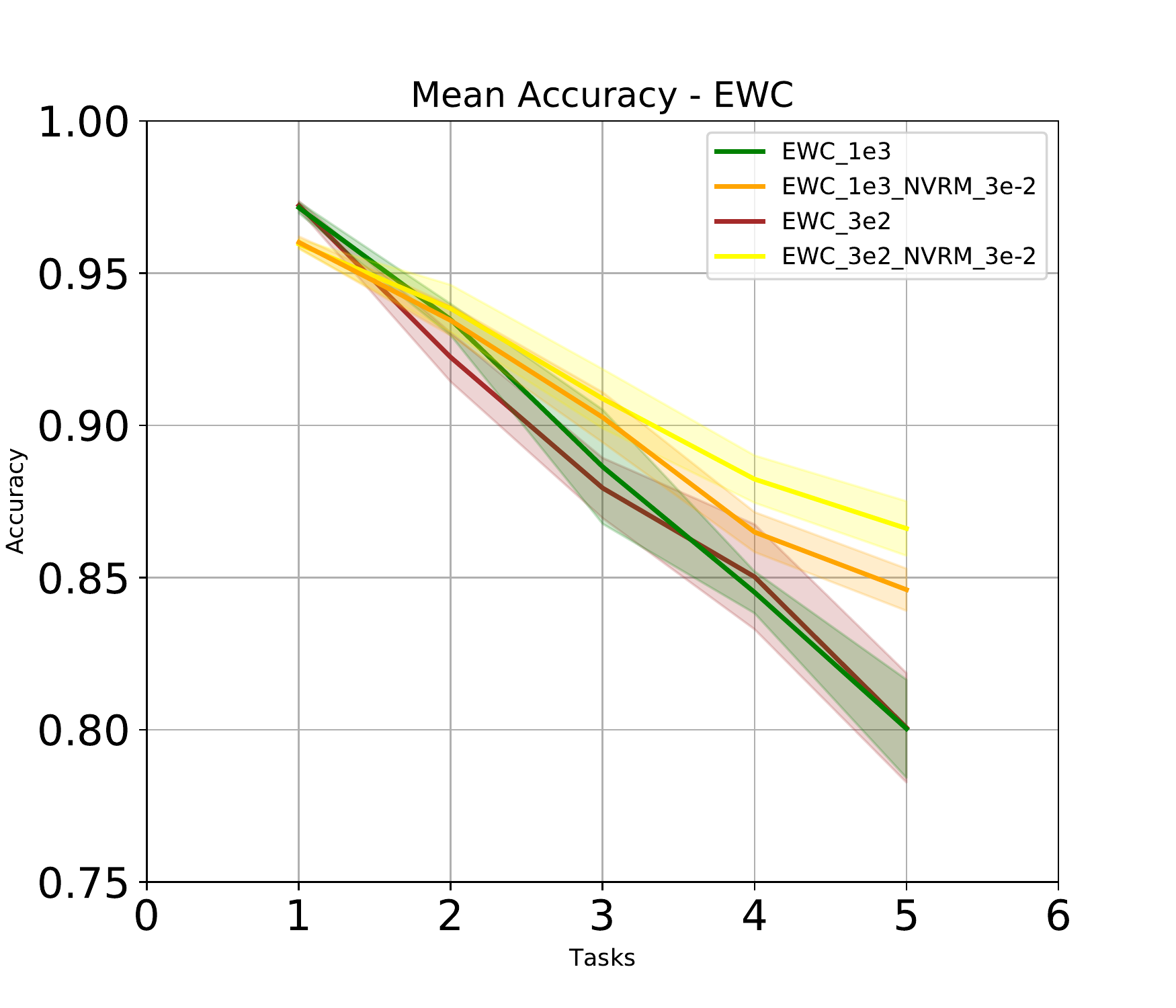}
        \caption{Curves of test accuracy to the number of tasks in continually learning Permuted MNIST. Letf: the accuracy of the base task with EWC. Right: the mean accuracy of all learned tasks with EWC. The importance hyperparameter of EWC is set to $300$ and $1000$. NVRM enhances EWC effectively.}
  \label{fig:cfewc}
    \end{minipage}
\end{figure}

\paragraph{5. Is the de-noising step really helpful?} We empirically compared the NVRM approach with PSGD, which uses a conventional noise injection method, on label noise memorization. We display the test errors of training ResNet34 on CIFAR-10 with $40\%$ asymmetric label noise under various variability scales/weight noise scales in Figure \ref{fig:psgdnoise}. The results demonstrate that, surprisingly, PSGD may prevent memorizing noisy labels much better than SGD, which has not been reported by existing papers yet. However, NVRM can still outperform PSGD significantly for learning with noisy labels. Thus, the denoising step in NVRM is not theoretically reasonable but also empirically powerful.

\begin{figure}
    \centering
        \begin{minipage}[t]{0.49\textwidth}
        \centering
 \includegraphics[width=0.7\linewidth]{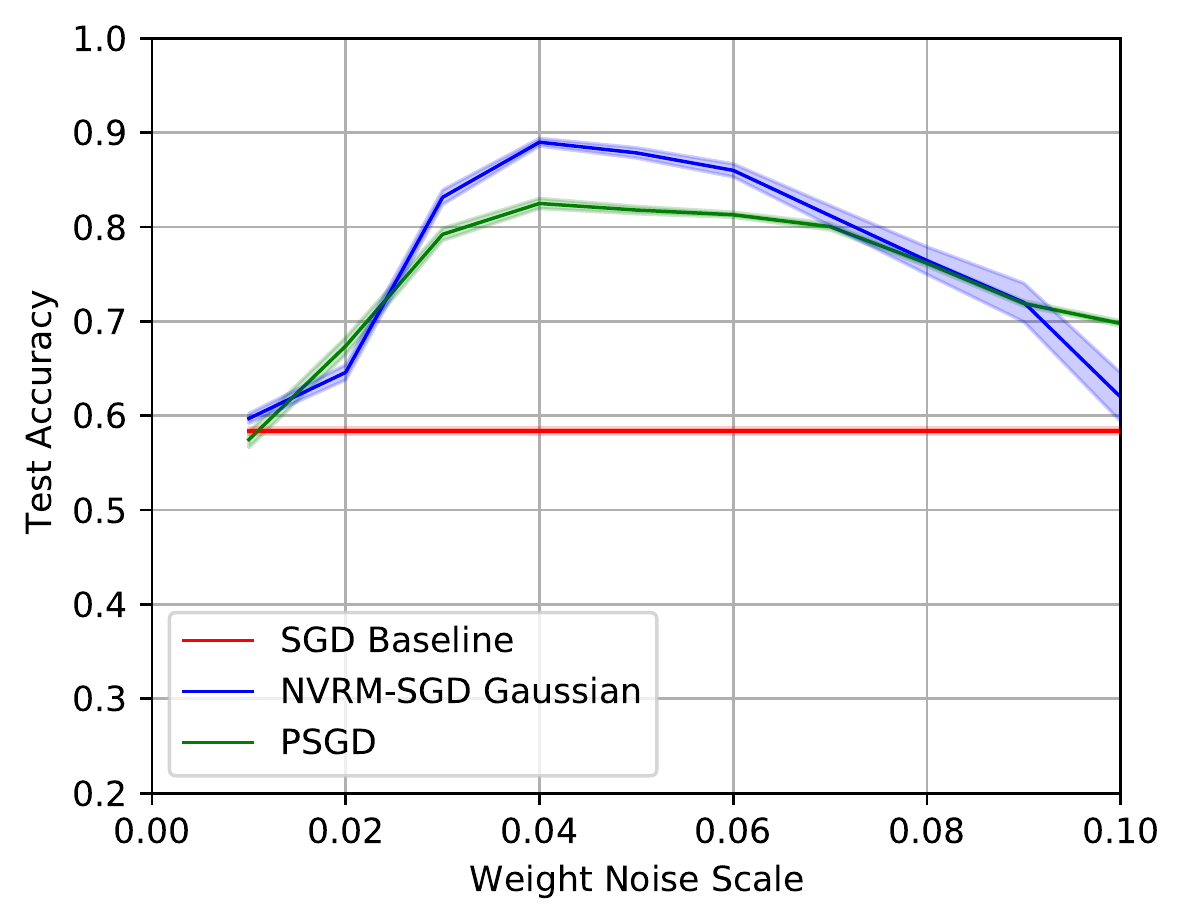}
  \caption{The de-noising step is helpful for preventing overfitting. Dataset: CIAF-10 with $40\%$ label noise. While we are the first to report that PSGD may prevent memorizing noisy labels much better than SGD, NVRM can still outperform PSGD significantly by nearly seven points. }
  \label{fig:psgdnoise}
    \end{minipage} \hfill
     \begin{minipage}[t]{0.49\textwidth}
        \centering
        \includegraphics[width=0.7\linewidth]{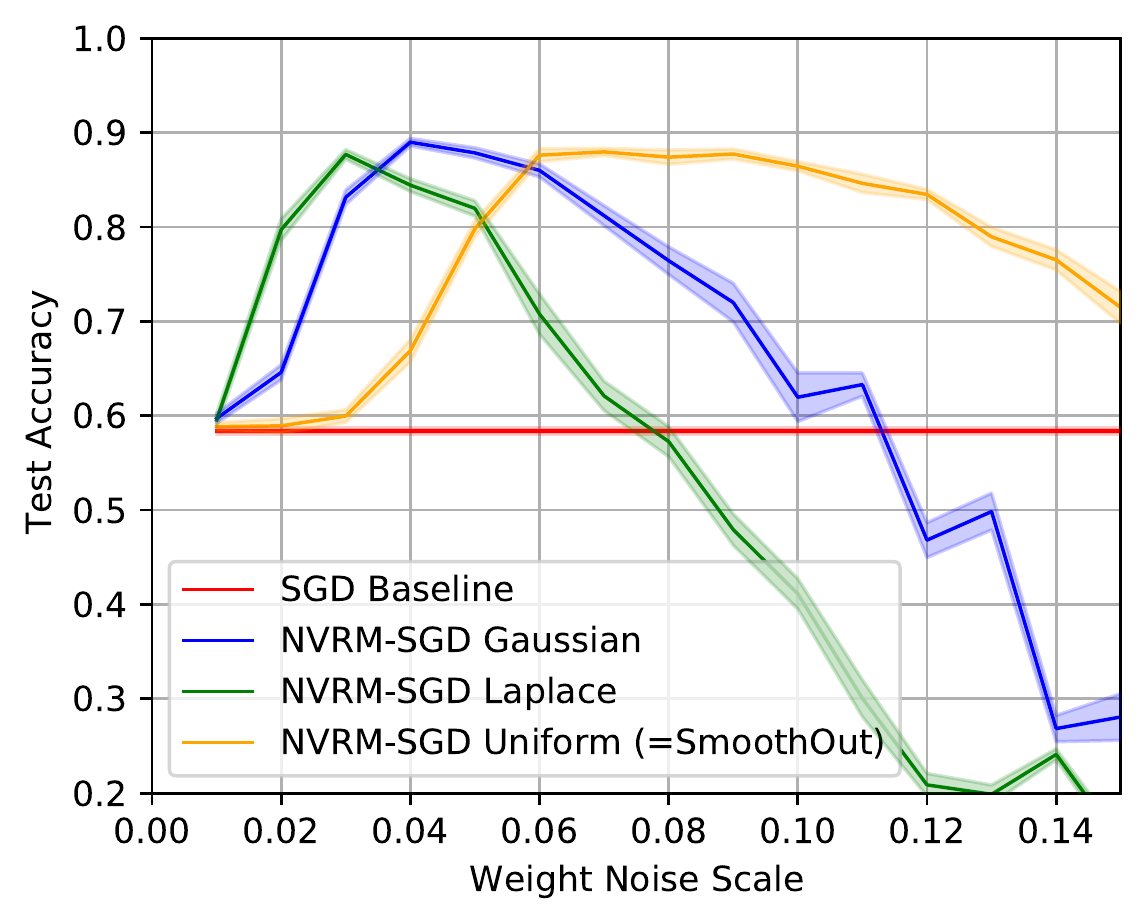}
        \caption{Choices of noise types for NVRM. Dataset: CIAF-10 with $40\%$ label noise. With a proper variability scale (weight noise scale) $b$, NVRM Gaussian outperforms NVRM Uniform and NVRM Laplace by nearly one point, while NVRM Uniform is more robust to the variability scale.}
  \label{fig:noisetype}
    \end{minipage}
\end{figure}

\paragraph{6. Choices of noise types.} As we mentioned above, we do not have to let $\epsilon$ be Gaussian. We consider the noise type as a hyperparamter, and empirically compare three common noise types, including Gaussian noise, Laplace noise, and uniform noise, on CIAF-10 with noisy labels. Because the tasks of learning with noisy labels can well reflect the ability to prevent overfitting. We display the test errors of training ResNet34 on CIFAR-10 under various variability scales $b$ with $40\%$ asymmetric label noise in Figure \ref{fig:noisetype}. The result demonstrates that, with a wide range of the variability hyperparameter $b$, NVRM with three noise types all can achieve remarkable improvements over the baseline SGD. This is not surprising because any of these noise may theoretically regularize the mutual information. Note that NVRM Uniform is identical to the original SmoothOut, which applies Uniform Smoothing to SGD. The original paper of SmoothOut \citep{wen2018smoothout} argued that uniform noise may be slightly better than Gaussian noise on clean datasets. However, existing work \citep{duchi2012randomized,wen2018smoothout} did not discover the ability of Randomized Smoothing or SmoothOut to learn with noisy label. Figure \ref{fig:noisetype} suggests that the conclusion of noise types is richer than \citet{wen2018smoothout} expected, and there is no ``free lunch''. We discovered that, the optimal test performance of NVRM Gaussian is better than NVRM Uniform and NVRM Laplace by nearly one point, while NVRM Uniform is more robust to the variability scale. It indicates that Gaussian noise and uniform noise have different advantages. 

\section{Conclusion}
\label{sec:conclusion}

A well-known term in neuroscience, named neural variability, suggests that the human brain response to the same stimulus exhibits substantial variability, and significantly contributes to balancing the accuracy and plasticity/flexibility in motor learning in natural neural networks. Inspired by this mechanism, this paper introduced ANV for balancing the accuracy and plasticity/flexibility in artificial neural networks. We proved that ANV acts as an implicit regularizer to control the mutual information between the training data and the learned model, which further secures preventing the learned model from overfitting and catastrophic forgetting. These two abilities are theoretically related to the robustness to weight perturbations. The proposed NVRM framework is an efficient approach to achieving ANV for artificial neural networks. The empirical results demonstrate that our method can (1) enhance the robustness to weight perturbation, (2) improve generalizability, (3) relieve the memorization of noisy labels, and (4) mitigate catastrophic forgetting. Particularly, NVRM, an optimization approach, may handle memorization of noisy labels well at negligible computational and coding costs. One line code of importing a neural variable optimizer is all you need to achieve ANV for your models.

\section*{Acknowledgement}
MS was supported by the International Research Center for Neurointelligence (WPI-IRCN) at The University of Tokyo Institutes for Advanced Study.

\bibliography{deeplearning}

\appendix

\section{Proofs}
\label{app:proofs}

\subsection{Proof of Theorem \ref{pr:pacgen}}
\begin{proof}
We consider a distribution $Q_{\mathrm{nv}}$ over predictors with weights of the form $\theta+\epsilon$, where $\theta$ is drawn from the distribution $Q$ and $\epsilon \sim \mathcal{N}(0, b^{2}I)$ is a random variable indicating weight perturbation. Suppose the model $\mathcal{M}(\theta)$ achieves $(b,\delta(\theta))$-neural variability, which follows the notation of the expected sharpness used by \citet{neyshabur2017exploring}. We start our theoretical analysis based on Equation 7 of \citet{neyshabur2017exploring}. We can bound the expected risk over the distribution $Q_{\mathrm{nv}}$ as:
 \begin{align}
L(Q_{\mathrm{nv}}) \leq&  \hat{L}(Q) + [\hat{L}(Q_{\mathrm{NV}}) -  \hat{L}(Q) ] + 4\sqrt{\frac{1}{m} [\KL(Q_{\mathrm{nv}}\| P) + \ln(\frac{2m}{\Delta})]}, \\
& =  \hat{L}(Q) +  4\sqrt{\frac{1}{m} [\KL(Q_{\mathrm{nv}}\| P) + \ln(\frac{2m}{\Delta})]} + \mathbb{E}_{\theta\sim Q}[\delta(\theta)].
\end{align}
We emphasis that this bound holds for any distribution $Q$ (any method of choosing $\theta$ dependent on the training dataset) and any prior $P$. We use a very special distribution $Q$:$Pr(\theta=\theta^{\star})=1$. Thus we can bound the expected risk over the distribution $Q_{\mathrm{nv}}$ as
 \begin{align}
L(Q_{\mathrm{nv}}) \leq \hat{L}(\theta^{\star}) +  4\sqrt{\frac{1}{m} [\KL(\theta^{\star} + \epsilon \| P) + \ln(\frac{2m}{\Delta})]} + \delta(\theta^{\star}).
\end{align}
The Kullback-Leibler divergence of the two Gaussians can be written as
\begin{align}
\KL(\theta^{\star} + \epsilon\| P) = \sum_{i=1}^{N} \left[  \log\left(\frac{\sigma}{b}\right) + \frac{b^{2} + \theta_{i}^{\star 2}}{2\sigma^{2}}  - \frac{1}{2}\right], 
\end{align}
where $N$ is the number of the model weights. Finally, we have
\begin{align}
L(Q_{\mathrm{nv}}) \leq \hat{L}(\theta^{\star}) +  4\sqrt{\frac{1}{m} \left[ \sum_{i=1}^{N} \left[  \log\left(\frac{\sigma}{b}\right) + \frac{b^{2} + \theta_{i}^{\star 2}}{2\sigma^{2}}  - \frac{1}{2}\right] + \ln(\frac{2m}{\Delta})\right]} + \delta. 
\end{align}

\end{proof}

\section{The Mutual-Information Generalization Bound}
\label{app:mibound}

We formulate a mutual information theoretical foundation of $(b, \delta)$-neural variability, which is more related to the neuroscience mechanism of penalizing the information carried by nerve impulses \citep{stein2005neuronal}. 

It is known that the information in the model weights relates to overfitting \citep{hinton1993keeping} and flat minima \citep{hochreiter1997flat}. According to Lemma \ref{pr:geninfo}, if the mutual information of the parameters and data decreases, the upper bound of the expected generalization gap will also decrease. 

 \begin{lemma}[\citep{xu2017information}]
 \label{pr:geninfo}
 Suppose $L(\theta, (x,y))$ is the loss function of the model $\mathcal{M}(\theta)$, such that $L(\theta, (x,y))$ is $\sigma$-sub-gaussian random variable for each $\theta$. Let the training dataset $S= \{(x^{(i)},y^{(i)})\}_{i=1}^{m} $ and the test sample $\bar{S} = (\bar{x}, \bar{y})$ be sampled from the data distribution $\mathcal{S}$ independently, and $\theta$ be the model weights learned from the algorithm $\mathcal{A}(\theta|S)$. Then the expected generalization gap meets the following property:
\[  \mathbb{E}\left[ L( \theta, (\bar{x}, \bar{y})) - L(\theta, S) \right]   \leq \sqrt{\frac{2\sigma^{2}}{m} I(\theta; S)}, \]
where $I(\theta; S)$ denotes the mutual information between the parameters $\theta$ and the training dataset $S$. 
\end{lemma}

 \begin{theorem}
 \label{pr:nvgen}
Suppose the conditions of Lemma \ref{pr:geninfo} hold, and the model $\mathcal{M}(\theta)$ achieves $(b, \delta)$-NV on the training dataset $S$. Then the expected generalization gap of the model $\mathcal{M}(\theta)$ satisfies:
\[  \mathbb{E}\left[ L( \theta + \epsilon , (\bar{x}, \bar{y})) - L(\theta, S)  \right]   \leq \sqrt{\frac{2\sigma^{2}}{m} I(\theta + \epsilon; S)} + \delta , \]
where $\epsilon \sim \mathcal{N}(0,b^{2}I)$ is Gaussian noise, and $\delta$ only depends on the training loss landscape.
 \end{theorem}

\begin{proof}
Given the model $\mathcal{M}(\theta)$, we can easily obtain a new model $\mathcal{M}(\theta+\epsilon)$ close to $\mathcal{M}(\theta)$ by injecting a Gaussian noise $\epsilon \sim \mathcal{N}(0, b^{2}I)$. By Lemma \ref{pr:geninfo}, we have the expected generalization gap of $\mathcal{M}(\theta+\epsilon)$ meets
\begin{align}
\mathbb{E}\left[ L( \theta+\epsilon, (\bar{x}, \bar{y})) - \frac{1}{m} \sum_{i=1}^{m} L(\theta+\epsilon, (x^{(i)}, y^{(i)})) \right]   \leq \sqrt{\frac{2\sigma^{2}}{m} I(\theta+\epsilon; S)}.
\end{align}
Based on the definition of $(b,\delta)$-NV, we have 
\begin{align}
\mathbb{E}\left[ \frac{1}{m} \sum_{i=1}^{m} L(\theta+\epsilon, (x^{(i)}, y^{(i)})) -  \frac{1}{m} \sum_{i=1}^{m} L(\theta, (x^{(i)}, y^{(i)})) \right] \leq \delta.
\end{align}
Thus, we obtain 
\begin{align}
\mathbb{E}\left[ L( \theta + \epsilon, (\bar{x}, \bar{y})) - \frac{1}{m} \sum_{i=1}^{m} L(\theta, (x^{(i)}, y^{(i)})) \right]   \leq \sqrt{\frac{2\sigma^{2}}{m} I(\theta + \epsilon; S)} + \delta.
\end{align}

\end{proof}
 
Obviously, the bound monotonically decreases with the variability scale $b$ given $\delta$. The bound is tighter than the bound in Lemma \ref{pr:geninfo} when the model has good ANV, which means $b$ is large given a small $\delta$. A large $b$ can even penalize the mutual information to nearly zero. Therefore, strong ANV brings a tighter generalization bound.

\section{Implementation Details}
\label{app:exp}

We introduce the details of each experiment in this section. 
In Experiment 1, we evaluated NVRM's robustness to weight perturbation. 
In Experiment 2, we evaluated the generalizability of NVRM. 
In Experiment 3, we evaluated NVRM's robustness to noisy labels. 
In Experiment 4, we evaluated NVRM's robustness to catastrophic forgetting. 
In Experiments 5, we studied the usefulness of the de-noising step.
In Experiments 6, we studied choices of noise types.

Our experiment is conducted on a computing cluster with GPUs of NVIDIA\textsuperscript{\textregistered} Tesla\textsuperscript{\texttrademark} V100 16GB and CPUs of Intel\textsuperscript{\textregistered} Xeon\textsuperscript{\textregistered} Gold 6140 CPU @ 2.30GHz.

\subsection{Robustness, Generalization, and Label Noise}

\textbf{General Settings.} Experiments are conducted based on three popular deep learning networks, VGG-16 \citep{Simonyan15} , MobileNetV2 \citep{sandler2018mobilenetv2} and ResNet-34 \citep{he2016deep}. The detailed architectures are presented in Table \ref{table:arch}. Similarly, all datasets involved in our experiments are generated based on two standard benchmark datasets, CIFAR-10 and CIFAR-100 \citep{krizhevsky2009learning}.\footnote{Please download them from \href{https://www.cs.toronto.edu/kriz/cifar.html}{https://www.cs.toronto.edu/kriz/cifar.html}.} We follow the official version to split training sets and test sets in our experiments. For pre-processing and data augmentation, we performed per-pixel mean subtraction, horizontal random flip and $32 \times 32$ random crops after padding with $4$ pixels on each side. The batch size is set as $128$, and the weight decay factor is set as $0.0001$. We selected the optimal learning rate from $\{0.0001,0.001,0.01,0.1,1,10\}$ and used $0.1$ for SGD/NVRM-SGD. Note that we used the common $L_{2}$ regularization as weight decay which is widely used in most cases, while \citet{loshchilov2018decoupled,xie2020stable} suggested that decoupled weight decay or stable weight decay is better in adaptive gradient methods. We employ SGD and NVRM-SGD to train models, unless we specify it otherwise. For the learning rate schedule, we initialized the learning rate as $0.1$ and divided it by $10$ after every $100$ epochs. All models are trained for $300$ epochs. The momentum factor is set as $0$ for VGG-16 and MobileNetV2 in Experiment 1, and $0.9$ for ResNet-34 in Experiment 2-3.

\textbf{Robustness to weight perturbation.} For Experiment 1, we injected isotropic Gaussian noise of different variances to all the model weights and then evaluate the changes of the test accuracy. Six different noise scales $\{0.01, 0.012, 0.014, 0.016, 0.018, 0.02\}$ are involved in our experiments.

\textbf{Learning with noisy labels.} For Experiment 3, we also generate a group of datasets with label noise. The symmetric label noise is generated by flipping every label to other labels with uniform flip rates $\{ 20\%, 40\%, 60\%, 80\% \}$. The asymmetric label noise by flipping label $i$ to label $i+1$ (except that label 9 is flipped to label 0) with pair-wise flip rates $ \{ 10\%, 20\%, 30\%, 40\% \}$. We employed the code of \citet{han2018co} for generating noisy labels for CIFAR-10 and CIFAR-100.

\textbf{The usefulness of the de-noising step and choices of noise types.} The hyperparameter settings of Experiments 5 and 6 follows Experiment 3, which are performed on learning with noisy labels. In Experiment 6, we let the weight noise $\epsilon$, respectively, obeys $ \mathcal{N}(0, b^{2})$, $ Laplace(0, b)$, $Uniform(-b, b)$ for NVRM Gaussian, NVRM Laplace, and NVRM Uniform. 

\subsection{Catastrophic Forgetting}

\textbf{Permuted MNIST.} For Experiment 4, we used Fully-Connected Network (FCN), which has two hidden layers and 1024 ReLUs per hidden layer. As continual learning tasks usually employ adaptive optimizers, we compared Adam with NVRM-Adam on the popular benchmark task, Permuted MNIST. In Permuted MNIST, we have five continual tasks. For each task, we generated a fixed, random permutation by which the input pixels of all images would be shuffled. Each task was thus of equal difficulty to the original MNIST problem, though a different solution would be required for each. We evaluated the accuracy of the base task (the first task) and the mean accuracy of all learned tasks after each task. 

In the experiment of EWC, we try to validate if NVRM can improve EWC. We validated the performance improvements under two different importance hyperparameters $\lambda \in \{30, 1000\}$. In Experiment 4, the batch size is set as $256$, and the weight decay factor is set as $0.0001$. As continual learning methods usually prefer adaptive optimizers, we employed Adam and NVRM-Adam as the optimizers. For the learning rate schedule, we fixed the learning rate as $0.001$ and applied no learning rate decay. We set the variability scale $b=0.03$ in NVRM-Adam. All models are trained for one epoch per task, as one-epoch training has ensured good test performance on newly learned tasks.

\textbf{Split MNIST.} For Experiment 4, we also supplied the experiment on split MNIST, which is another classical continual learning task. It is called incremental class learning. We train the models on the samples with a specific subset of labels for five continual tasks. We followed the usual setting \citep{zenke2017continual}: $y \in \{0,1\}$, $y \in \{2,3\}$ , $y \in \{4,5\}$ , $y \in \{6,7\}$ , and $y \in \{8,9\}$ for five tasks, respectively. In each task, the model may only learn two new digits and may forget previously learned digits. 

The model is the same as the model architecture for Permuted MNIST, except that we used the five-header output layers for five tasks, respectively. When we trained the models for one task, the headers for other tasks are frozen. The batch size is set as $256$, and the weight decay factor is set as $0$. Again, we employed Adam and NVRM-Adam as the optimizers, and new optimizers are used for each continual task. For the learning rate schedule, we fixed the learning rate as $0.001$ and applied no learning rate decay. We also let the variability scale $b=0.03$ in NVRM-Adam, unless we otherwise specify it.

\begin{table*}[t]
\caption{The detailed architectures of models used in the experiments. ``conv $x$ - $c$" represents a convolution layer with kernel size $x\times x$ and $c$ output channels, and ``fc - $c$" represents a fully-connected layer with $c$ output channels. In the architecture of MobileNetV2, $[\cdot]$ represents a ``bottleneck", and $(\cdot)$ is simply a combination of three convolution layers but can halve both the width and height of the input of the block. The $k$ in $[\cdot]$ or $(\cdot)$ denotes the number of channels of the input of the corresponding block. In the architecture of ResNet-34, $[\cdot]$ represents a ``basic block".
}
\label{table:arch}
\begin{center}
\small
\begin{sc}
\resizebox{0.4\textwidth}{!}{%
\begin{tabular}{ccc}
\toprule
VGG-16 & MobileNetV2 & ResNet-34 \\
\hline
conv3-64 $\times$ 2 & fc-32 & conv3-64 \\

maxpool & $\left[ \begin{matrix} \text{conv1-$k$} \\ \text{conv3-$k$} \\ \text{conv1-16} \end{matrix} \right]$ $\times$ 1 & $\left[ \begin{matrix} \text{conv3-64} \\ \text{conv3-64} \end{matrix} \right]$ $\times$ 3 \\

conv3-128 $\times$ 2 & $\left[\begin{matrix} \text{conv1-$6k$} \\ \text{conv3-$6k$} \\ \text{conv1-24} \end{matrix} \right]$ $\times$ 2 & $\left[\begin{matrix} \text{conv3-128} \\ \text{conv3-128} \end{matrix} \right]$ $\times$ 4 \\

maxpool & $\left(\begin{matrix} \text{conv1-$6k$} \\ \text{conv3-$6k$} \\ \text{conv1-32} \end{matrix} \right)$ $\times$ 3 & $\left[\begin{matrix} \text{conv3-256} \\ \text{conv3-256} \end{matrix} \right]$ $\times$ 6 \\

conv3-256 $\times 3$ & $\left(\begin{matrix} \text{conv1-$6k$} \\ \text{conv3-$6k$} \\ \text{conv1-64} \end{matrix} \right)$ $\times$ 4 & $\left[\begin{matrix} \text{conv3-512} \\ \text{conv3-512} \end{matrix} \right]$ $\times$ 3 \\

maxpool & $\left[\begin{matrix} \text{conv1-$6k$} \\ \text{conv3-$6k$} \\ \text{conv1-96} \end{matrix} \right]$ $\times$ 3 & avgpool \\

conv3-512 $\times$ 3 & $\left( \begin{matrix} \text{conv1-$6k$} \\ \text{conv3-$6k$} \\ \text{conv1-160} \end{matrix} \right)$ $\times$ 3 \\

maxpool & $\left[\begin{matrix} \text{conv1-$6k$} \\ \text{conv3-$6k$} \\ \text{conv1-320} \end{matrix} \right]$ $\times$ 1 & \\

conv3-512 $\times$ 3 & fc-1280 \\

maxpool & & \\

fc-512 $\times$ 2 & & \\

\hline

\multicolumn{3}{c}{fc-10 or fc-100} \\
\bottomrule
\end{tabular} 
}
\end{sc}
\end{center}
\end{table*}

\section{Additional Algorithms}
\label{app:algo}

\begin{algorithm}[H]
 \label{algo:psgd}
 \caption{Perturbed SGD}
  \KwIn{Training data $S$, the noise scale hyperparameter $b$, the number of iterations $T$, learning rate $\eta$, initialized weights $\theta_{0}$}
  \KwOut{The model weights $\theta$}
  \Repeat{stopping criterion is not met}{
    \For{$(x,y) \in S$}
    {
      $ \epsilon_{t} \sim \mathcal{N}(0,b^{2}I)$\;
      $g_{t} = \frac{\partial L(\theta_{t-1} , (x,y))}{\partial \theta}$\;
      $\theta_{t} = \theta_{t-1} - \eta (g_{t} + \epsilon_{t})   $\;
    }
  }
\end{algorithm}

\section{Supplementary Experimental Results}
\label{app:results}

\paragraph{Robustness to weight perturbation.} Figure \ref{fig:sppwn}. The empirical results demonstrate that NVRM can also make VGG-16 and MobileNetV2 more robust to weight perturbations. We also report that, obviously, the architecture of ResNet is much more optimal to achieve strong neural variability than VGG and MobileNet. We leave the network architecture study as a future work.

\paragraph{Improved generalization.} Figure \ref{fig:vgg_mobilenet}. 

\paragraph{Robustness to noisy labels.} Figure \ref{fig:resnet_sym}. SGD almost memorizes all corrupted labels. The results demonstrate that NVRM can also significantly improve the robustness to symmetric label noise. 

\paragraph{Robustness to catastrophic forgetting.} Figures \ref{fig:permuted} and \ref{fig:split}. NVRM also enhances robustness to catastrophic forgetting in the setting of both incremental class learning. The variability hyperparameter $b$ of NVRM is defaulted to be $0.03$. The NVRM curves consistently outperform the counterpart curves.

\paragraph{SGD with large gradient noise cannot relieve noise memorization.} It is well-known that increasing the ratio of the learning rate and the batch size $\frac{\eta}{B}$ may enhance the scale of gradient noise in SGD and help find flatter minima \citep{jastrzkebski2017three,he2019control}. \citet{xie2021positive} reported that amplifying stochastic gradient noise may mitigate label noise. However, our theoretical analysis suggests that, as stochastic gradient noise carries the information about training data, there is no theoretical guarantee that large stochastic gradient noise can work as well as NVRM. We empirically studied SGD with large stochastic gradient noise in Figure \ref{fig:labelnoiselr}. It demonstrate that, SGD with various learning rate finally still memorizes noisy labels, while NVRM-SGD with various learning rates can consistently relieve overfitting noisy labels. Noe that, in Figure \ref{fig:labelnoiselr}, we initialized the learning rates as $\{0.1, 0.3, 1, 3\}$, respectively, and divided the learning rate by $10$ after every $60$ epochs.

\begin{figure}[t!]
  \begin{subfigure}{0.99\linewidth}
    \centering
    \includegraphics[width=0.24\linewidth]{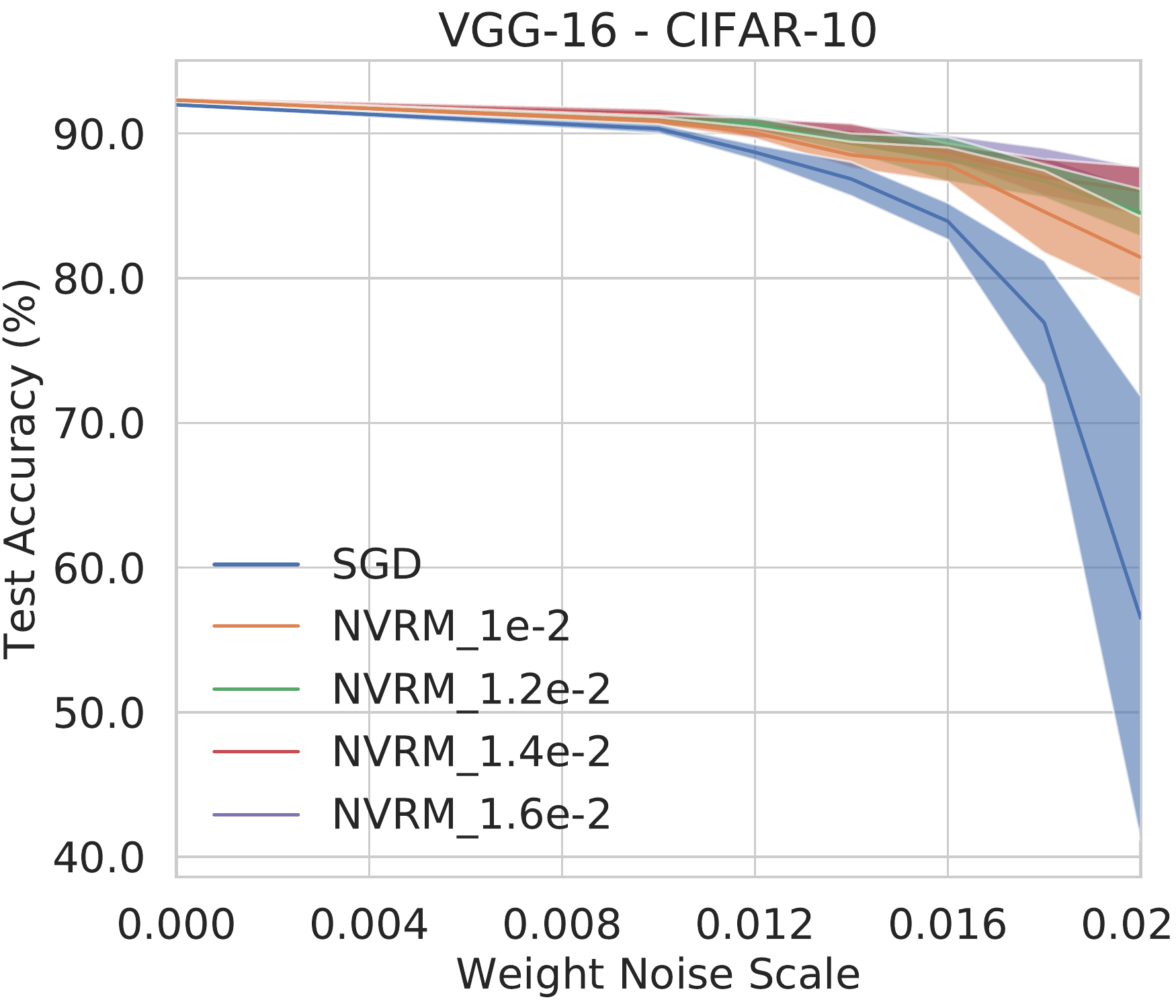}
    \includegraphics[width=0.24\linewidth]{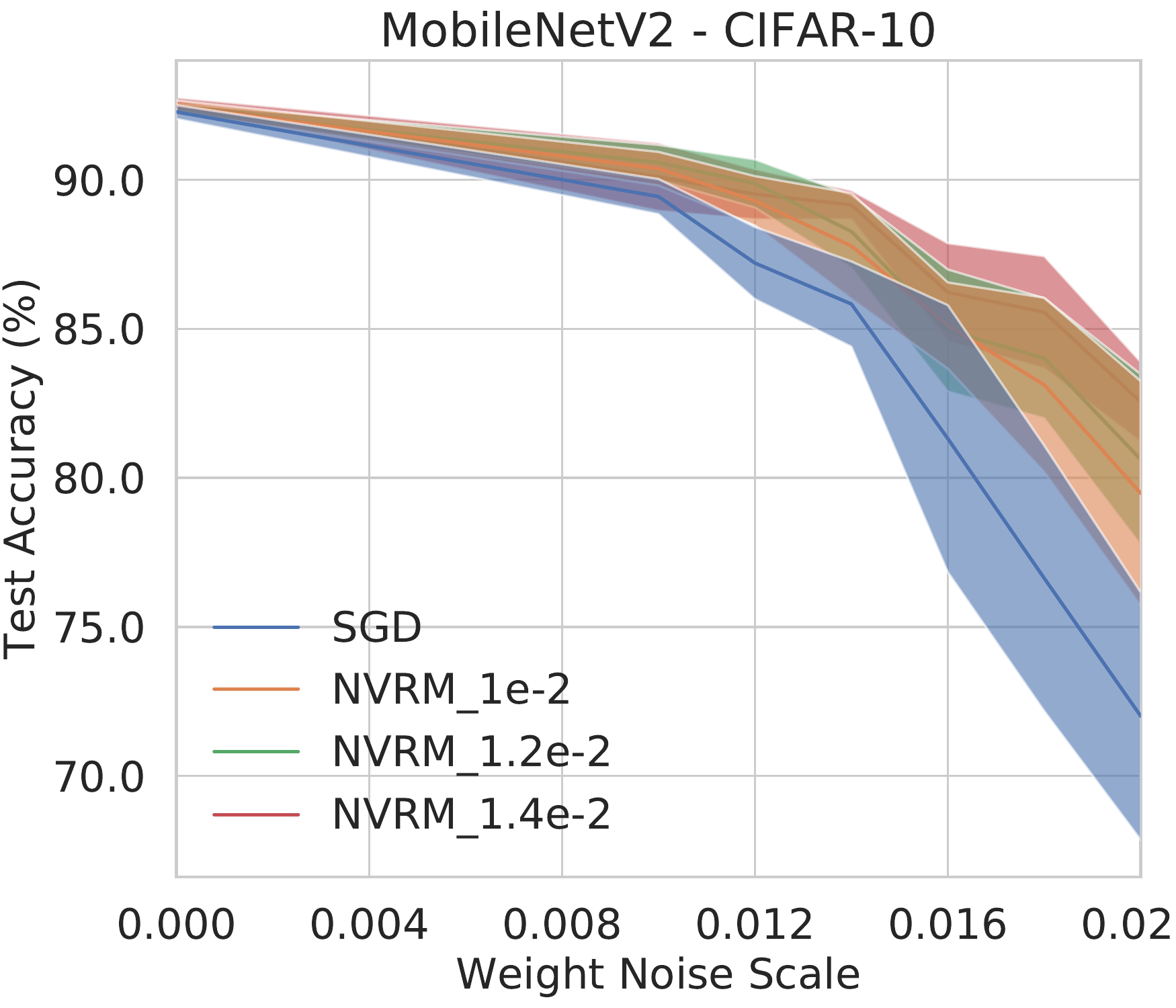}
    \caption{Test accuracy to weight noise scale on CIFAR-10.}
    \label{fig:wn_C10}
  \end{subfigure} 
  
  \begin{subfigure}{0.99\linewidth}
   \centering
    \includegraphics[width=0.24\linewidth]{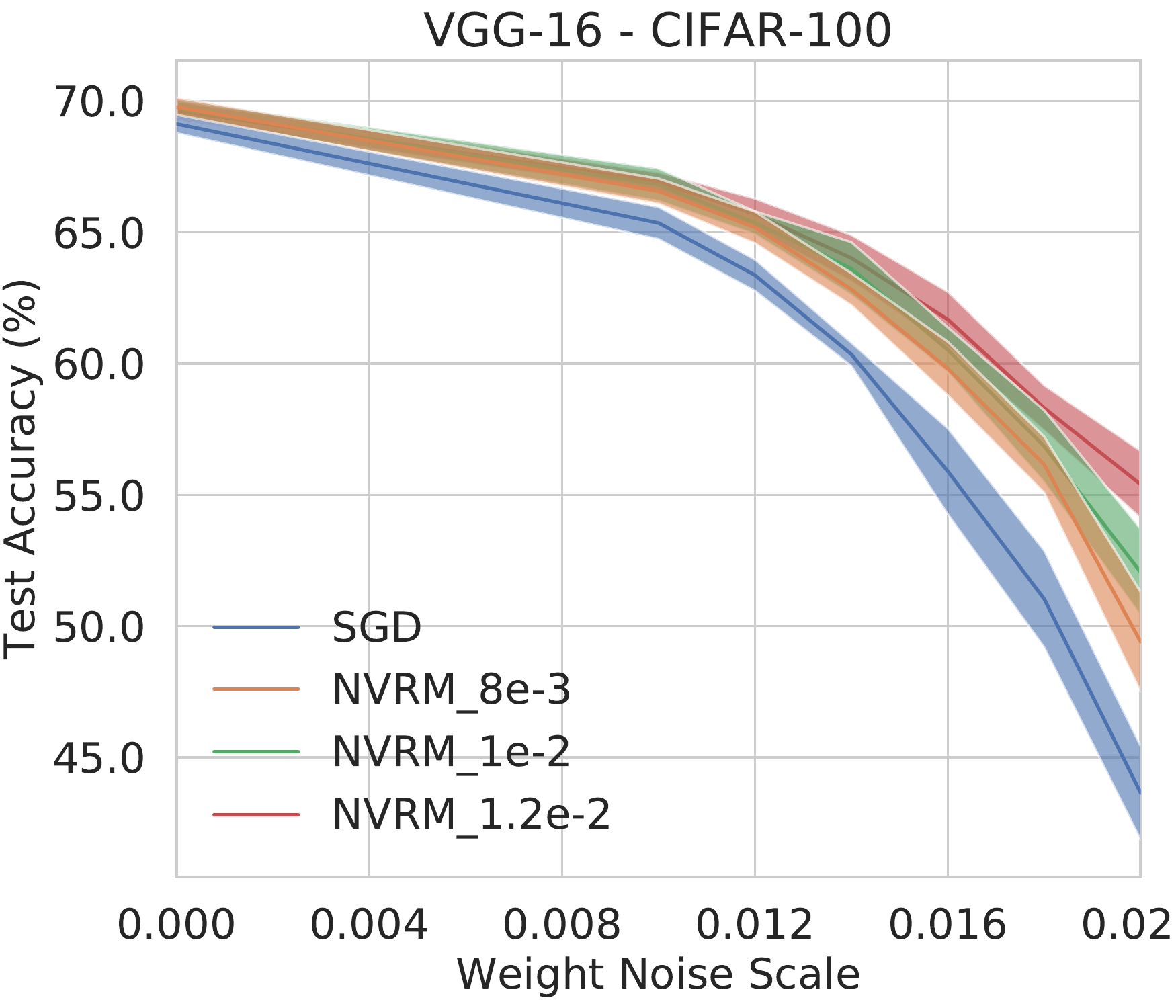}
    \includegraphics[width=0.24\linewidth]{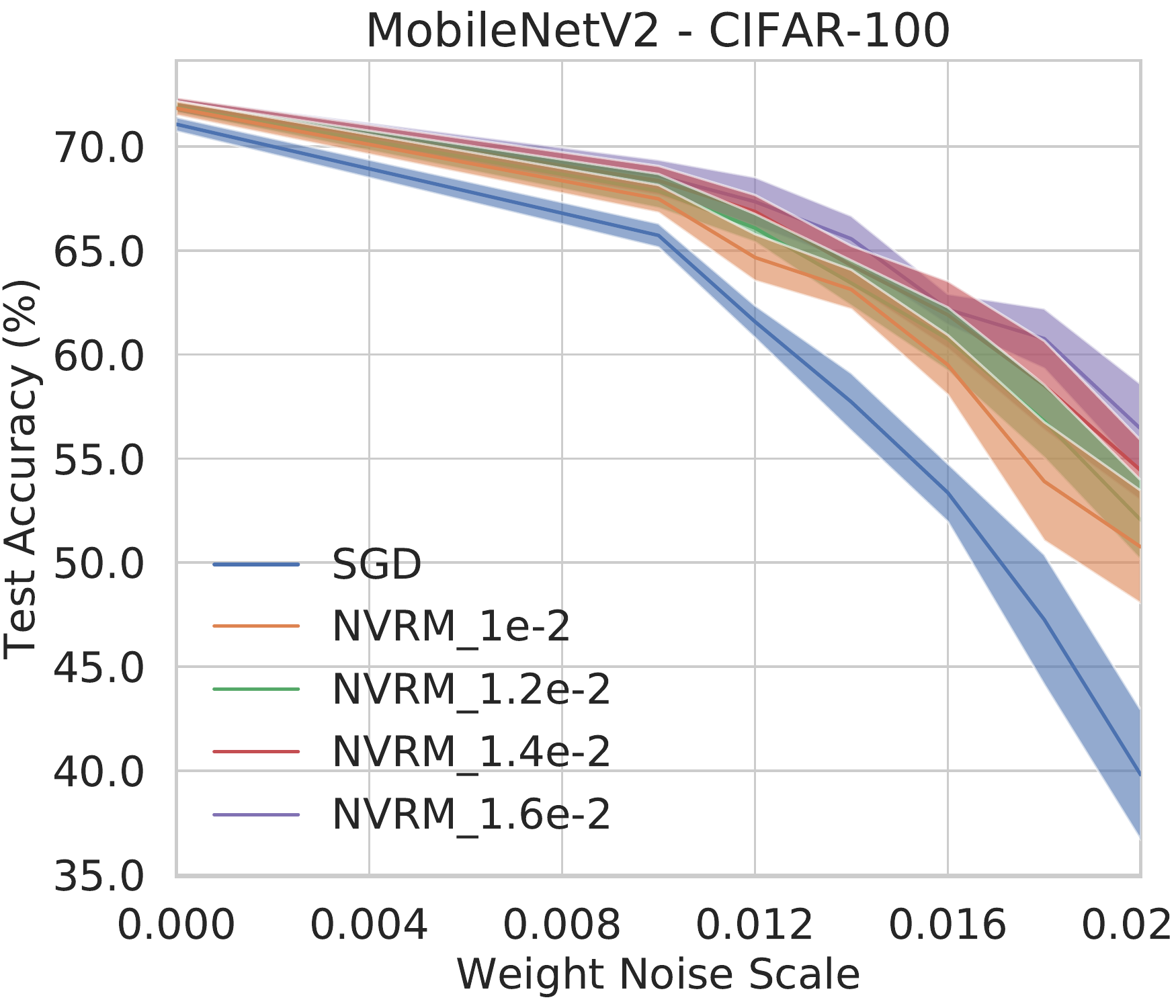}
    \caption{Test accuracy to weight noise scale on CIFAR-100.}
    \label{fig:wn_C100}
  \end{subfigure}
  
  \caption{Curves of test accuracy to weight noise scale. Top Row: CIFAR-10; Bottom Row: CIFAR-100. The three columns are VGG-16 and MobileNetV2 respectively.}
  \label{fig:sppwn}
\end{figure}

\begin{figure}[h]
  \begin{subfigure}{0.97\linewidth}
    \includegraphics[width=0.2425\linewidth]{Pictures/vgg16_bn_cifar10_0_st2_gap.pdf}
    \includegraphics[width=0.2425\linewidth]{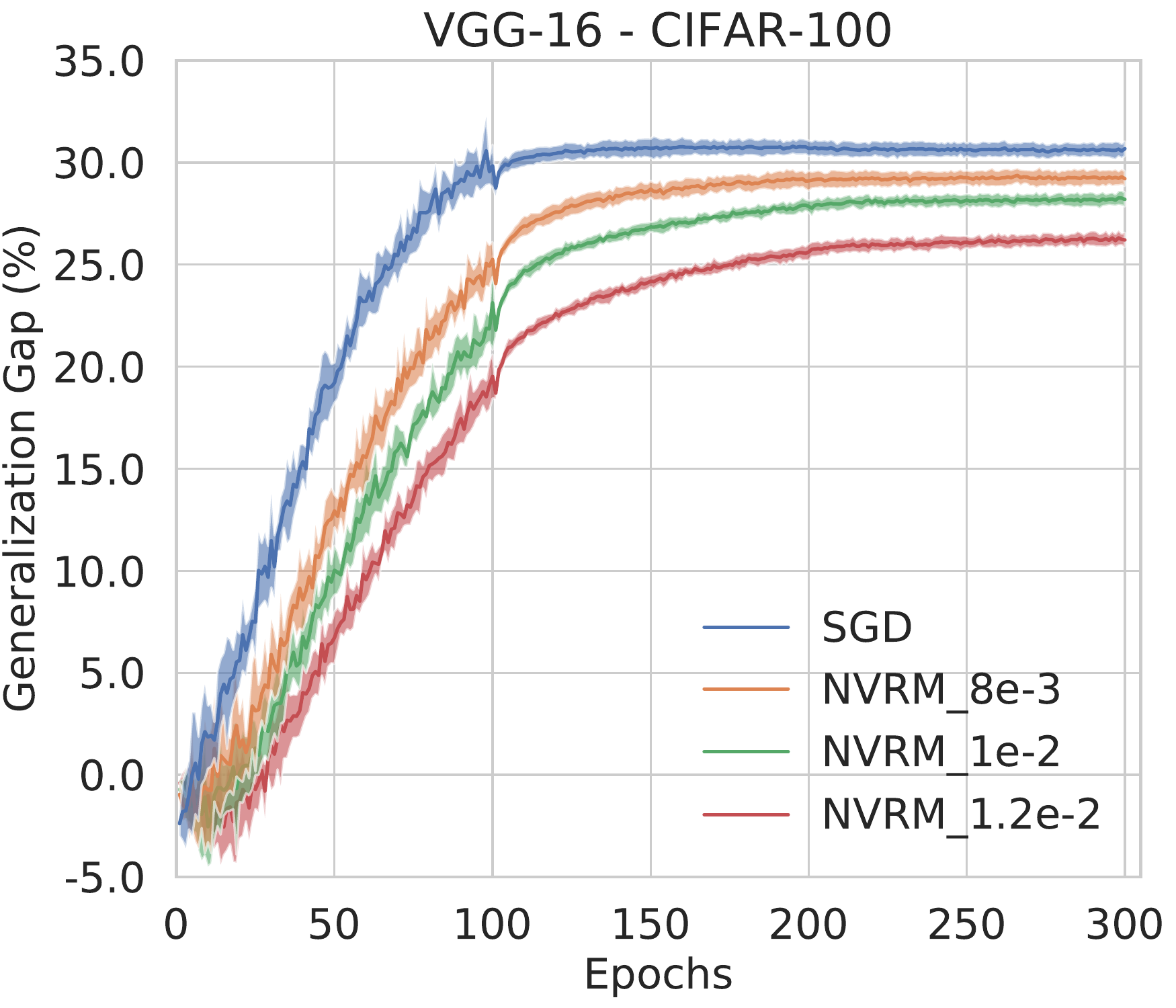}
    \includegraphics[width=0.2425\linewidth]{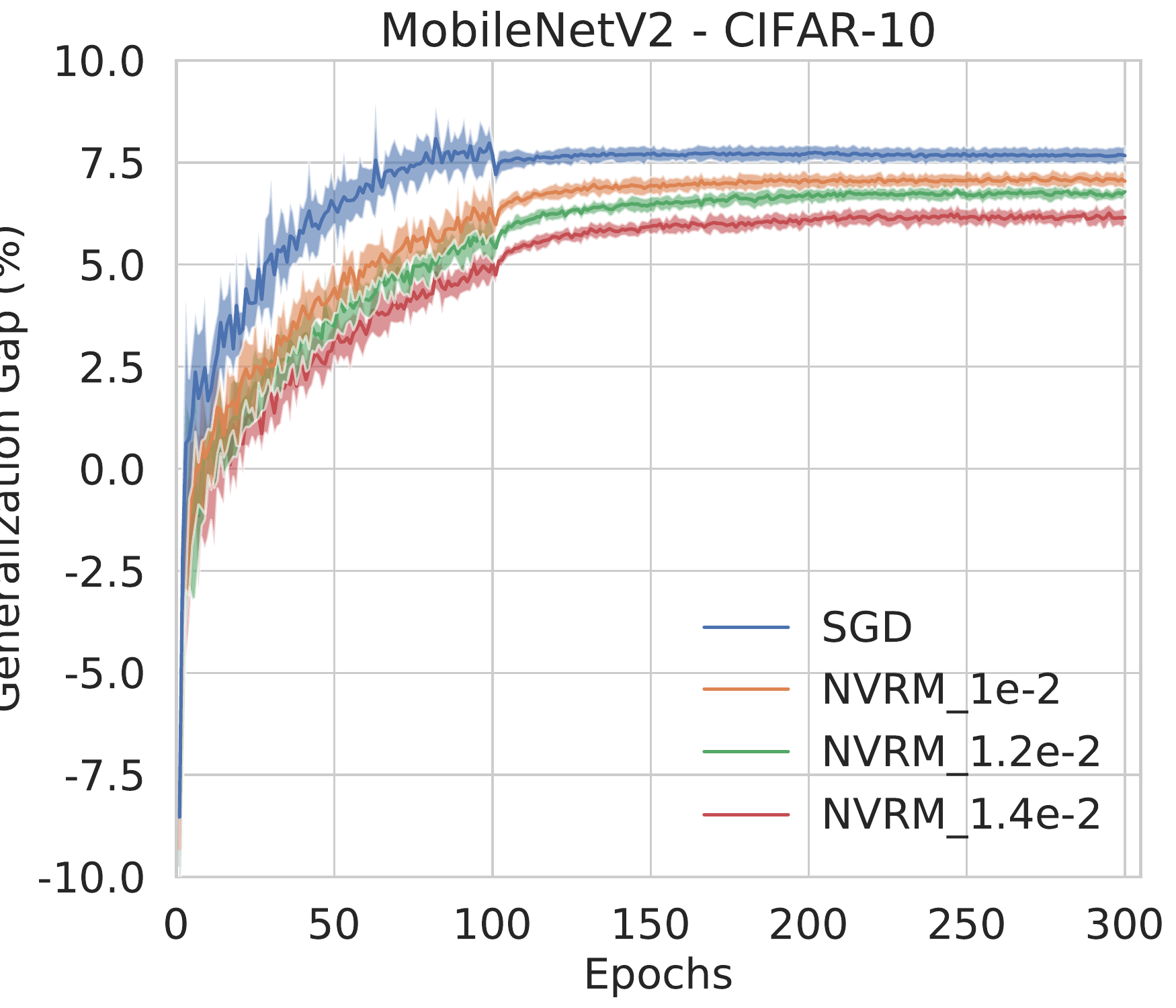}
    \includegraphics[width=0.2425\linewidth]{Pictures/mobilenetv2_cifar100_0_st2_gap.pdf}
    \caption{Generalization gap to epochs.}
    \label{fig:gap}
  \end{subfigure} \\[0.5em]

  \begin{subfigure}{0.97\linewidth}
    \includegraphics[width=0.2425\linewidth]{Pictures/vgg16_bn_cifar10_0_st2_test.pdf}
    \includegraphics[width=0.2425\linewidth]{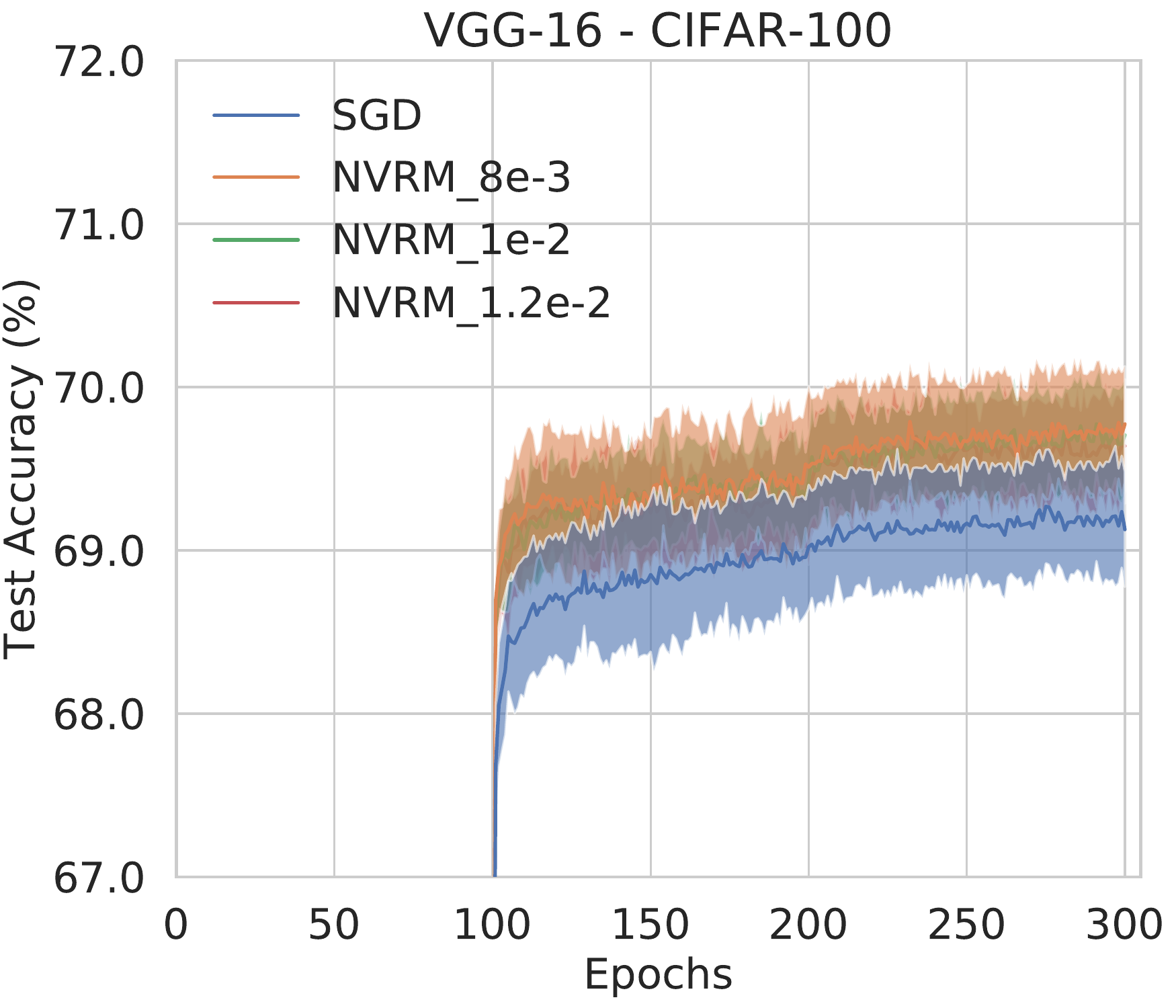}
    \includegraphics[width=0.2425\linewidth]{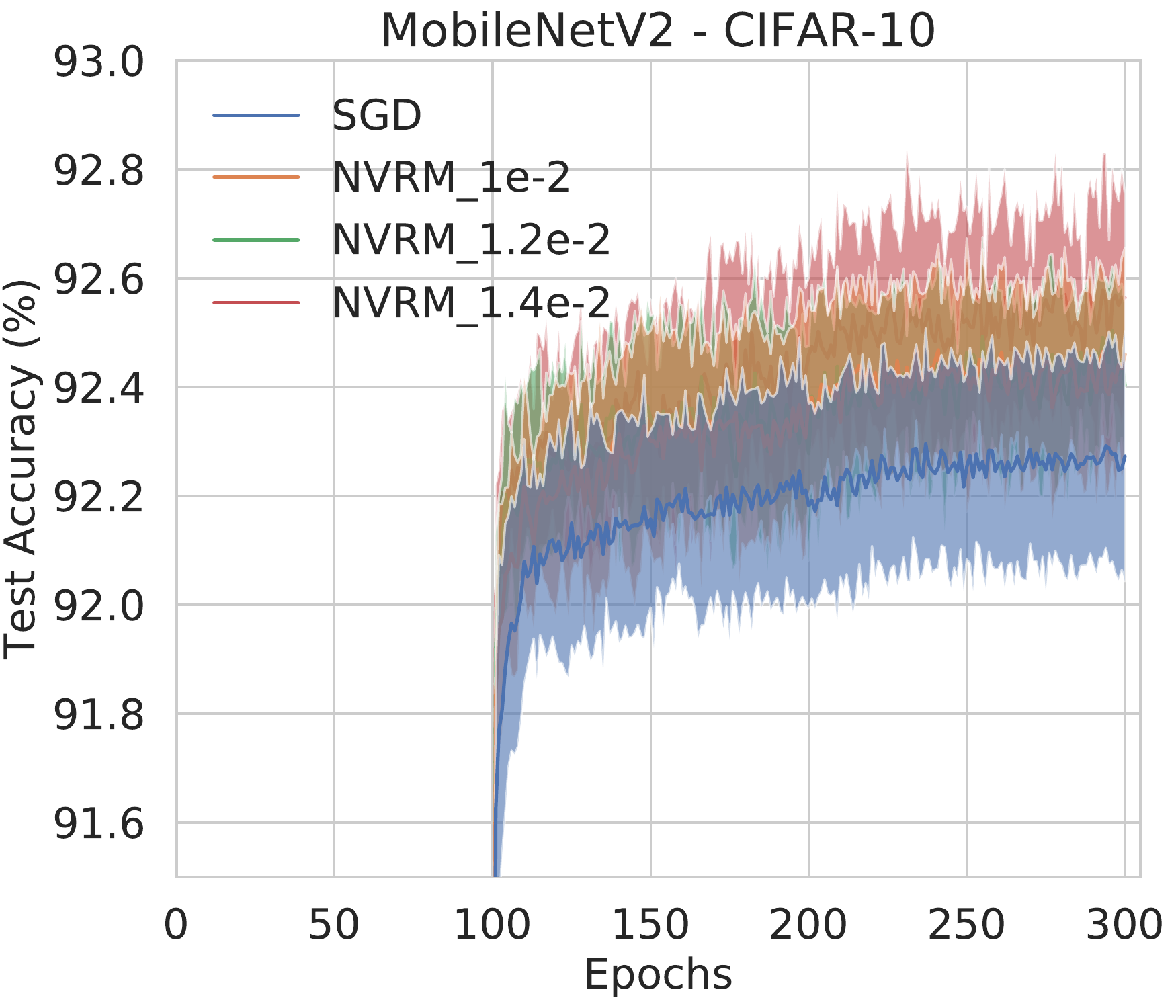}
    \includegraphics[width=0.2425\linewidth]{Pictures/mobilenetv2_cifar100_0_st2_test.pdf}
    \caption{Test accuracy to epochs.}
    \label{fig:test_acc}
  \end{subfigure}
  
  \caption{Curves of generalization gap and test accuracy to epochs. NVRM with various variability scales $b$ can consistently improve generalization. The four columns are, respectively, for (1) VGG-16 on CIFAR-10, (2) VGG-16 on CIFAR-100, (3) MobileNetV2 on CIFAR-10, (4) MobileNetV2 on CIFAR-100.}
  \label{fig:vgg_mobilenet}
\end{figure}

\begin{figure}[t]
  \begin{subfigure}{0.97\linewidth}
    \includegraphics[width=0.2425\linewidth]{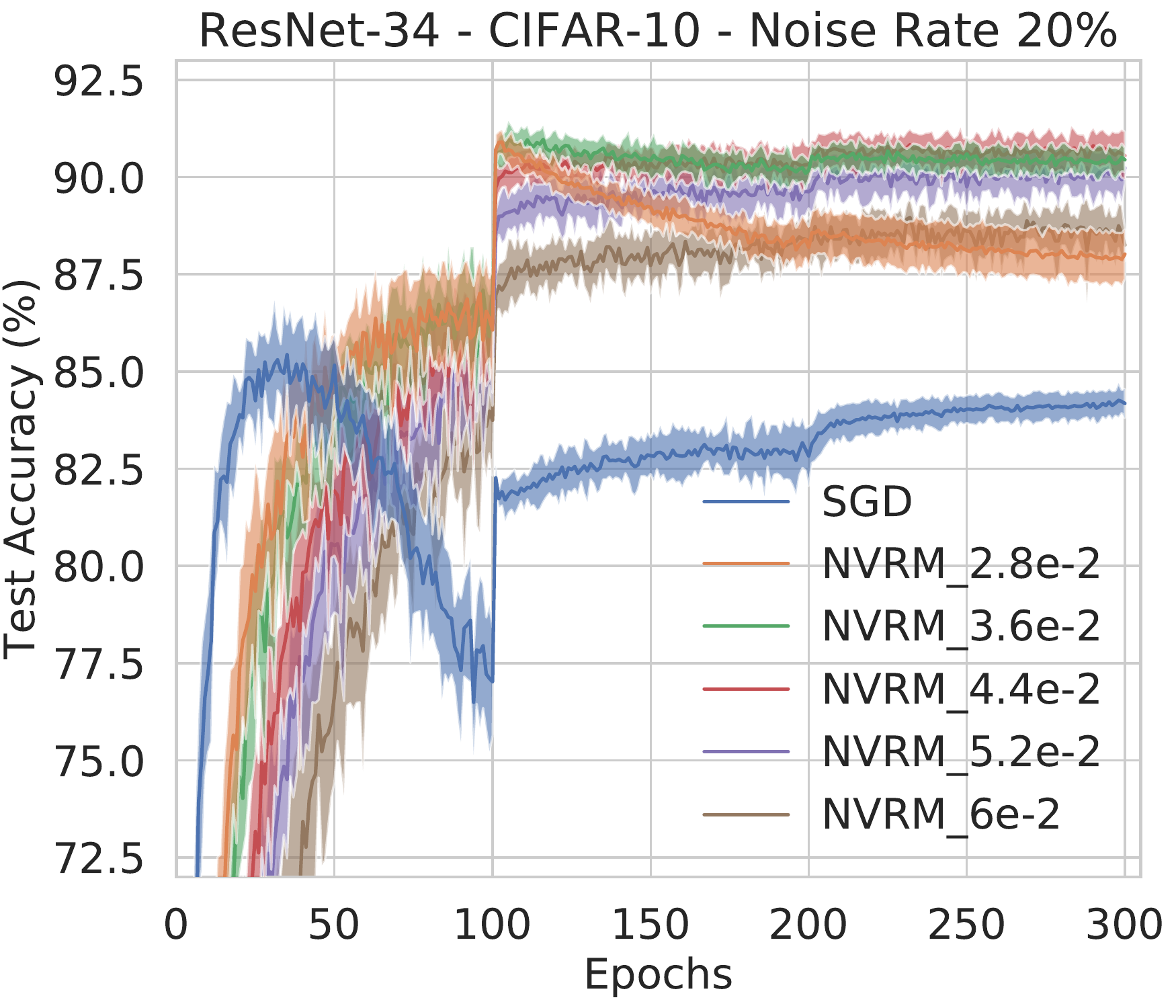}
    \includegraphics[width=0.2425\linewidth]{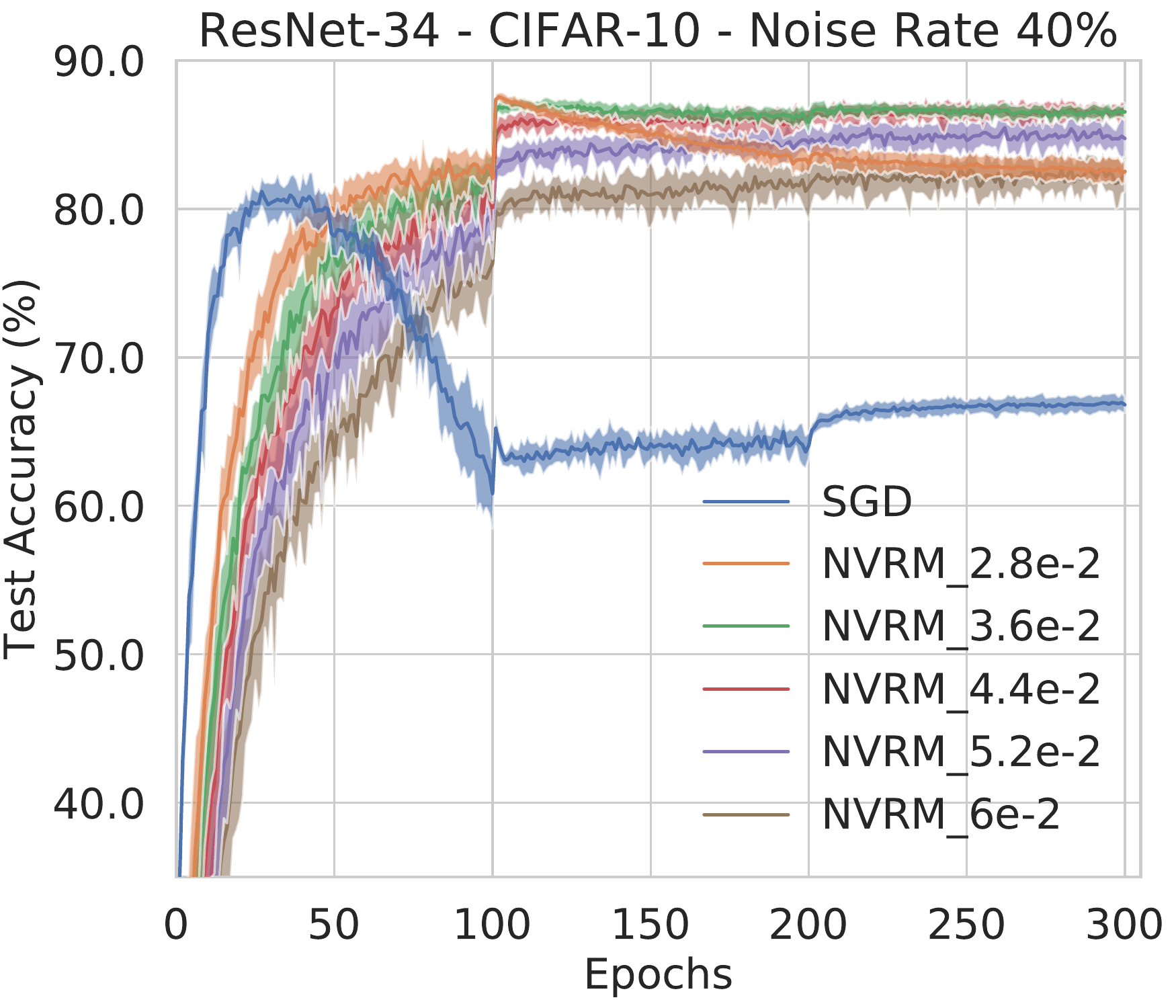}
    \includegraphics[width=0.2425\linewidth]{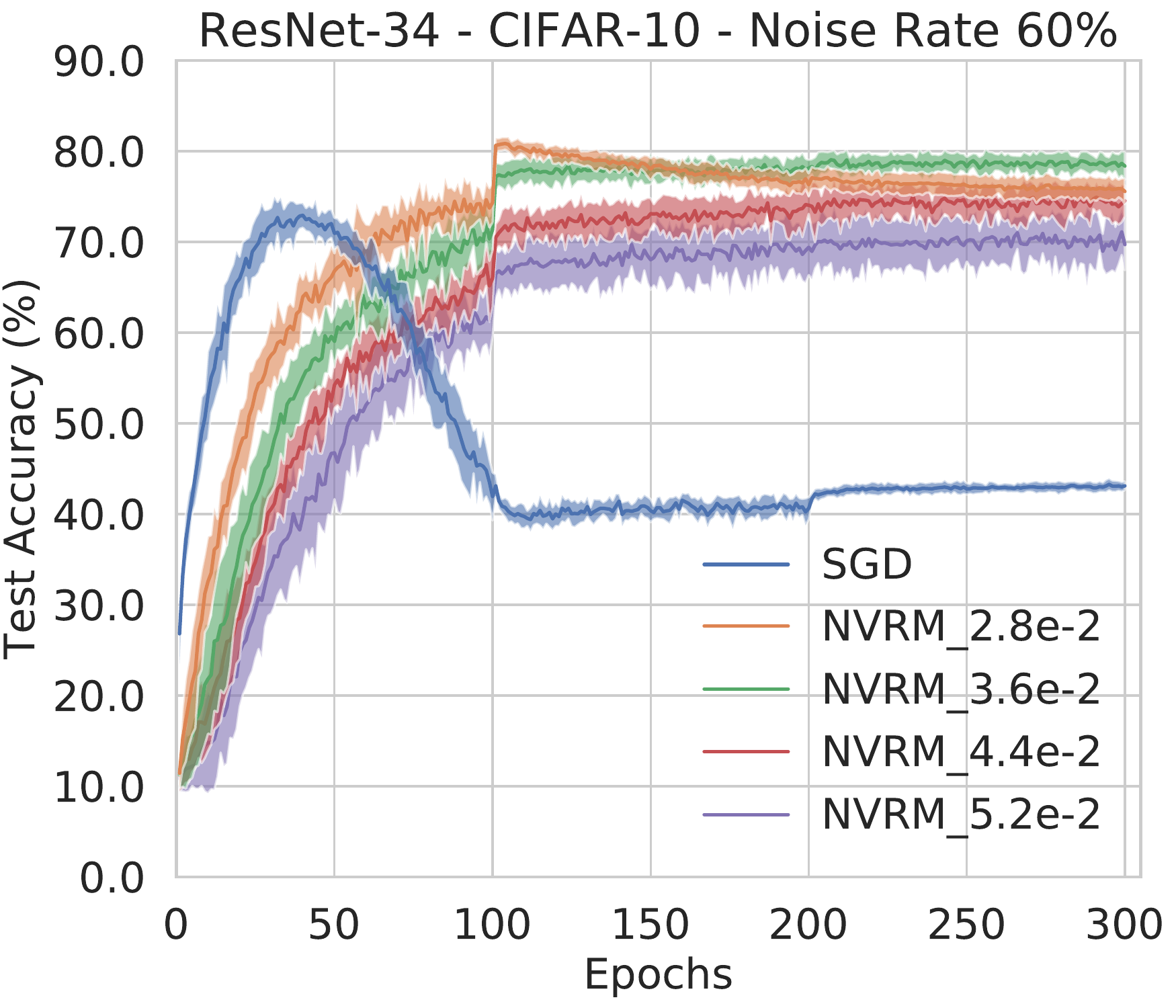}
    \includegraphics[width=0.2425\linewidth]{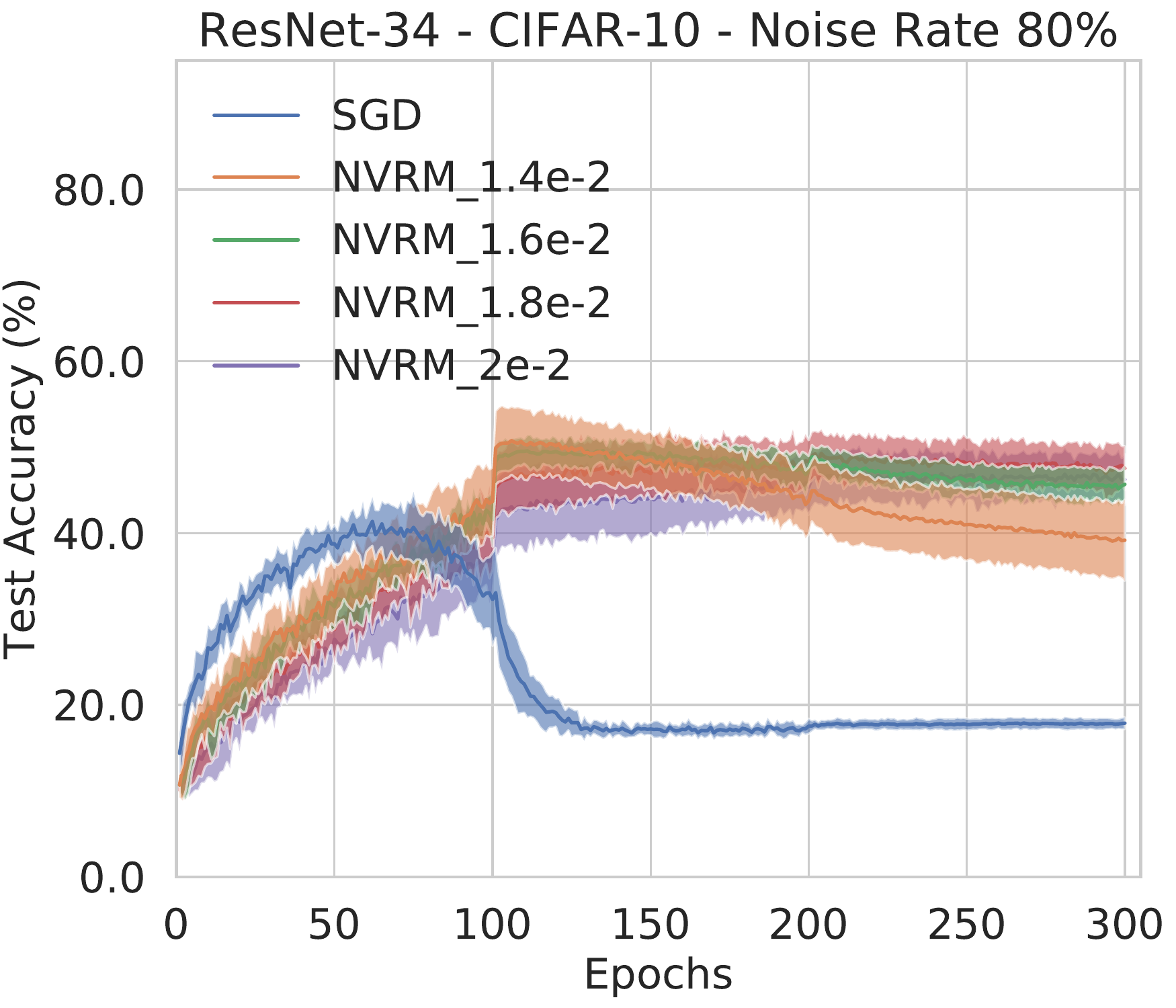}
    \caption{Test accuracy to epochs of ResNet-34 on CIFAR-10.}
    \label{fig:C10_noise}
  \end{subfigure} \\[0.5em]

  \begin{subfigure}{0.97\linewidth}
    \includegraphics[width=0.2425\linewidth]{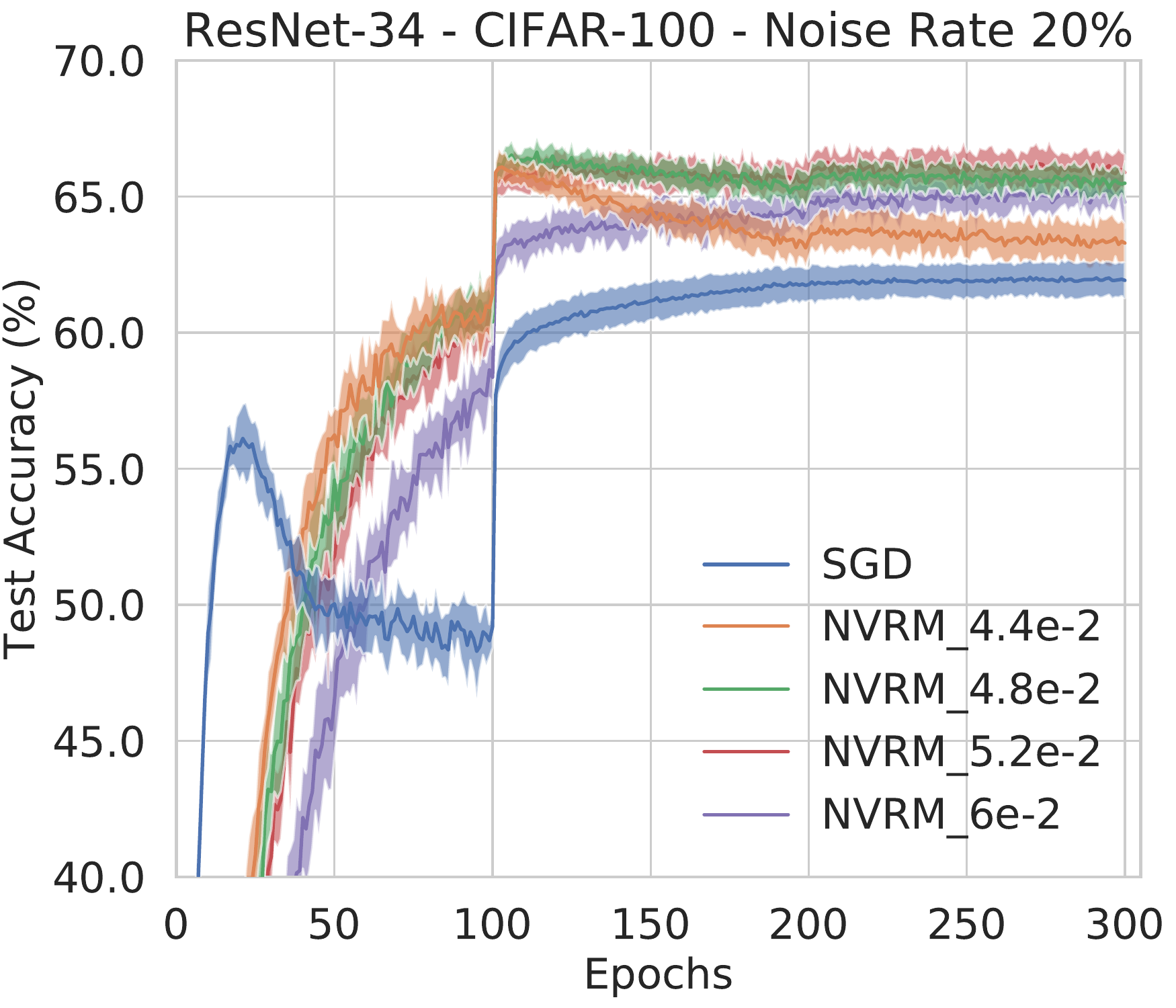}
    \includegraphics[width=0.2425\linewidth]{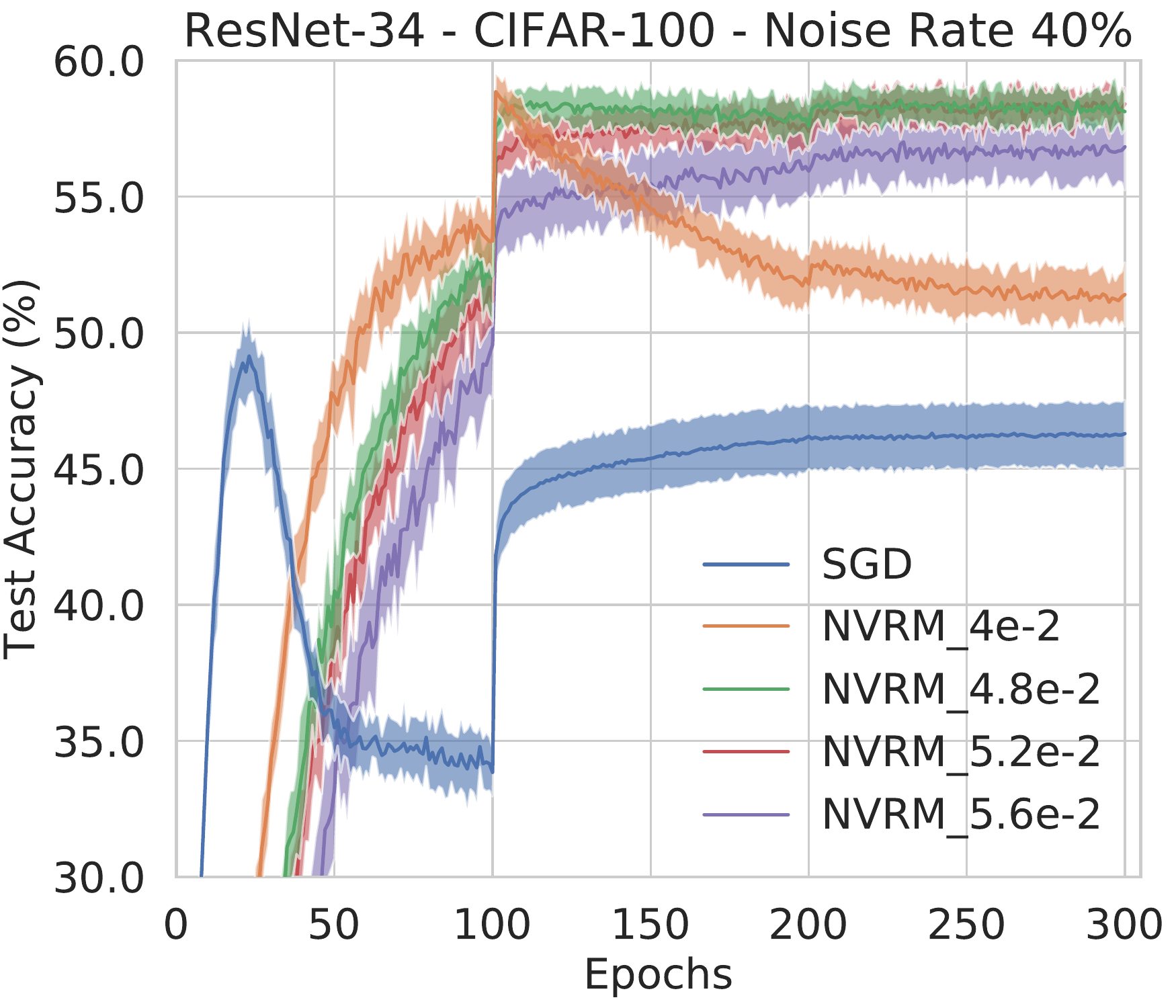}
    \includegraphics[width=0.2425\linewidth]{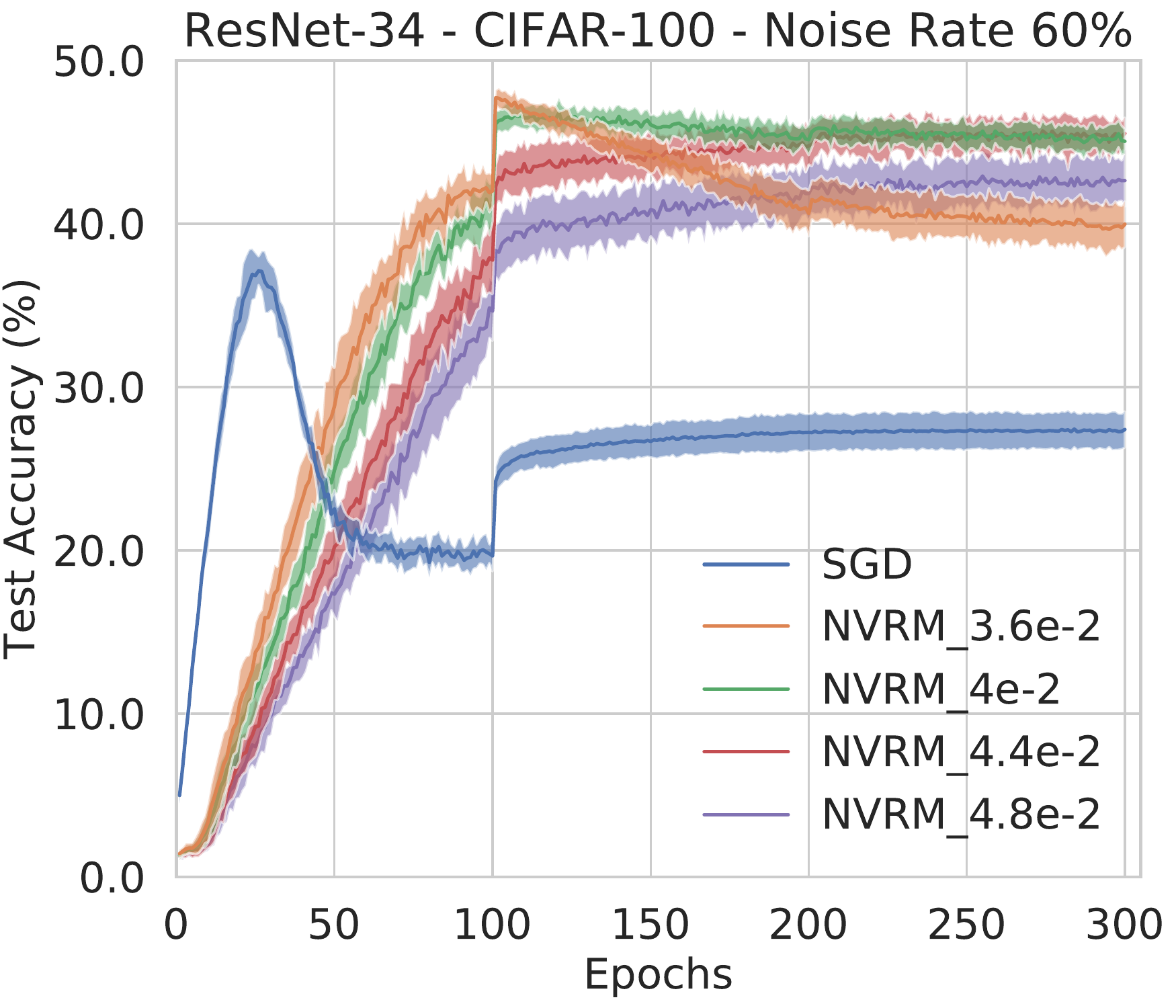}
    \includegraphics[width=0.2425\linewidth]{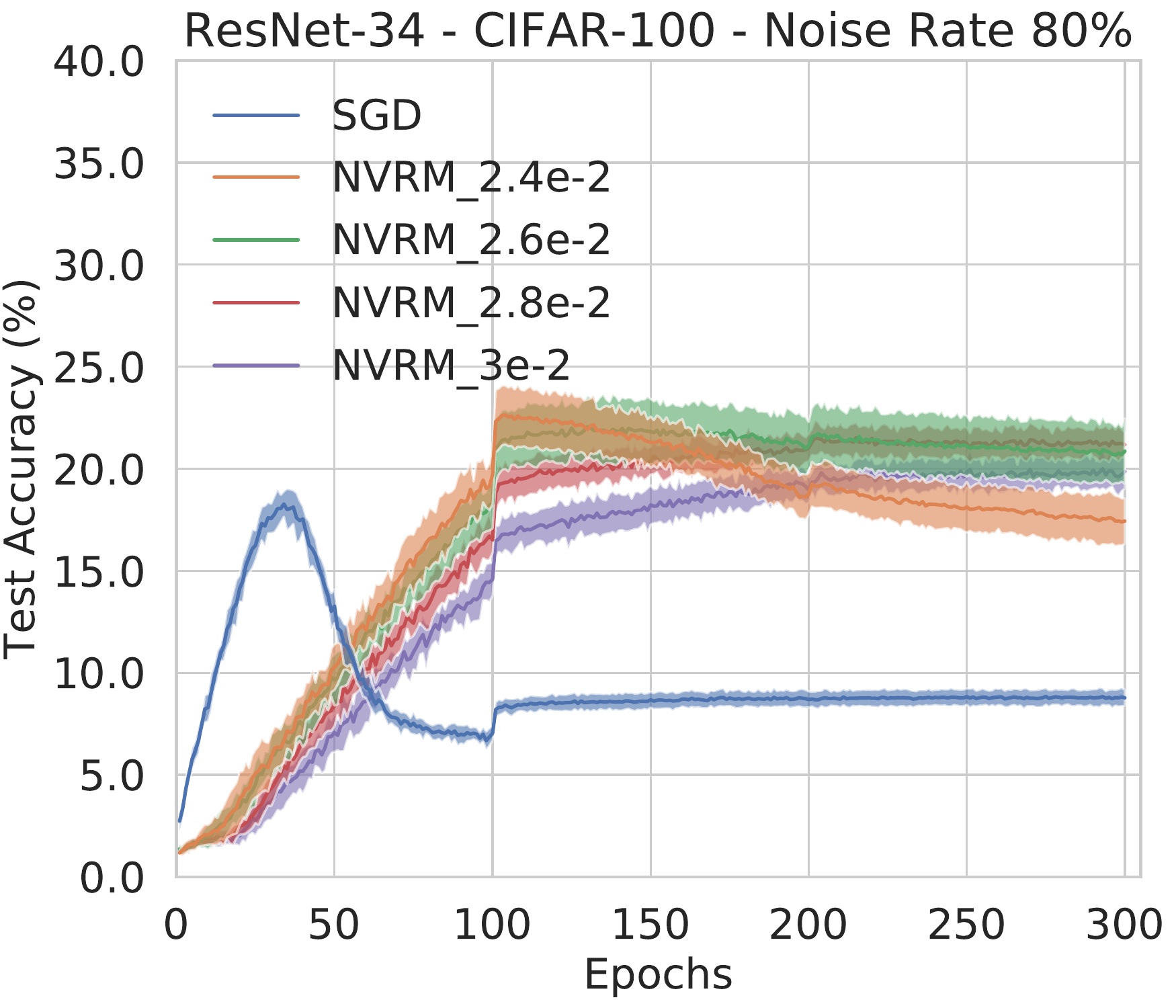}
    \caption{Test accuracy to epochs of ResNet-34 on CIFAR-100.}
    \label{fig:C100_noise}
  \end{subfigure}
  
  \caption{Curves of test accuracy to epochs of ResNet-34. The first row is on CIFAR-10, and the second is on CIFAR-100. The four columns are, respectively, for label noise rate $20\%$, $40\%$, $60\%$ and $80\%$. NVRM with various variability scales $b$ can consistently relieve memorizing noisy labels.}
  \label{fig:resnet_sym}
\end{figure}

\begin{figure}[t!]
  \begin{subfigure}{0.99\linewidth}
  \centering
    \includegraphics[width=0.3\linewidth]{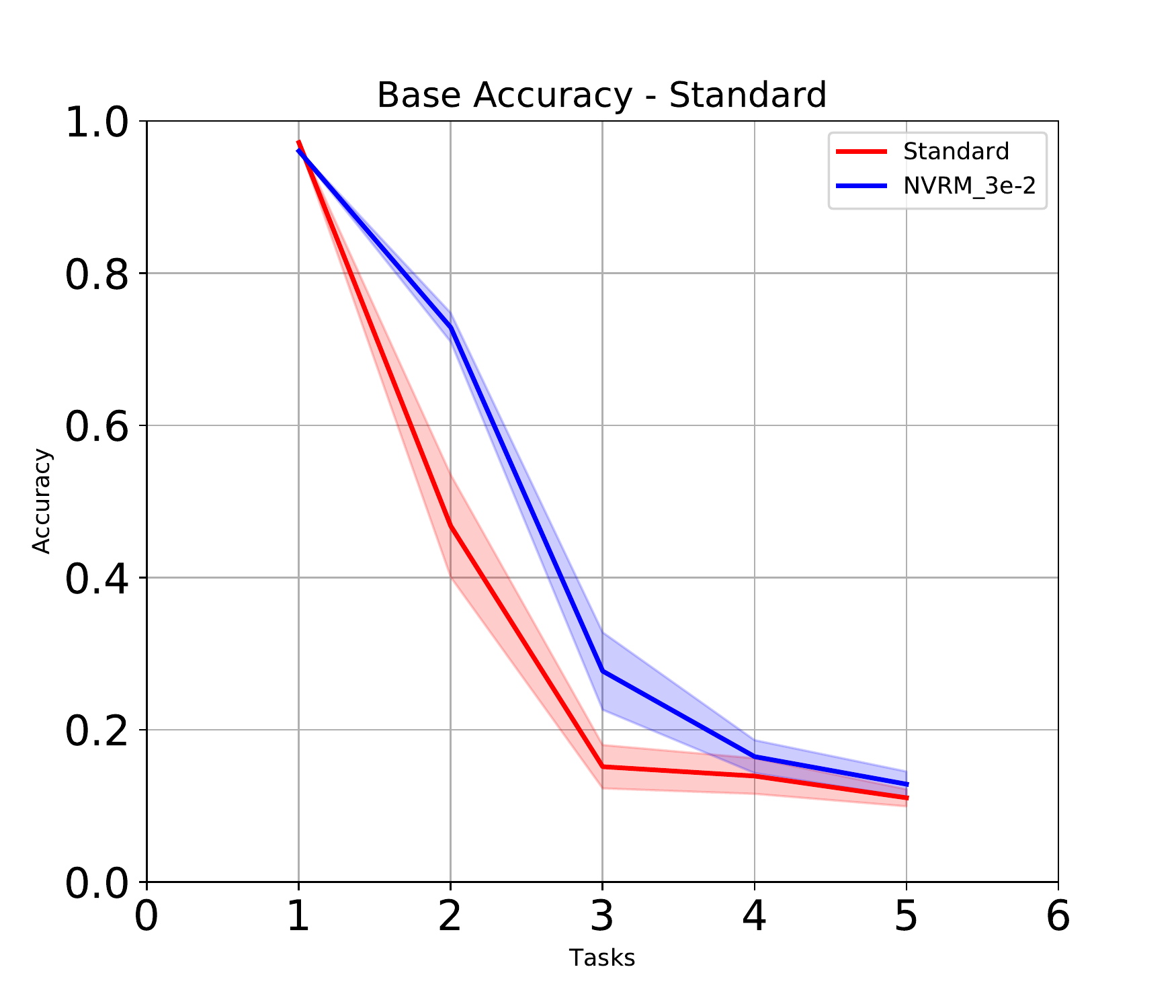}
    \includegraphics[width=0.3\linewidth]{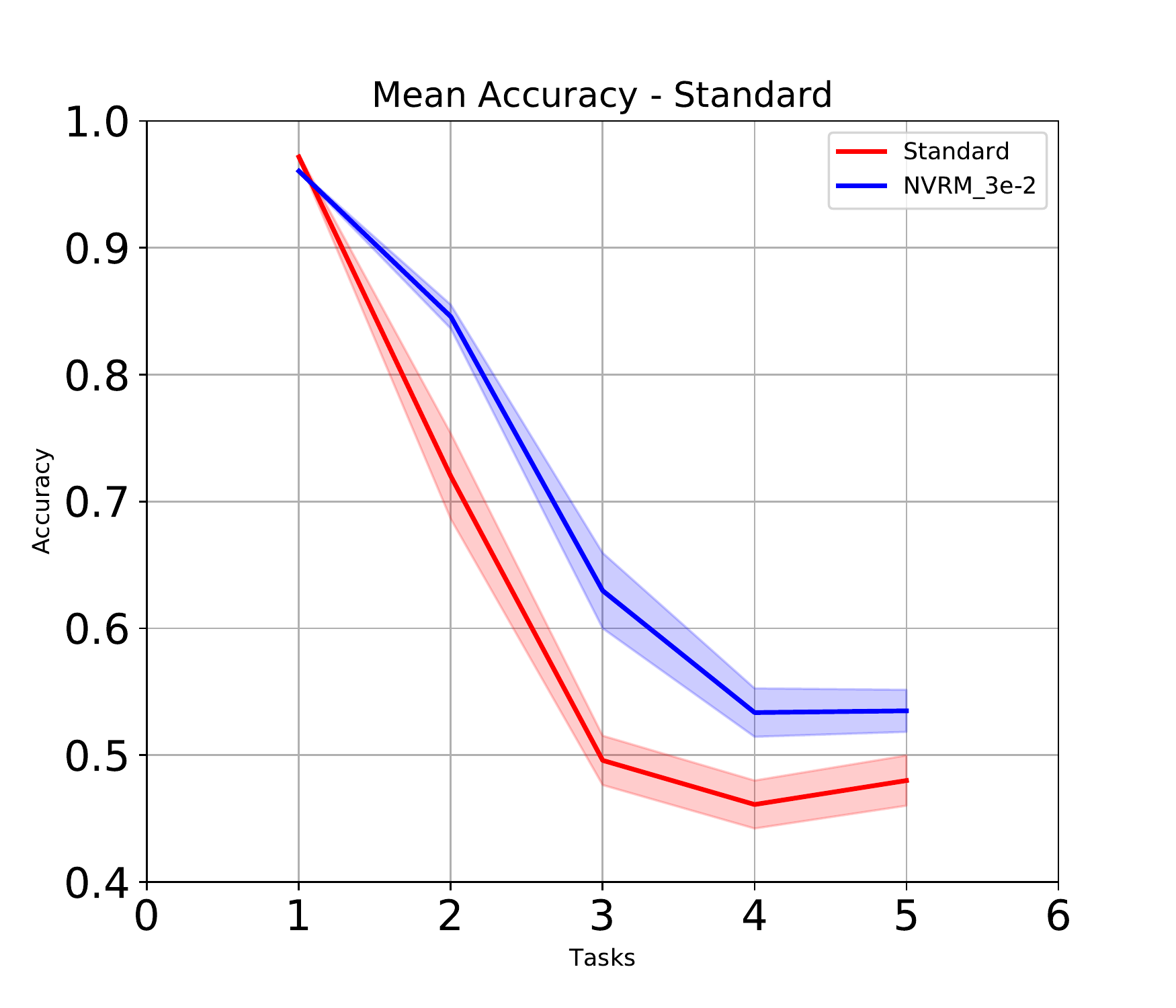}
  \end{subfigure}
  \caption{Curves of test accuracy to the number of tasks in continually learning Permuted MNIST. The two subfigures are respective (1): the accuracy of the base task; (2): the mean accuracy of all learned tasks.}
  \label{fig:permuted}
\end{figure}

\begin{figure}[t!]
  \begin{subfigure}{0.99\linewidth}
  \centering
    \includegraphics[width=0.3\linewidth]{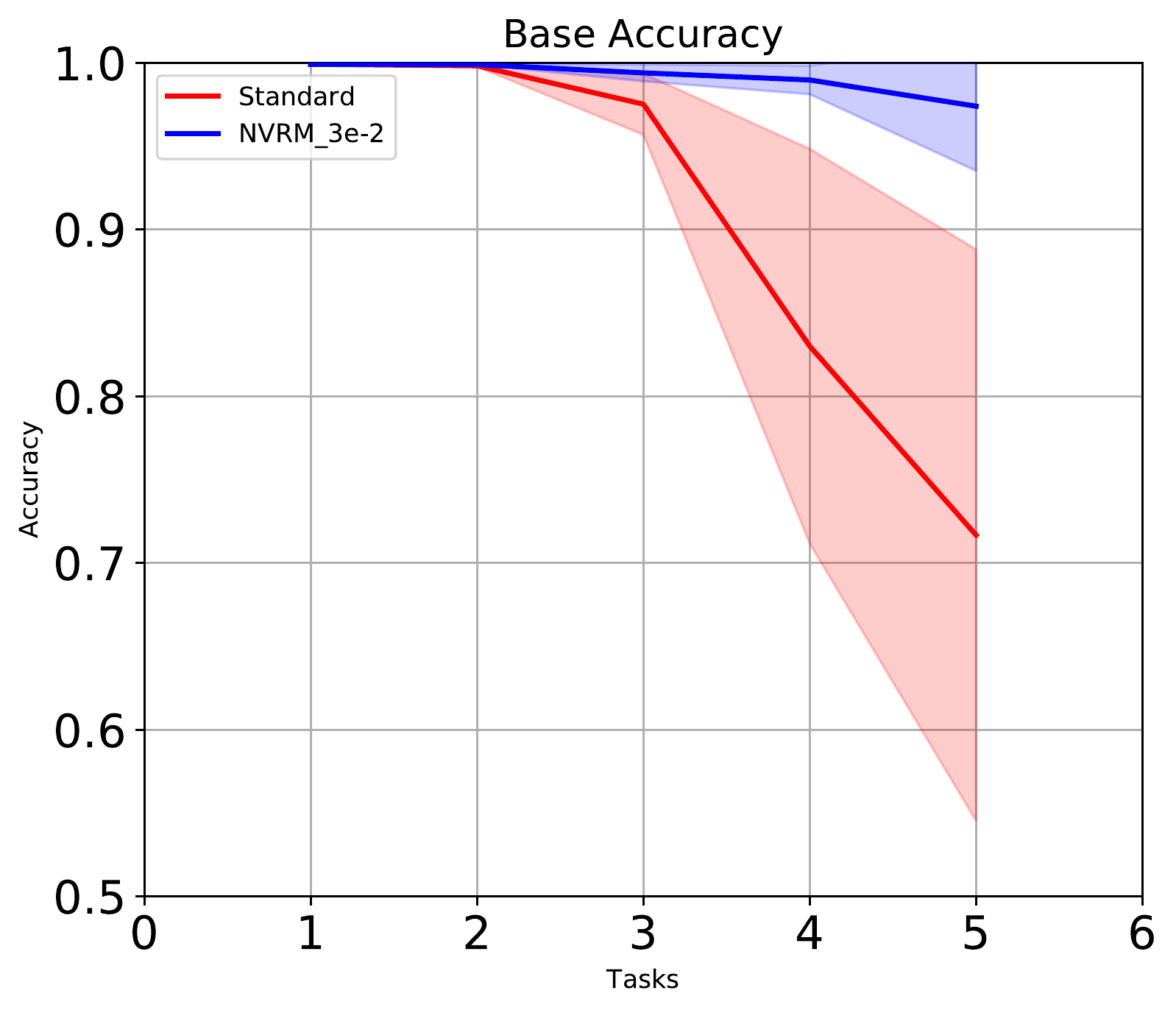}
    \includegraphics[width=0.3\linewidth]{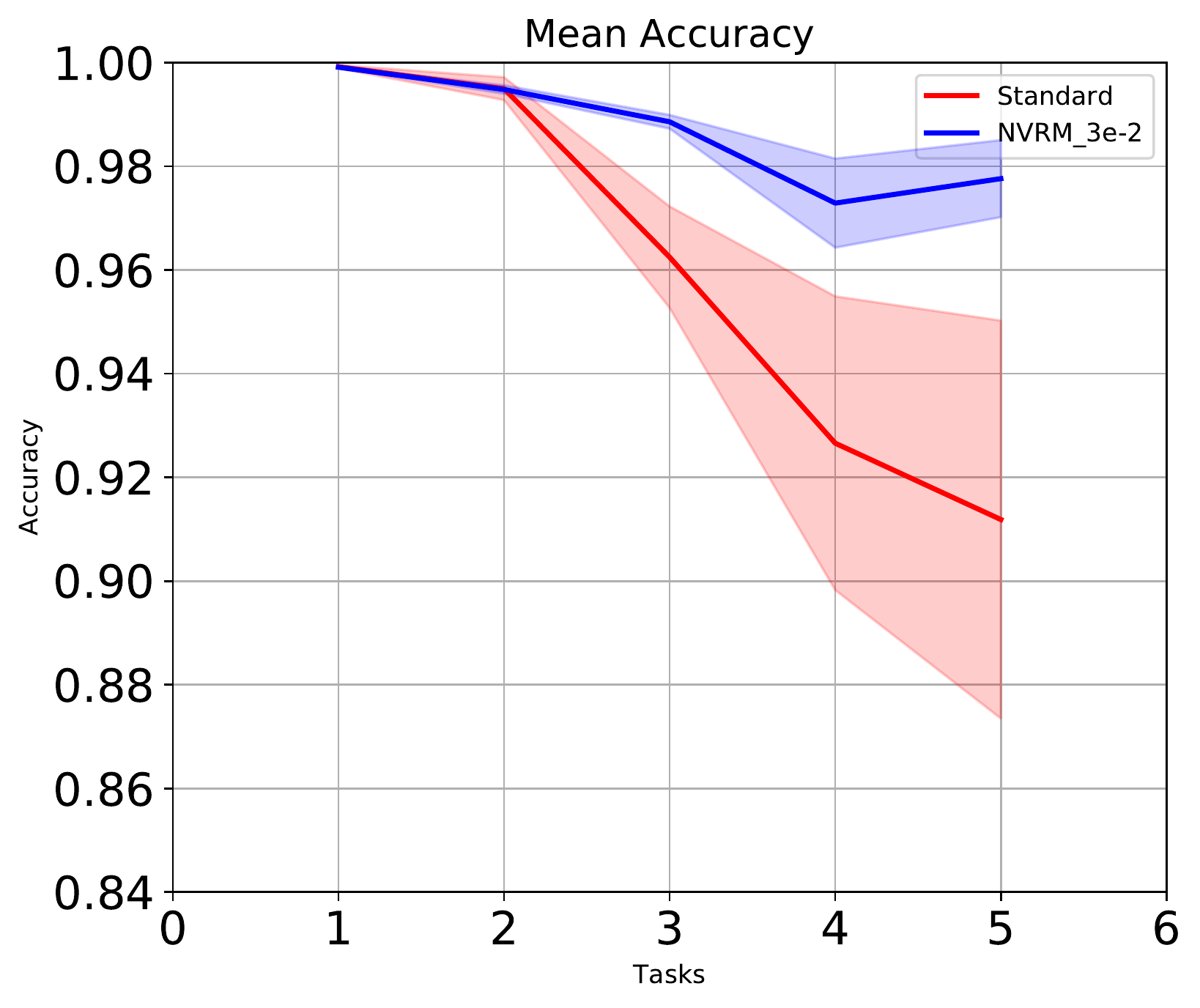}
  \end{subfigure}
  \caption{Curves of test accuracy to the number of tasks in continually learning split MNIST. The two subfigures are respective (1): the accuracy of the base task; (2): the mean accuracy of all learned tasks.}
  \label{fig:split}
\end{figure}

\begin{figure}[t!]
    \begin{subfigure}{0.99\linewidth}
    \centering
    \includegraphics[width=0.32\linewidth]{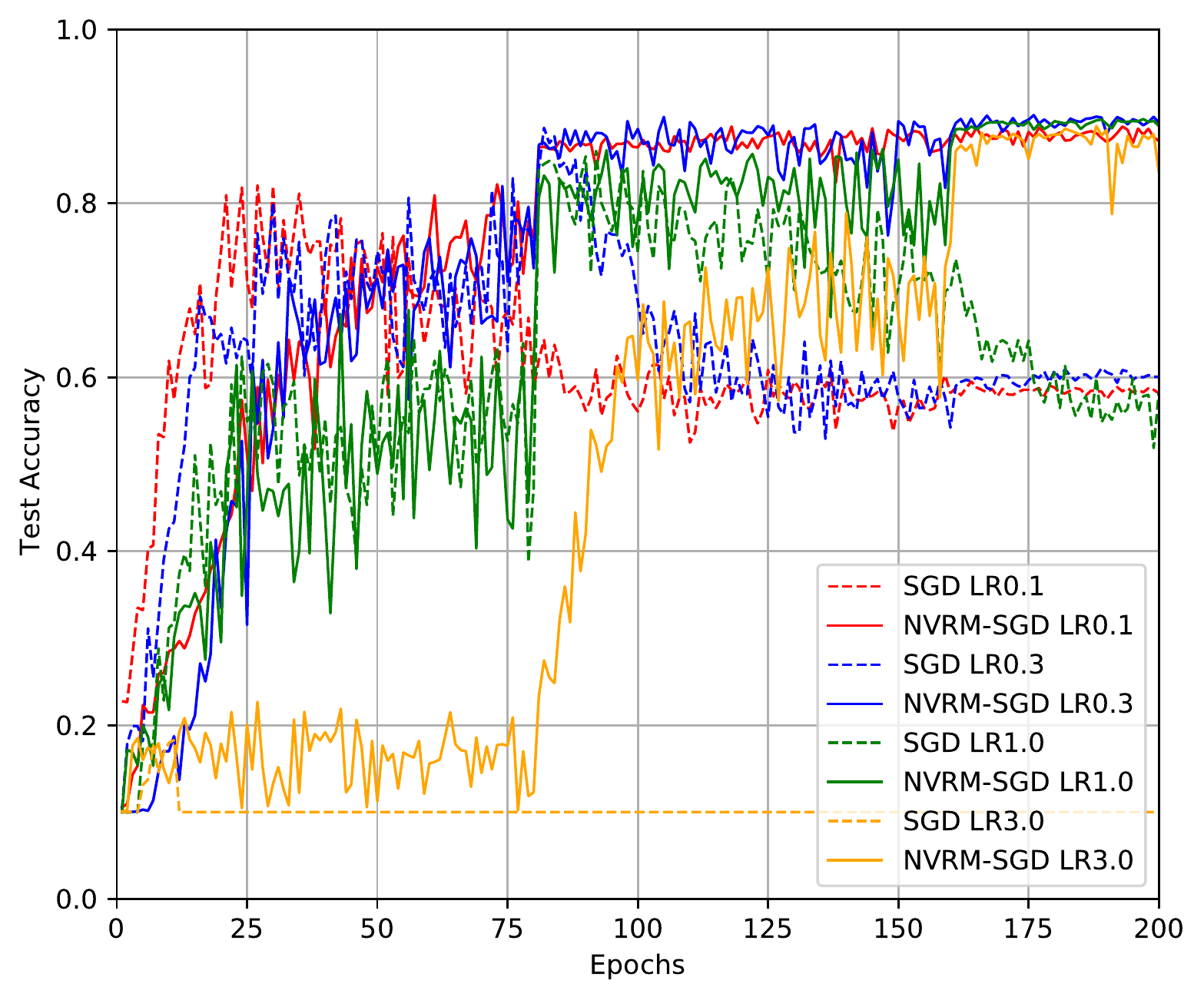}
    \includegraphics[width=0.32\linewidth]{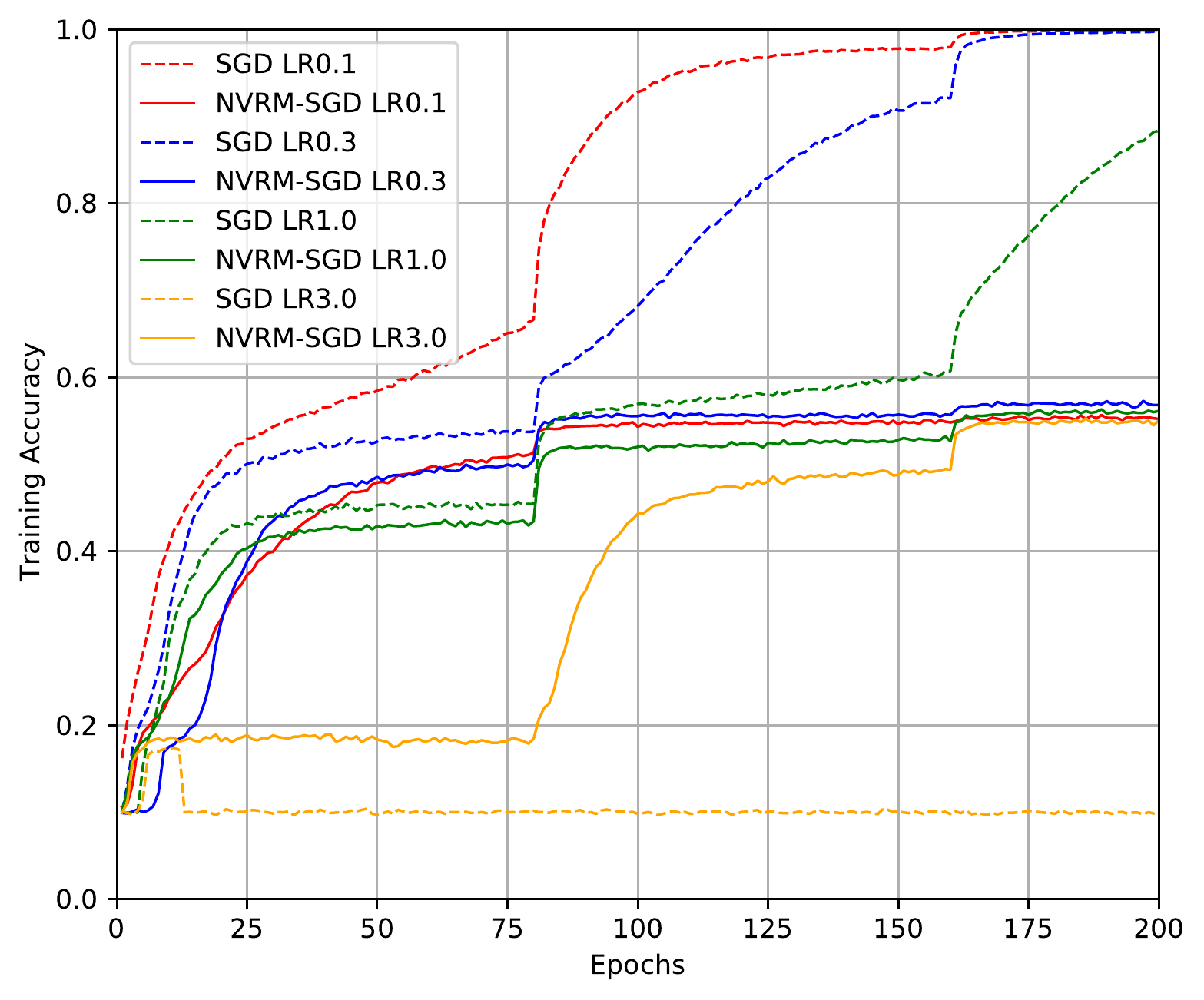}
  \end{subfigure}
  
  \caption{Curves of test accuracy to learning rates. Dataset: CIAF-10 with $40\%$ label noise. While SGD with larger stochastic gradient noise memorizes noisy labels more slowly, it still memorize nearly all noisy labels at the final phase of training. In contrast, NVRM-SGD with various learning rates can consistently relieve overfitting noisy labels. }
  \label{fig:labelnoiselr}
\end{figure}

\end{document}